%% file: piccolo.tex
\documentclass{article}

\usepackage{microtype}
\usepackage{graphicx}
\usepackage{booktabs} 

\usepackage{hyperref}

\usepackage[accepted]{icml2019}

\graphicspath{ {figures/} }
\usepackage{xspace}
\usepackage{color}
\usepackage{subcaption}  
\usepackage{url}
\def\piccolo{\textsc{PicCoLO}\xspace}
\def\piccolofull{PredICtor-COrrector poLicy Optimization\xspace}

\def\trpo{\textsc{trpo}\xspace} 
 
\def\dagger{\textsc{DAgger}\xspace} 
 
\def\aggrevate{\textsc{AggreVaTe}\xspace}

\def\mobil{\textsc{MoBIL}\xspace}

\def\adam{\textsc{Adam}\xspace} 
\def\adagrad{\textsc{Adagrad}\xspace} 
\def\natgrad{\textsc{Natgrad}\xspace} 
\def\dyna{\textsc{Dyna}\xspace}

\newcommand{\paren}[1]{ \left( #1 \right) }

\def\update{\text{\footnotesize\texttt{update}}}
\def\adapt{\text{\footnotesize\texttt{adapt}}}
\def\project{ \text{\footnotesize\texttt{project}}}
\def\shift{ \text{\footnotesize\texttt{shift}}}
\def\half{\frac{1}{2}}

\input{shortcuts}

\icmltitlerunning{Predictor-Corrector Policy Optimization}

\begin{document}
	
	\twocolumn[
	\icmltitle{Predictor-Corrector Policy Optimization}
	
	
	
	\icmlsetsymbol{equal}{*}
	
	\begin{icmlauthorlist}
		\icmlauthor{Ching-An Cheng}{gt,nv}
		\icmlauthor{Xinyan Yan}{gt}
		\icmlauthor{Nathan Ratliff}{nv}
		\icmlauthor{Byron Boots}{gt,nv}
	\end{icmlauthorlist}

	\icmlaffiliation{gt}{Georgia Tech}
	\icmlaffiliation{nv}{NVIDIA}
	\icmlcorrespondingauthor{Ching-An Cheng}{cacheng@gatech.edu}

	\icmlkeywords{Reinforcement Learning, Online Learning}
	\vskip 0.3in
	]
	
	
	
	\printAffiliationsAndNotice{}  

\begin{abstract}
	We present a predictor-corrector framework, called \piccolo, that can transform a first-order model-free reinforcement or imitation learning algorithm into a new hybrid method that leverages predictive models to accelerate policy learning.  
	The new ``{\piccolo}ed'' algorithm optimizes a policy by recursively repeating two steps: In the Prediction Step, the learner uses a model to predict the unseen future gradient 
	and then applies the predicted estimate to update the policy; in the Correction Step, the learner runs the updated policy in the environment, receives the true gradient, and then corrects the policy using the gradient error.  
	Unlike previous algorithms, \piccolo corrects for the mistakes of using imperfect predicted gradients and hence does not suffer from model bias.
	The development of \piccolo is made possible by 
	a novel
	reduction from predictable online learning to adversarial online learning,  which provides a systematic way to modify existing first-order algorithms to achieve the optimal regret with respect to predictable information.
	We show, in both theory and simulation, that the convergence rate 
	of several firs	t-order model-free algorithms can be improved by \piccolo. 
\end{abstract}

\section{Introduction}

Reinforcement learning (RL) has recently solved a number of challenging problems~\cite{mnih2013playing,duan2016benchmarking,silver2017mastering}. 
However, many of these successes are confined to games and simulated environments, where a large number of agent-environment interactions can be cheaply performed. 
Therefore, they are often unrealistic in real-word applications (like robotics) where data collection is an expensive and time-consuming process. Improving sample efficiency still remains a critical challenge for RL. 

Model-based RL methods improve sample efficiency by leveraging an {accurate} model that can cheaply simulate interactions to compute policy updates in lieu of real-world interactions~\cite{tan2018sim}. A classical example of pure model-based methods is optimal control~\cite{jacobson1970differential,todorov2005generalized,deisenroth2011pilco,pan2014probabilistic}, which has recently been extended to model abstract latent dynamics with neural networks~\cite{silver2016predictron,oh2017value}. These methods use a (local) model of the dynamics and cost functions 
to predict cost-to-go functions, policy gradients, or promising improvement direction
when updating policies~\citep{levine2013guided,wen2018dual,anthony2017thinking}. 
Another way to use model information is the {hybrid} \dyna framework~\cite{sutton1991dyna,sutton2012dyna}, which interleaves model-based and model-free updates, ideally cutting learning time in half. 
However, all of these approaches, while potentially accelerating policy learning, suffer from a common drawback: when the model is inaccurate,  the performance of the policy can become  \emph{biased} away from the best achievable in the policy class.

Several strategies have been proposed to remove this performance bias. {Learning-to-plan} attempts to train the planning process end-to-end~\cite{pascanu2017learning,srinivas2018universal,amos2018differentiable}, so the performance of a given planning structure is directly optimized.
However, these algorithms are still optimized through standard model-free RL techniques; it is unclear as to whether they are more sample efficient.
In parallel, another class of bias-free algorithms is control variate methods~\cite{chebotar2017combining,grathwohl2017backpropagation,papini2018stochastic}, which use models to reduce the variance of sampled gradients to improve convergence.

In this paper, we provide a novel learning framework that can leverage models to improve sample efficiency while avoiding performance bias due to modeling errors.
Our approach is built on techniques from online learning~\cite{gordon1999regret,zinkevich2003online}. The use of online learning to analyze policy optimization was pioneered by~\citet{ross2011reduction}, who proposed to reduce imitation learning (IL) to adversarial online learning problems. 
This reduction provides a framework for performance analysis, leading to algorithms such as \dagger~\cite{ross2011reduction} and  \aggrevate~\cite{ross2014reinforcement}. 
However, it was recently shown that the na\"{\i}ve reduction to adversarial online learning loses information~\cite{cheng2018convergence}: in practice, IL is \emph{predictable}~\cite{cheng2019accelerating} and can be thought of as a predictable online learning problem~\cite{rakhlin2013online}. 
Based on this insight, \citet{cheng2019accelerating} recently proposed a two-step algorithm, \mobil. The authors prove that, by leveraging predictive models to estimate future gradients, 
\mobil can speed up the convergence of IL, without incurring performance bias due to imperfect models.

Given these theoretical advances in IL, it is natural to ask if similar ideas can be extended to RL. 
In this paper, we show that RL can also be formulated as a predictable online learning problem, and 
we propose a novel first-order learning framework, \piccolo (\piccolofull), for general predictable online learning problems.
\piccolo is a \emph{meta-algorithm}: it takes a standard online learning algorithm designed for adversarial problems (e.g. \adagrad~\cite{duchi2011adaptive}) as input and returns a new hybrid algorithm that can use model information to accelerate convergence. This new ``{\piccolo}ed'' algorithm optimizes the policy by 
alternating between Prediction and Correction steps.
In the Prediction Step, the learner uses a predictive model to estimate the gradient of the next loss function and then uses it to update the policy; in the Correction Step, the learner executes the updated policy in the environment, receives the true gradient 
, and then corrects the policy using the gradient \emph{error}.
We note that \piccolo is
orthogonal to control variate methods; it can still improve learning even 
in the noise-free setting (see Section~\ref{sec:comparison}).

Theoretically, we prove that \piccolo can improve the convergence rate of \emph{any} base algorithm that can be written as mirror descent~\cite{beck2003mirror} or Follow-the-Regularized-Leader (FTRL)~\cite{mcmahan2010adaptive}. This family of algorithms is rich and covers most first-order algorithms used in RL and IL~\cite{cheng2018fast}. 
And, importantly, we show that \piccolo does not suffer from performance bias due to model error, unlike previous model-based approaches. To validate the theory, we ``{\piccolo}'' multiple algorithms in simulation. The experimental results show that the {\piccolo}ed versions consistently surpass the base algorithm and are robust to model errors. 

The design of \piccolo is made possible by a novel reduction that converts a given  predictable online learning problem into a new adversarial problem, so that standard online learning algorithms can be applied optimally without referring to specialized algorithms. 
We show that \piccolo includes and generalizes many existing algorithms, e.g., \mobil, mirror-prox~\citep{juditsky2011solving}, and optimistic mirror descent~\citep{rakhlin2013online} (Appendix \ref{app:relationship}).
Thus, we can treat \piccolo as an automatic process for designing new algorithms that safely leverages imperfect predictive models (such as off-policy gradients or  gradients simulated through dynamics models) to speed up learning.

\section{Problem Definition}

We consider solving policy optimization problems: given state and action spaces $\Sbb$ and $\Abb$, and a parametric policy class $\Pi$, we desire a stationary policy $\pi \in \Pi$ that solves
\begin{align}  \label{eq:RL problem}
\min_{\pi \in \Pi} J(\pi), \quad J(\pi) \coloneqq \E_{(s,t) \sim d_{\pi}}  \E_{a \sim \pi_s} \left[ c_t(s, a) \right]
\end{align}
where $c_t(s,a)$ is the instantaneous cost at time $t$ of state $s\in\Sbb$ and  $a\in\Abb$,  $\pi_s$ is the distribution of $a$ at state $s$ under policy $\pi$, and $d_{\pi}$ is a generalized stationary distribution of states generated by running policy $\pi$ in a Markov decision process (MDP); the notation $\E_{a \sim \pi_s}$ denotes evaluation when $\pi$ is deterministic. 
The use of $d_{\pi}$ in~\eqref{eq:RL problem} abstracts different discrete-time RL/IL problems into a common setup. For example, an infinite-horizon $\gamma$-discounted problem with time-invariant cost $c$ can be modeled by setting $c_t =c$ and $d_{\pi}(s,t) = (1-\gamma) \gamma^t d_{\pi,t}(s)$, where $d_{\pi,t}$ is the state distribution visited by policy $\pi$ at time $t$ starting from some \emph{fixed} but unknown initial state distribution.

For convenience, we will usually omit the random variable in expectation notation 
(e.g. we will write  \eqref{eq:RL problem} as $\E_{d_{\pi}}\E_{\pi} \left[ c \right]$). For a policy $\pi$, we overload the notation $\pi$ to also denote its parameter, and write $Q_{\pi,t}$ and $V_{\pi,t} \coloneqq \E_{\pi} [Q_{\pi,t}]$ as its Q-function and value function at time $t$, respectively.

\section{IL and RL as Predictable Online Learning} \label{sec:reduction to online learning}

We study policy optimization through the lens of online learning~\cite{hazan2016introduction}, 
by treating a policy optimization algorithm as the learner in online learning and \emph{each intermediate policy} that it produces as an online decision.
This identification recasts the iterative process of policy optimization into a standard online learning setup: 
in round $n$, the learner plays a decision $\pi_n \in \Pi$, a \emph{per-round loss} $l_n$ is then selected, and finally some information of $l_n$ is revealed to the leaner for making the next decision. We note that the ``rounds'' considered here are the number of episodes that an algorithm interacts with the (unknown) MDP environment to obtain new information, not the time steps in the MDP. 
And we will suppose the learner receives an unbiased stochastic approximation $\tilde{l}_n$ of $l_n$ as feedback.

We show that, when the per-round losses $\{l_n\}$ are properly selected, the policy performance $\{J(\pi_n)\}$ in IL and RL can be upper bounded 
in terms the $N$-round weighted regret
\begin{align} \label{eq:weighted regret}
\regret_N(l) \coloneqq \sum_{n=1}^{N} w_n l_n(\pi_n) -  \min_{\pi \in \Pi} \sum_{n=1}^{N} w_n {l}_n(\pi)
\end{align}
and an expressiveness measure of the policy class $\Pi$
\begin{align}\label{eq:baseline performance}
\textstyle
\epsilon_{\Pi,N}(l) \coloneqq  \frac{1}{w_{1:N}} \min_{\pi \in \Pi} \sum_{n=1}^{N} w_n l_n(\pi)
\end{align}
where $w_n> 0$ and  $w_{1:n} \coloneqq \sum_{m=1}^{n} w_m$. 
Moreover, we show that these online learning problems are \emph{predictable}: that is, the per-round losses are not completely adversarial but can be estimated from past information.
We will use these ideas to design \piccolo in the next section.

\subsection{IL as Online Learning} \label{sec:IL as OL}
We start by reviewing the classical online learning approach to IL (online IL for short)~\citep{ross2011reduction} to highlight some key ideas.
IL leverages domain knowledge about a policy optimization problem through expert demonstrations. Online IL, in particular, optimizes policies by letting the learner $\pi$ query the expert $\pi^\*$ for desired actions, so that a policy can be quickly trained to perform as well as the expert. At its heart, online IL is based on the
following lemma, which relates the performance between $\pi$ and $\pi^\*$.
\begin{lemma} \label{lm:performance difference}
	{\normalfont \citep{kakade2002approximately}}
	Let $\pi$ and $\pi'$ be two policies and  $ A_{\pi',t}(s, a) \coloneqq Q_{\pi',t}(s,a) - V_{\pi',t}(s)$. 
	Then 
	$J(\pi) = J(\pi') + \E_{d_\pi} \E_{\pi} [ A_{\pi'}  ].$   
\end{lemma}
Given the equality in~\cref{lm:performance difference}, the performance difference between $\pi $ and $\pi^\*$ can then be upper-bounded as 
\begin{align*} 
\hspace{-2mm}
J(\pi) \hspace{-0.5mm} - \hspace{-0.5mm} J(\pi^\*) 
\hspace{-0.5mm}=  \hspace{-0.5mm}\E_{d_{\pi}} \E_{\pi} [ A_{\pi^\*} ] 
\leq
C_{\pi^\*} \E_{(s,t) \sim d_{\pi}}  [D_t(\pi_s^\* || \pi_s)]
\end{align*}
for some positive constant $C_{\pi^\*}$ and function $D_t$, which is often derived from statistical distances such as KL divergence~\citep{cheng2018fast}.  When $A_{\pi^*,t}$ is available, we can also set $D_t(\pi_s^\*||\pi_s) = \E_{a \sim \pi_s}[ A_{\pi^\*,t}(s,a)]$, as in value aggregation (\aggrevate)~\cite{ross2014reinforcement}.

Without loss of generality, 
let us suppose $D_t(\pi_s^\star||\pi_s) = \E_{a\sim\pi_s}[\bar{c}_t(s,a) ]$ for some $\bar{c}_t$.
Online IL converts policy optimization into online learning with  per-round loss 
\begin{align} \label{eq:IL online loss}
l_n(\pi) \coloneqq \E_{d_{\pi_n}} \E_{\pi}[\bar{c}].
\end{align}
By the inequality above, it holds that $J(\pi_n) - J(\pi^\*) \leq C_{\pi^\*} l_n(\pi_n)$ for every $n$, establishing the reduction below. 
\begin{lemma} \label{lm:IL performance bound}
{\normalfont \citep{cheng2019accelerating}}
	For $l_n$ defined in~\eqref{eq:IL online loss}, 	
	{\small
	$
		\E\left[  \sum_{n=1}^{N} \frac{ w_n J(\pi_n) }{w_{1:N}} \right] \leq J(\pi^\*) 
		+  C_{\pi^\*} \E\left[  \epsilon_{\Pi,N}(l)+    \frac{\regret_N(\tilde{l}) }{ w_{1:N}}\right]
	$},
	where the expectation is due to sampling $\tilde{l}_n$. 
\end{lemma}
That is, when a no-regret algorithm is used, 
the performance 
concentrates toward $J(\pi^\* ) + C_{\pi^\*}\E[ \epsilon_{\Pi,N} (l) ]$. 

\subsection{RL as Online Learning}\label{sec:RL as OL}
Can we also formulate RL as online learning? 
Here we propose a new perspective on RL using Lemma~\ref{lm:performance difference}. Given a policy $\pi_n$ in round $n$, we define a per-round loss 
\begin{align} \label{eq:RL online loss}
l_n(\pi) \coloneqq \E_{d_{\pi_n}} \E_{\pi} [ A_{\pi_{n-1}} ].
\end{align}
which describes how well a policy $\pi$ performs relative to the previous policy $\pi_{n-1}$ under the state distribution of $\pi_n$. 
By~\cref{lm:performance difference}, 
for $l_n$ defined in~\eqref{eq:RL online loss},
$l_n(\pi_n) = J(\pi_{n}) - J(\pi_{n-1})$ for every $n$, similar to the pointwise inequality of $l_n$ that \cref{lm:IL performance bound} is based on. With this observation, we derive the reduction below (proved in~\cref{app:proof of RL performance bound}).
\begin{lemma} \label{lm:RL performance bound}	
	Suppose $\frac{w_{n+k}}{w_{n}} \leq  \frac{w_{m+k}}{w_{m}}$, for all $ n \geq m \geq 1$ and $k \geq 0 $. For~\eqref{eq:RL online loss} and any $\pi_0$,
	$
	\E[ \sum_{n=1}^{N}\frac{w_n J(\pi_n) }{w_{1:N}}  ] \leq J(\pi_0) 
	+ \sum_{n=1}^{N}  \frac{w_{N-n+1}}{w_{1:N}}  \E[\regret_n (\tilde{l}) + w_{1:n} \epsilon_{\Pi,n} (l) ] 
	$, 
	where the expectation is due to sampling $\tilde{l}_n$.
\end{lemma}

\subsubsection{Interpretations}

\cref{lm:RL performance bound} is a policy improvement lemma, 
which shows that when the learning algorithm is no-regret, the policy sequence improves on-average from the initial reference policy $\pi_0$ that defines $l_1$. 
This is attributed to an important property of the definition in~\eqref{eq:RL online loss} that $\min_{\pi \in \Pi} l_n (\pi) \leq 0 $. 
To see this, suppose $\E[\epsilon_{\Pi,n} (l)]\leq -\Omega(1)$ (i.e. there is a policy that is better than all previous $n$ policies); this is true for small $n$ or when the policy sequence is concentrated. 
Under this assumption, if $w_n = 1$ and $\regret_n (\tilde{l}) \leq O(\sqrt{n})$, then the average performance improves roughly $N \E[\epsilon_{\Pi,N}(l)]$ away from $J(\pi_0)$.

While it is unrealistic to expect  $\E[\epsilon_{\Pi,n} (l)]\leq 0$ for large $n$, 
we can still use \cref{lm:RL performance bound} to comprehend \emph{global} properties of policy improvement, for two reasons.
First, the inequality in~\cref{lm:RL performance bound} holds for any interval of the policy sequence.
Second, as we show in~\cref{app:proof of RL performance bound}, the \cref{lm:RL performance bound} also applies to dynamic regret~\citep{zinkevich2003online}, with respect to which $\E[\epsilon_{\Pi,n} (l)]$ is always negative. 
Therefore, if an algorithm is strongly-adaptive~\citep{daniely2015strongly} (i.e. it is no-regret for any interval) or has sublinear dynamic regret~\citep{jadbabaie2015online}, then its generated policy sequence will strictly, non-asymptotically improve.
In other words, for algorithms with a stronger notion of convergence, \cref{lm:RL performance bound} describes the global improvement rate.

\subsubsection{Connections}

The choice of per-round loss in~\eqref{eq:RL online loss} has an interesting relationship to both actor-critic in RL~\cite{konda2000actor} and \aggrevate in IL~\cite{ross2014reinforcement}.

\vspace{-2mm}
\paragraph{Relationship to Actor-Critic} 
Although actor-critic methods, theoretically, use $\E_{d_{\pi_n}} (\nabla \E_{\pi}) [A_{\pi_n}] |_{\pi=\pi_n}$ to update policy $\pi_n$, {in practice}, they use
$\E_{d_{\pi_n}} (\nabla \E_{\pi}) [A_{\pi_{n-1}}] |_{\pi=\pi_n}$, because the advantage/value function estimate in round $n$ is updated {after} the policy update in order to prevent bias due to over-fitting on finite samples~\citep{sutton1998introduction}. 
This practical gradient is \emph{exactly} $\nabla \tilde{l}_n(\pi_n)$, the sampled gradient of~\eqref{eq:RL online loss}. Therefore, \cref{lm:RL performance bound} explains the properties of these practical modifications.

\vspace{-2mm}
\paragraph{Relationship to Value Aggregation}
\aggrevate~\citep{ross2014reinforcement} can be viewed as taking a policy improvement step from some reference policy: e.g., with the per-round loss
$\E_{d_{\pi_n}} \E_{\pi} [ A_{\pi^\*}]$, it improves one step from $\pi^*$. 
Realizing this one step improvement in \aggrevate, however, requires solving multiple rounds of online learning, 
as it effectively solves an equilibrium point problem 
~\cite{cheng2018convergence}.
Therefore, while ideally one can solve multiple \aggrevate problems (one for each policy improvement step) to optimize policies, computationally this can be very challenging.
Minimizing the loss in~\eqref{eq:RL online loss} can be viewed as an approximate policy improvement step in the \aggrevate style. 
Rather than waiting until convergence in each \aggrevate policy improvement step, it performs only a \emph{single} policy update and then switches to the next \aggrevate problem with a new reference policy (i.e. the latest policy $\pi_{n-1}$). 
This connection is particularly tightened if we choose $\pi_0 = \pi^\*$ and the bound in~\cref{lm:RL performance bound} becomes relative to $J(\pi^\*)$.

\vspace{-1mm}
\subsection{Predictability}
An important property of the above online learning problems is that they are not completely adversarial, as pointed out by \citet{cheng2018convergence}  for IL.
This can be seen from the definitions of $l_n$ in~\eqref{eq:IL online loss} and~\eqref{eq:RL online loss}, respectively.  
For example, suppose the cost $c_t$ in the original RL problem~\eqref{eq:RL problem} is known; then the information unknown before playing the decision $\pi_n$ in the environment is only the state distribution $d_{\pi_n}$.
Therefore, the per-round loss cannot be truly adversarial, 
as the same dynamics and cost functions are used across different rounds. 
That is, in an idealized case where the true dynamics and cost functions are exactly known, 
using the policy returned from a model-based RL algorithm would incur zero regret, since only the interactions with the real MDP environment, not the model, counts as rounds. We will exploit this property to design \piccolo.

\vspace{-1mm}
\section{Predictor-Corrector Learning}
\label{sec:piccolo}

We showed that the performance of RL and IL can be bounded by the regret of properly constructed predictable online learning problems. These results provide a foundation for designing policy optimization algorithms: 
efficient learning algorithms for policy optimization can be constructed from powerful online learning algorithms that achieve small regret. 
This perspective explains why common methods (e.g. mirror descent) based on gradients of~\eqref{eq:IL online loss} and~\eqref{eq:RL online loss} work well in IL and RL. However, the predictable nature of policy optimization problems suggests that directly applying these standard online learning algorithms designed for adversarial settings is \emph{suboptimal}. 
The predictable information must be considered 
to achieve optimal convergence. 

One way to include predictable information is to develop specialized two-step algorithms based on, e.g., mirror-prox or FTRL-prediction~\cite{juditsky2011solving,rakhlin2013online,ho2017exploiting}. For IL, \mobil was recently proposed~\cite{cheng2019accelerating}, which updates policies by approximate Be-the-Leader~\cite{kalai2005efficient} and provably achieves faster convergence than previous methods. 
However, these two-step algorithms often have obscure and non-sequential update rules, 
and their adaptive and accelerated versions are less accessible~\cite{diakonikolas2017accelerated}.
This can make it difficult to implement and tune them in practice. 

Here we take an alternative, \emph{reduction-based} approach. 
We present \piccolo, a general first-order framework for solving predictable online learning problems. 
\piccolo is a meta-algorithm that turns a base algorithm designed for adversarial problems into a new algorithm that can leverage the predictable information to achieve better performance.  
As a result, we can adopt sophisticated first-order adaptive algorithms  to optimally learn policies, without reinventing the wheel.  
Specifically, given \emph{any} first-order base algorithm belonging to the family of (adaptive) mirror descent and FTRL algorithms, we show how one can ``\piccolo it'' to achieve a faster convergence rate without introducing additional performance bias due to prediction errors. 
Most first-order policy optimization algorithms belong to this family~\cite{cheng2018fast}, so we can \piccolo these model-free algorithms into new hybrid algorithms that can robustly use (imperfect) predictive models, such as off-policy gradients and simulated gradients, to improve policy learning.

\vspace{-1mm}
\subsection{The \piccolo Idea}

The design of \piccolo is based on the observation that an $N$-round predictable online learning problem can be written as a new adversarial problems with $2N$ rounds.
To see this, let $\{l_n\}_{n=1}^N$ be the original predictable loss sequence. Suppose, before observing $l_n$, we have access to a \emph{model loss} $\hat{l}_n$ that contains the predictable information of $l_n$. 
Define $\delta_n = l_n - \hat{l}_n$. We can then write the accumulated loss (which regret concerns) as 
$
\sum_{n=1}^{N} l_n(\pi_n) = \sum_{n=1}^{N} \hat{l}_n(\pi_n) + \delta_n(\pi_n) 
$. That is, we can view the predictable problem with $\{l_n\}_{n=1}^N$ as a new adversarial online learning problem with a loss sequence $\hat{l}_1, \delta_1, \hat{l}_2, \delta_2, \dots, \hat{l}_N, \delta_N$. 

The idea of \piccolo is to apply standard online learning algorithms designed for adversarial settings to this new $2N$-round problem.
This would create a new set of decision variables $\{\hat{\pi}_n\}_{n=1}^N$, in which $\hat{\pi}_n$ denotes the decision made before seeing $\hat{l}_n$, and leads to the following sequence $\pi_1, \delta_1, \hat{\pi}_2, \hat{l}_2, \pi_2, \delta_2, \dots$ (in which we define $\delta_1 = l_1$). 
We show that when the base algorithm is optimal in adversarial settings, this simple strategy results in a decision sequence $\{\pi_n\}_{n=1}^N$ whose regret with respect to $\{l_n\}_{n=1}^N$ is optimal, just as those specialized two-step algorithms~\cite{juditsky2011solving,rakhlin2013online,ho2017exploiting}. 
In Appendix~\ref{app:relationship}, we show \piccolo unifies and generalize these two-step algorithms to be adaptive.

\subsection{The Meta Algorithm \piccolo}

We provide details to realize this reduction. 
We suppose, in round $n$, the model loss is given as $\hat{l}_n(\pi) = \lr{\hat{g}_n}{\pi}$ for some vector $\hat{g}_n$, and stochastic first-order feedback $g_n = \nabla \tilde{l}_n(\pi_n)$ from $l_n$ is received. 
Though this linear form of model loss seems restrictive, later in~\cref{sec:model loss} we will show that it is sufficient to represent predictable information.

\vspace{-2mm}
\subsubsection{Base Algorithms} \label{sec:base alg}
We first give a single description of different base algorithms for the formal definition of the reduction steps. Here we limit our discussions to mirror descent and postpone the FTRL case to~\cref{app:basic operations}. We assume that $\Pi$ is a convex compact subset in some normed space with norm $\norm{\cdot}$, and we use $B_R(\pi || \pi') = R(\pi) - R(\pi') - \lr{\nabla R(\pi')}{\pi - \pi'}$ to denote a Bregman divergence generated by a strictly convex function $R$, called the distance generator. 

Mirror descent updates decisions based on proximal maps.
In round $n$, given direction $g_n$ and weight $w_n$, it executes 
\begin{align} \label{eq:mirror descent}
\textstyle
\pi_{n+1} =  \argmin_{\pi \in \Pi} \lr{w_n g_n}{\pi} +  B_{R_n}(\pi|| \pi_n)
\end{align}
where $R_n$ is a strongly convex function; \eqref{eq:mirror descent} reduces to gradient descent with step size $\eta_n$ when $R_n(\cdot) = \frac{1}{2\eta_n}\norm{\cdot}^2$.
More precisely, \eqref{eq:mirror descent} is composed of two steps: 1) the update of the distance generator to $R_n$, and 2) the update of the decision to $\pi_{n+1}$; different mirror descent algorithms differ in how the regularization is selected and adapted.

\piccolo explicitly treats a base algorithm as the composition of two basic operations (this applies also to FTRL)
\begin{align} \label{eq:general adaptive scheme}
\begin{split}
H_n &= \adapt(h_n, H_{n-1}, g_n, w_n)\\
h_{n+1} &=  \update(h_n, H_n, g_n, w_n) 
\end{split}
\end{align}
so that later it can recompose them to generate the new algorithm.
For generality,
we use $h$ and $H$ to denote the abstract representations of the decision variable and the regularization, respectively. 
In mirror descent, $h$ is exactly the decision variable, $H$ is the distance generator, and we can write
$
\update(h, H, g, w) =   \argmin_{\pi' \in \Pi} \lr{wg}{\pi'} +  B_H(\pi' || h) 
$.
The operation $\adapt$ denotes the algorithm-specific scheme for the regularization update (e.g. changing the step size), which in general updates the size of regularization to grow {slowly} and {inversely} proportional to the norm of $g_n$.

\vspace{-2mm}
\subsubsection{The {\piccolo}ed Algorithm}
\label{sec:piccolo rules}
\piccolo generates decisions by applying a given base algorithm in~\eqref{eq:general adaptive scheme} to the new problem with losses $\delta_1, \hat{l}_2, \delta_2, \dots$. 
This is accomplished by recomposing the basic operations in~\eqref{eq:general adaptive scheme} into the Prediction and the Correction Steps:
\begin{align*}
h_{n} &= \update(\hat{h}_{n}, H_{n-1}, \hat{g}_{n}, w_n)  &  \text{[Prediction]}
\\[2.5mm]
\begin{split}
H_{n} &= \adapt(h_n, H_{n-1}, e_n, w_n) \\[-1mm]
\hat{h}_{n+1} &= \update(h_n, H_{n}, e_n, w_n) 
\end{split} & \text{[Correction]}  
\end{align*} 
where $\hat{h}_n$ is the abstract representation of $\hat{\pi}_n$, and $e_n = g_n - \hat{g}_n$ is the error direction. 
We can see that the Prediction and Correction Steps are exactly the update rules resulting from applying~\eqref{eq:general adaptive scheme} to the new adversarial problem, except that only $h_n$ is updated in the Prediction Step, \emph{not} the regularization (i.e. the step size).
This asymmetry design is important for achieving optimal regret, because in the end we care only about the regret of $\{\pi_n\}$ on the original loss sequence $\{l_n\}$. 

In round $n$, the ``{\piccolo}ed'' algorithm first 
performs the Prediction Step using $\hat{g}_n$ to generate the learner's decision (i.e. $\pi_n$) and runs this new policy in the environment to get the true gradient $g_n$.
Using this feedback, the algorithm performs the Correction Step to amend the bias of using $\hat{g}_n$. This is done by first adapting the regularization to $H_n$ and then updating $\pi_n$ to $\hat{\pi}_{n+1}$ along the error $e_n = g_n - \hat{g}_n$.

\subsubsection{Model Losses and Predictive Models} \label{sec:model loss}

The Prediction Step of \piccolo relies on the vector $\hat{g}_n$ to approximate the future gradient $g_n$. 
Here we discuss different ways to specify $\hat{g}_n$ based on the concept of predictive models~\citep{cheng2019accelerating}. 
A \emph{predictive model} $\Phi_n$ is a first-order oracle such that $\Phi_n(\cdot)$ approximates $\nabla l_n(\cdot)$.
In practice, a predictive model can be a simulator with an (online learned) dynamics model~\cite{tan2018sim,deisenroth2011pilco}, or a neural network trained to predict the required gradients~\cite{silver2016predictron,oh2017value}. An even simpler heuristic is to construct predictive models by \emph{off-policy} gradients $\Phi_n(\cdot) = \sum_{m=n-K}^{n-1}\nabla \tilde{l}_{m} (\cdot)$ where $K$ is the buffer size.

In general, we wish to set $\hat{g}_n$ to be close to $g_n$, as we will later show  in~\cref{sec:piccolo theories} that the convergence rate of \piccolo depends on their distance. However, even when we have perfect predictive models, this is still a non-trivial task. We face a chicken-or-the-egg problem: $g_n$ depends on $\pi_n$, which in turn depends on $\hat{g}_n$ from the Prediction Step. 

\citet{cheng2019accelerating} show one effective heuristic is to set $\hat{g}_n = \Phi_n(\hat{\pi}_n)$, because we may treat $\hat{\pi}_n$ as an estimate of $\pi_n$. However, due to the mismatch between $\hat{\pi}_n$ and $\pi_n$, this simple approach has errors even when the predictive model is perfect. To better leverage a given predictive model, we propose to solve for $\hat{g}_n$ and $\pi_n$ \emph{simultaneously}. 
That is, we wish to solve a fixed-point problem, finding $h_n$ such that 
\begin{align} \label{eq:fixed-point problem}
h_{n} &= \update(\hat{h}_{n}, H_{n-1}, \Phi_n(\pi_n(h_{n})), w_n)  
\end{align}
The exact formulation of the fixed-point problem depends on the class of base algorithms. For mirror descent, it is a variational inequality: find $\pi_{n} \in \Pi$ such that
$\forall \pi \in \Pi$, 
$
\lr{\Phi_n(\pi_{n}) + \nabla R_{n-1}(\pi_n) - \nabla R_{n-1}(\hat{\pi}_n) }{\pi - \pi_n} \geq 0
$. In a special case  when $\Phi_n = \nabla f_n$ for some function $f_n$, the above variational inequality is equivalent to finding a stationary point of the optimization problem $\min_{\pi \in \Pi} f_n(\pi)  + B_{R_{n-1}}(\pi||\hat\pi_n)$. 
In other words, one way to implement the Prediction Step is to solve the above minimization problem for $\pi_n$ and use $\nabla f_n(\pi_n)$ as the effective prediction $\hat{g}_n$. 

\vspace{-1.5mm}
\subsection{Summary: Why Does \piccolo Work?}\label{sec:example}
We provide a summary of the full algorithm for policy optimization in~\cref{alg:piccolo}. We see that \piccolo uses the predicted gradient to take an extra step to accelerate learning, and, meanwhile, to prevent the error accumulation, it adaptively adjusts the step size (i.e. the regularization) based on the prediction error and corrects for the bias on the policy right away. 
To gain some intuition, let us consider \adagrad~\cite{duchi2011adaptive} as a base algorithm\footnote{We provide another example 
	in~\cref{app:example}.}: 
\begin{small}
	\begin{align*}
	G_n &= G_{n-1} + \diag(w_n g_n \odot w_n g_n)\\[-1mm]
	\pi_{n+1} &= \argmin_{\pi \in \Pi} \lr{w_n g_n}{\pi} + \frac{1}{2 \eta} (\pi - \pi_n)^\top G_n^{1/2} (\pi - \pi_n) \nonumber
	\end{align*}\end{small}%
where $G_0 = \epsilon I$ and $\eta,\epsilon>0$, and $\odot$ denotes element-wise multiplication. This update has an 
$\adapt$ operation as
$
\adapt(h, H, g, w) = G + \diag(w g \odot w g)
$ which updates the Bregman divergence based on the gradient size. 

\piccolo transforms \adagrad into a new algorithm. In the Prediction Step, it performs
\begin{small}
	\begin{align*}
	\pi_{n} &= \argmin_{\pi \in \Pi} \lr{w_n \hat{g}_n}{\pi} + \frac{1}{2 \eta} (\pi - \pi_{n-1})^\top G_{n-1}^{1/2} (\pi - \pi_{n-1})
	\end{align*}
\end{small}%
In the Correction Step, it performs
\begin{small}
	\begin{align*}
	G_n &= G_{n-1} + \diag(w_n e_n \odot w_n e_n) \\[-1mm]
	\hat{\pi}_{n+1} &= \argmin_{\pi \in \Pi} \lr{w_n e_n}{\pi} + \frac{1}{2 \eta} (\pi -\hat{\pi}_n)^\top G_n^{1/2} (\pi - \hat{\pi}_n) \nonumber
	\end{align*}
\end{small}%
We see that the {\piccolo}-\adagrad updates $G_n$ proportional to the prediction error $e_n$ instead of $g_n$.
It takes larger steps when models are accurate, and decreases the step size once the prediction deviates. As a result, \piccolo is robust to model quality: it accelerates learning when the model is informative, and prevents inaccurate (potentially adversarial) models from hurting the policy. We will further demonstrate this in theory and in the experiments.

\renewcommand*{\thefootnote}{\fnsymbol{footnote}}
\def\PredictionStep{\text{\texttt{PredictionStep}}}
\def\CorrectionStep{\text{\texttt{CorrectionStep}}}
\def\ModelUpdate{\text{\texttt{ModelUpdate}}}
\def\DataCollection{\text{\texttt{DataCollection}}}
\begin{algorithm}[t] 
	{\small
		\caption{\piccolo }\label{alg:piccolo} 
		\begin{algorithmic} [1]
			\renewcommand{\algorithmicensure}{\textbf{Input:}}		
			\renewcommand{\algorithmicrequire}{\textbf{Output:}}
			\ENSURE  policy $\pi_1$, cost sequence $\{\psi_n\}$, regularization $H_0$, model $\Phi_1$,  iteration $N$, exponent $p$
			\REQUIRE $\bar \pi_N$
			\STATE Set $\hat{\pi}_1 = \pi_1$ and weights $w_n = n^p$
			\STATE Sample integer $K \in [1,N]$ with $P(K=n) \propto w_n$ 
			\FOR {$n = 1\dots K-1$\footnotemark}
			\STATE $\pi_n, \hat{g}_n = \PredictionStep(\hat{\pi}_n, \Phi_n, H_{n-1}, w_n )$
			\STATE $\DD_{n}, g_n = \DataCollection(\pi_n)$
			\STATE $H_n, \hat{\pi}_{n+1} = \CorrectionStep(\pi_n, e_n, H_{n-1}, w_n)$, where $e_n = g_n - \hat{g}_n$.
			\STATE $\Phi_{n+1}  = \ModelUpdate(\Phi_n, \DD)$, where $\DD = \DD \bigcup \DD_n$. 
			\\			
			\ENDFOR
			\STATE Set $\bar \pi_N = \pi_{K-1}$			
		\end{algorithmic}
	}
\end{algorithm} 
\footnotetext{Here we assume $\project$ is automatically performed inside $\PredictionStep$ and $\CorrectionStep$.}
\renewcommand*{\thefootnote}{\arabic{footnote}}

\vspace{-2mm}
\section{Theoretical Analysis} \label{sec:piccolo theories}

In this section, we show that \piccolo has two major benefits over previous approaches: 1) it accelerates policy learning when the models predict the required gradient well {on average}; and 2) it does not bias the performance of the policy, even when the prediction is incorrect. 

To analyze \piccolo, we introduce an assumption to quantify the $\adapt$ operator of a base algorithm. 
\begin{assumption} \label{as:base algorithm}
	\def\dist{\mathrm{dist}}
	{\normalfont$\adapt$} chooses a regularization sequence such that, for some $M_{N}= o(w_{1:N})$,
	$
	\norm{H_0}_{\RR} + \sum_{n=1}^{N} \norm{H_n - H_{n-1}}_{\RR} \le M_N 
	$
	for some norm $\norm{\cdot}_{\RR}$ which measures the size of regularization.
\end{assumption} 
This assumption, which requires the regularization to increase slower than the growth of $w_{1:N}$,  is satisfied by most reasonably-designed base algorithms. For example, in a uniformly weighted problem, gradient descent with a decaying step size $O(\frac{1}{\sqrt{n}})$ has $M_N = O(\sqrt{N})$. In general, for stochastic problems, an optimal base algorithm would ensure $M_N = O(\frac{w_{1:N}}{\sqrt{N}})$.

\vspace{-1mm}
\subsection{Convergence Properties}

Now we state the main result, which quantifies the regret of \piccolo with respect to the sequence of linear loss functions that it has access to. The proof is given in Appendix~\ref{app:piccolo regret analysis}. 

\begin{restatable}{theorem}{piccolobound} \label{th:piccolo}
	Suppose $H_n$ defines a strongly convex function with respect to $\norm{\cdot}_n$.	
	Under Assumption~\ref{as:base algorithm}, running \piccolo ensures
	$
	\sum_{n=1}^N  \lr{ w_n g_n }{\pi_n - \pi} 
	\leq 
	M_{N}  +   \sum_{n=1}^N  \frac{w_n^2}{2} \norm{ e_n }_{*,n}^2  -  \frac{1}{2} \norm{\pi_n - \hat{\pi}_n}_{n-1}^2
	$, for all $\pi \in \Pi$.
\end{restatable}
The term $\norm{ e_n }_{*,n}^2$ in Theorem~\ref{th:piccolo} says that the performance of \piccolo depends on how well the base algorithm adapts to the error $e_n$ through the $\adapt$ operation in the Correction Step. Usually $\adapt$ updates $H_n$ gradually (Assumption~\ref{as:base algorithm}) while minimizing $\frac{1}{2} \norm{ e_n }_{*,n}^2$, like we showed in \adagrad.

In general, when the base algorithm is adaptive and optimal for adversarial problems, we show in Appendix~\ref{app:analysis of policy optimization} that its {\piccolo}ed version guarantees that 
$\E[\sum_{n=1}^N  \lr{ w_n  g_n }{\pi_n - \pi} ]
\leq O(1) + C_{\Pi,\Phi}  \frac{w_{1:N}}{\sqrt{N}}$, 
where $C_{\Pi,\Phi}  = O(  |\Pi| + E_\Phi + \sigma_g^2 + \sigma_{\hat{g}}^2)$ is some constant related to the diameter of $\Pi$ (denoted as $|\Pi|$),  the model bias $E_\Phi$, and the sampling variance $\sigma_g^2$ and $\sigma_{\hat{g}}^2$ of $g_n$ and $\hat{g}_n$, respectively.
Through Lemma~\ref{lm:IL performance bound} and~\ref{lm:RL performance bound}, this bound directly implies accelerated and bias-free policy performance.
\begin{theorem} \label{th:piccolo averageregret}
	 Suppose $\tilde{l}_n$ is convex\footnotemark and $w_n \geq \Omega(1)$. Then running \piccolo yields $\E[\regret_n(\tilde{l})/w_{1:N}] = O(\frac{C_{\Pi,\Phi} }{\sqrt{N}})$, where  $C_{\Pi,\Phi}  = O(  |\Pi| + E_\Phi + \sigma_g^2 + \sigma_{\hat{g}}^2) = O(1)$.
\end{theorem}
\footnotetext{The convexity assumption is standard, as used in~\cite{duchi2011adaptive,ross2011reduction,kingma2014adam,cheng2018convergence}, which holds for tabular problems as well as some special cases, like continuous-time problems (cf. \cite{cheng2018convergence}).}

\subsection{Comparison}\label{sec:comparison}

\begin{table}[t]
	\caption{Upper bounds of the average regret of different policy optimization algorithms.}
	\label{tb:comparison of bounds}
	\begin{center}
		\begin{small}
			\begin{sc}
				\begin{tabular}{c c } 
					\toprule
					Algorithms &  Upper bounds in Big-O \\
					\midrule
					\piccolo &  $ \frac{1}{\sqrt{N}} \paren{\abs{\Pi} + \sigma_g^2 + \sigma_{\hat g}^2 + E_\Phi }$ \\
					model-free & $\frac{1}{\sqrt{N}} \paren{\abs{\Pi} +  G_g^2 +  \sigma_g^2}$ \\
					model-based & $ \frac{1}{\sqrt{N}} \paren{\abs{\Pi} +  G_{\hat{g}}^2 +  \sigma_{\hat g}^2  } + E_\Phi$ \\
					\dyna & 
					$\frac{1}{\sqrt{2 N}} \paren{ \abs{\Pi} \hspace{-0.5mm} + \hspace{-0.5mm} \frac{1}{2}\paren{G_g^2 + G_{\hat g}^2 + \sigma_g^2 + \sigma_{\hat g}^2} } \hspace{-0.5mm}+ \hspace{-0.5mm}E_\Phi$  \\
					\bottomrule
				\end{tabular}
			\end{sc}
		\end{small}
	\end{center}		
	\vskip -0.1in
	\vspace{-4mm}
\end{table}

To appreciate the advantages of \piccolo, we review several policy optimization algorithms and compare their regret. 
We show that they can be viewed as incomplete versions of \piccolo, which only either
result in accelerated learning \emph{or} are unbiased, but not both (see in Table~\ref{tb:comparison of bounds}).

We first consider the common model-free approach~\cite{sutton2000policy,kakade2002natural,peters2008natural,peters2010relative,  silver2014deterministic,sun2017deeply,cheng2018fast}, i.e. applying the base algorithm with $g_n$. To make the comparison concrete, suppose $\norm{\E[g_n]}_{*}^2 \leq G_g^2$ for some constant $G_g$, where we recall $g_n$ is the sampled true gradient. 
As the model-free approach is equivalent to setting $\hat{g}_n = 0$ in \piccolo, by Theorem~\ref{th:piccolo} (with $e_n = g_n$),  the constant $C_\Pi$ in Theorem~\ref{th:piccolo averageregret} would become $O(|\Pi|+G_g^2 + \sigma_g^2)$. In other words, {\piccolo}ing the base algorithm improves the constant factor from $G_g^2$ to $\sigma_{\hat{g}}^2 + E_{\Phi}$. 
Therefore, while the model-free approach is bias-free, its convergence can be further improved by \piccolo, as long as the models $\{\Phi_n\}$ are reasonably accurate  on average.\footnote{It can be shown that if the model is learned online with a no-regret algorithm, it would perform similarly to the best model in the hindsight (cf. Appendix~\ref{app:model learning})}

Next we consider the pure model-based approach with a model that is potentially learned online~\cite{jacobson1970differential,todorov2005generalized,deisenroth2011pilco,pan2014probabilistic,levine2013guided,wen2018dual}. 
As this approach is equivalent to only performing the Prediction Step\footnote{These algorithms can be realized by the fixed-point formulation of the Prediction Step with (arbitrarily small) regularization.}, its performance suffers from any modeling error. 
Specifically, 
suppose $\norm{\E[\hat{g}_n]}_{*}^2 \leq G_{\hat{g}}^2$ for some constant $G_{\hat{g}}$. One can show that the bound in Theorem~\ref{th:piccolo averageregret} would become 
$O((|\Pi| + G_{\hat{g}}^2+\sigma_{\hat{g}}^2)/\sqrt{N} + E_{\Phi})$, introducing a constant bias in $O(E_{\Phi})$.

A hybrid heuristic to combine the model-based and model-free updates
is \dyna~\cite{sutton1991dyna,sutton2012dyna}, which interleaves the two steps during policy optimization. This is equivalent to applying $g_n$, instead of the error $e_n$, in the Correction Step of \piccolo. Following a similar analysis as above, one can show that the convergence rate in Theorem~\ref{th:piccolo averageregret} would become $O((|\Pi| + G^2 + \sigma^2)/\sqrt{2 N} + E_{\Phi})$, where $G^2 = \half(G_{g}^2 + G_{\hat{g}}^2)$ and $\sigma^2 = \half(\sigma_g^2 + \sigma_{\hat{g}}^2)$. Therefore, \dyna is effectively twice as fast as the pure model-free approach when the model is accurate.
However, it would eventually suffer from the performance bias due model error, as reflected in the term $E_{\Phi}$. We will demonstrate this property experimentally in Figure~\ref{fig:simple exps}.

Finally, we note that the idea of using $\Phi_n$ as control variate~\cite{chebotar2017combining,grathwohl2017backpropagation,papini2018stochastic} is orthogonal to the setups considered above, and it can be naturally combined with \piccolo. For example, we can also use $\Phi_n$ to compute a better sampled gradient $g_n$ with smaller variance (line 5 of Algorithm~\ref{alg:piccolo}). This would improve $\sigma_g^2$ in the bounds of \piccolo to a smaller $\tilde\sigma_g^2$, the size of reduced variance.

\vspace{-2mm}
\section{Experiments}

\def\last{\textsc{last}\xspace}
\def\replay{\textsc{replay}\xspace}
\def\sim{\textsc{TrueDyn}\xspace}
\def\biased{\textsc{BiasedDyn}\xspace}
\def\biasedE{\textsc{BiasedDyn}0.8\xspace}
\def\fp{\textsc{fp}\xspace}
\def\rand{\textsc{random}\xspace}
\def\adv{\textsc{adversarial}\xspace}
\begin{figure} 	
	\centering
	\begin{subfigure}{.23\textwidth}
		\includegraphics[width=\textwidth]{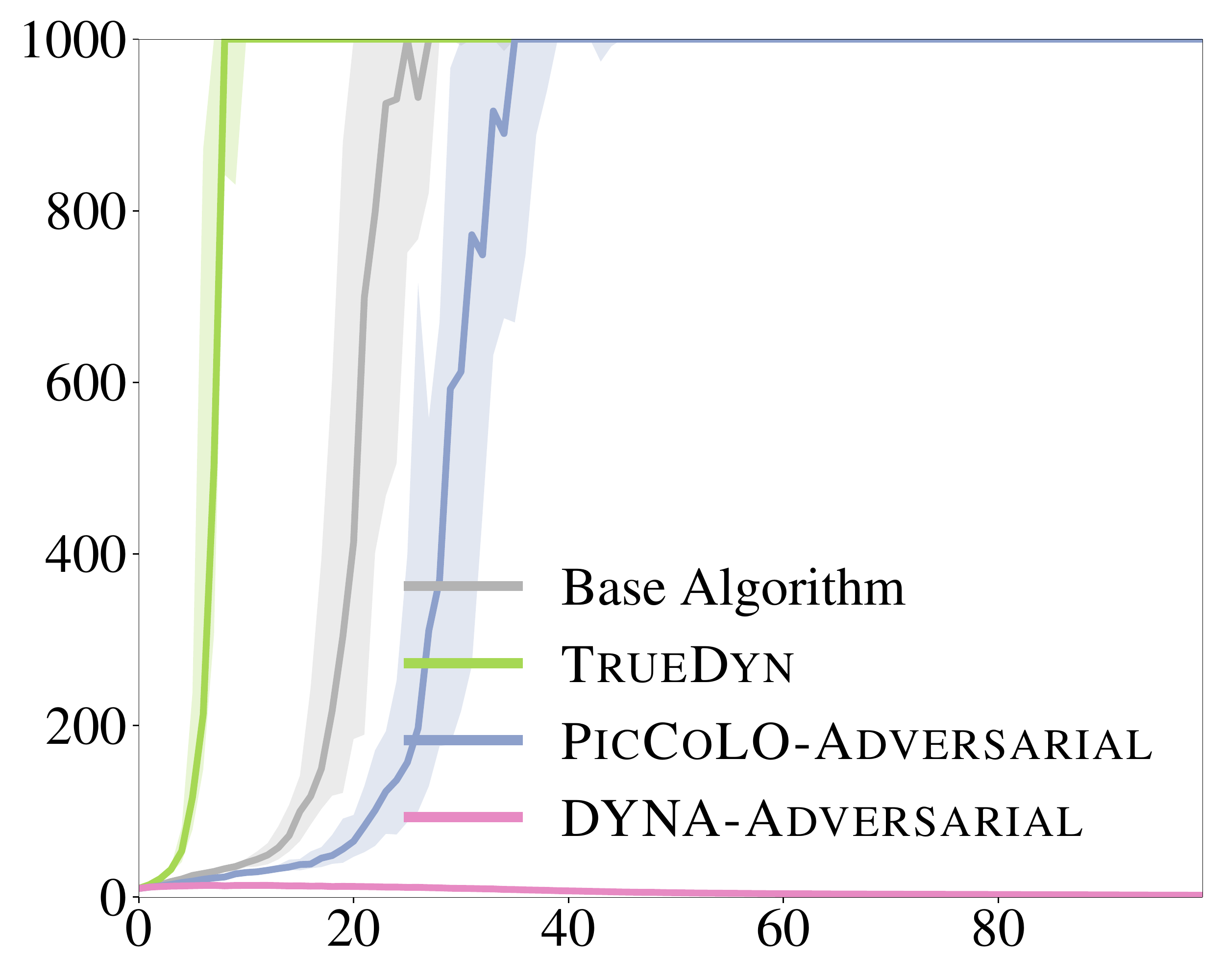}
		\caption{Adversarial model}
	\end{subfigure}
		\begin{subfigure}{.23\textwidth}
	\includegraphics[width=\textwidth]{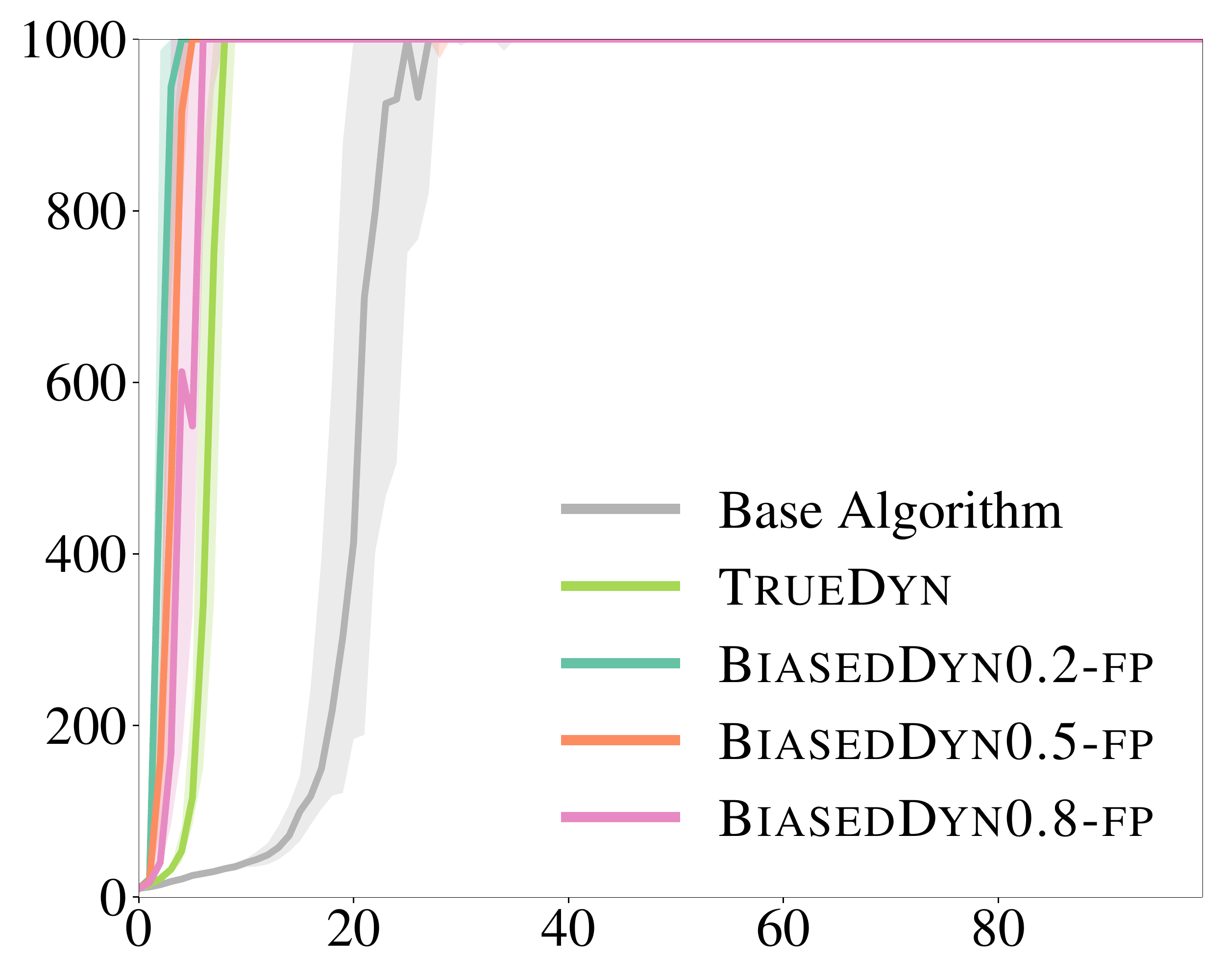}
	\caption{Different model fidelities}
\end{subfigure}
	\caption{ 
	Performance of \piccolo with different predictive models. 
	$x$ axis is iteration number and $y$ axis is sum of rewards. The curves are the median among 8 runs with different seeds, and the shaded regions account for $25\%$ percentile.
	\adam is used as the base algorithm, and
	the update rule, by default, is $\piccolo$; e.g. \sim in (a) refers to \piccolo with \sim predictive model. 
	(a) Comparison of \piccolo and \dyna with adversarial model.  
	(b) \piccolo with the fixed-point setting~\eqref{eq:fixed-point problem} with dynamics model in different fidelities.
	 \biasedE indicates that the mass of each individual robot link is either increased or decreased by $80\%$ with probability 0.5 respectively. 
	}
	\label{fig:simple exps}
	\vspace{-2mm}
\end{figure}

\begin{figure*} 	
	\centering
	\begin{subfigure}{.24\textwidth}
	\includegraphics[width=\textwidth]{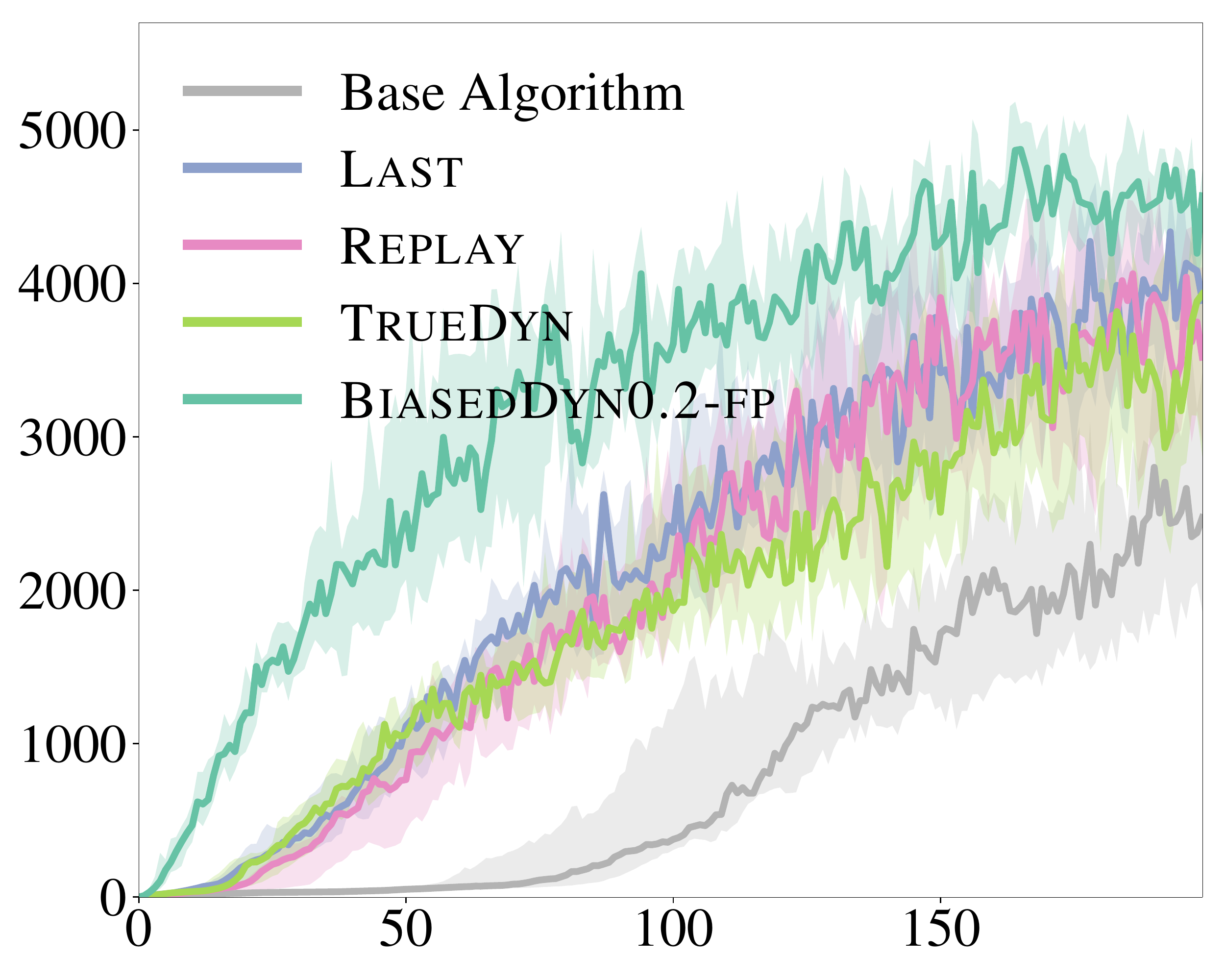}
	\caption{Hopper \adam}
\end{subfigure}
\begin{subfigure}{.24\textwidth}
	\includegraphics[width=\textwidth]{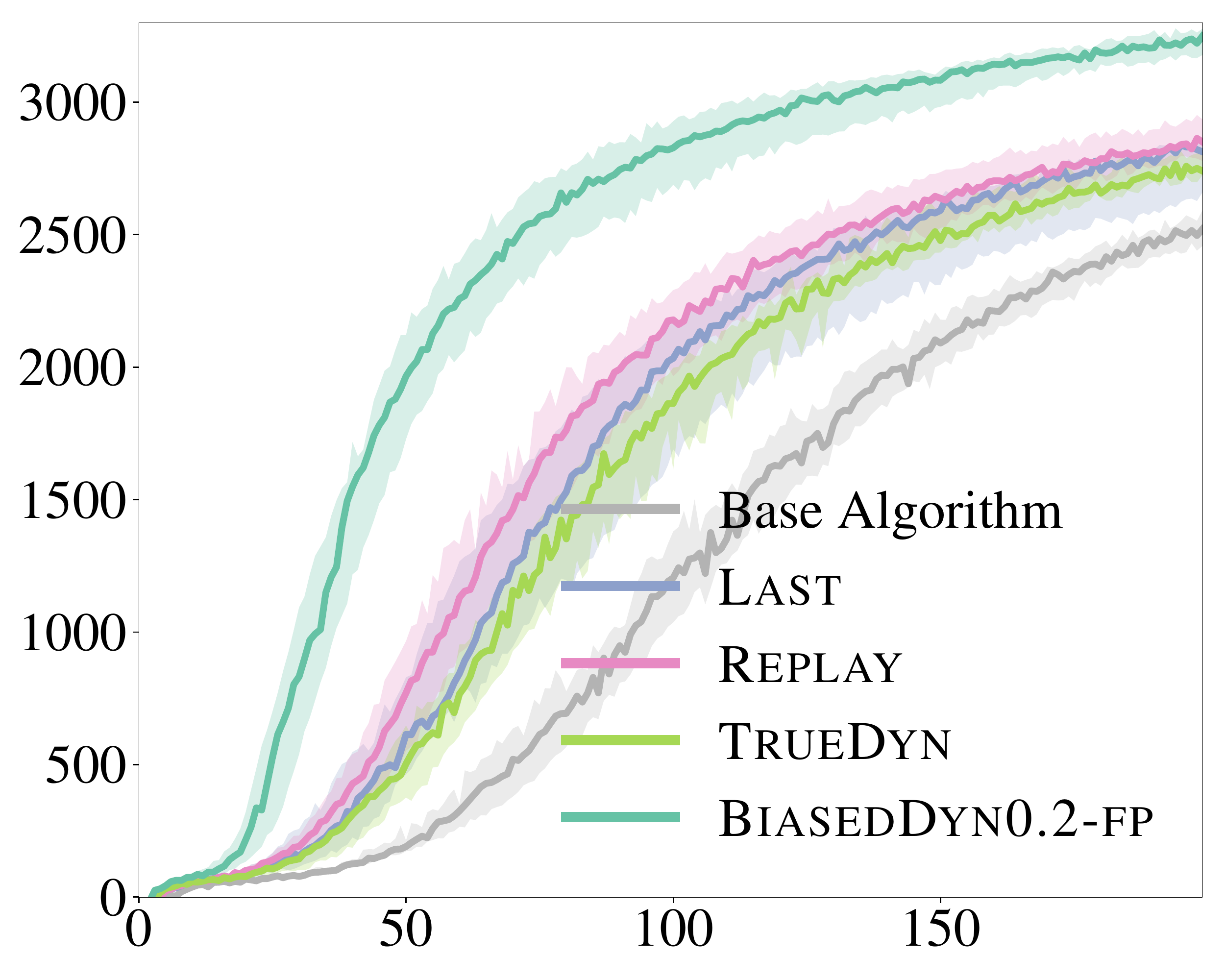}
	\caption{Snake \adam}
\end{subfigure}
	\begin{subfigure}{.24\textwidth}
	\includegraphics[width=\textwidth]{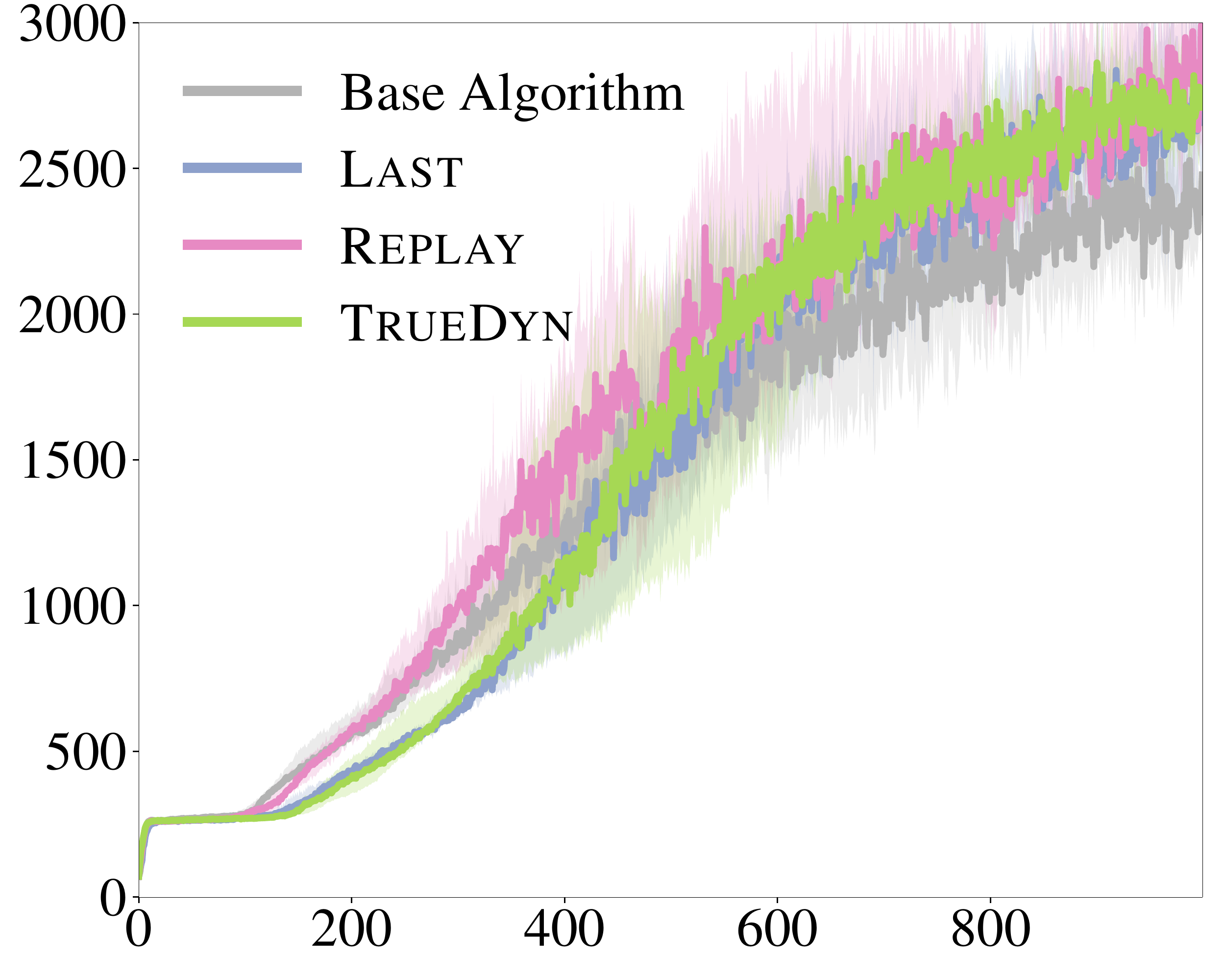}
	\caption{Walker3D \natgrad }
\end{subfigure}
\begin{subfigure}{.24\textwidth}
	\includegraphics[width=\textwidth]{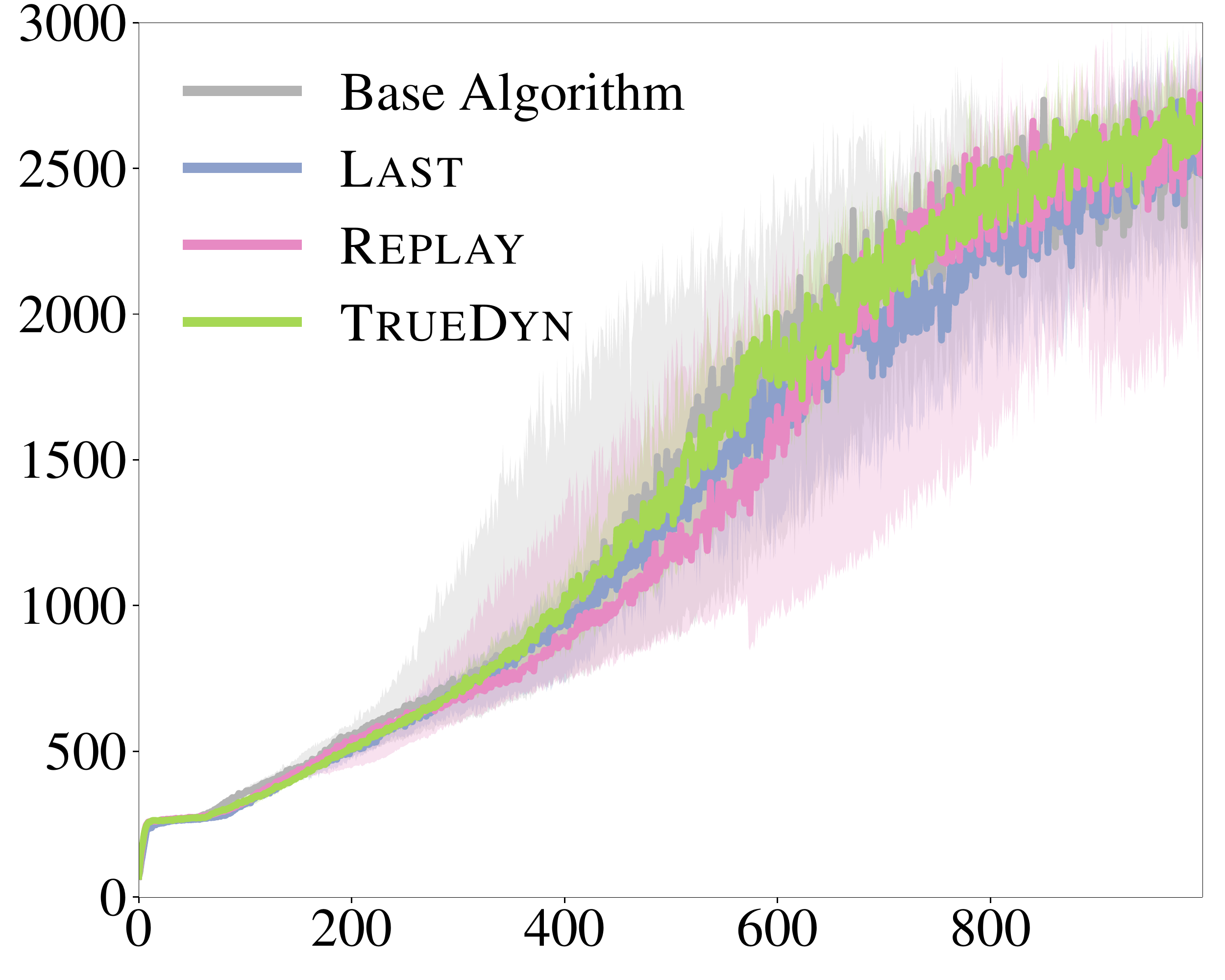}
	\caption{Walker3D \trpo}
\end{subfigure}
\caption{Performance of \piccolo in various tasks. 	
	$x$ axis is iteration number and $y$ axis is sum of rewards.
	The curves are the median among 8 runs with different seeds, and the shaded regions account for $25\%$ percentile. 
}
\label{fig:new_exps}
\vspace{-1mm}
\end{figure*}

We corroborate our theoretical findings with experiments\footnote{The codes are available at \url{https://github.com/gtrll/rlfamily}.} in learning neural network policies to solve robot RL tasks (CartPole, Hopper, Snake, 
and Walker3D) from OpenAI Gym~\citep{brockman2016openai} with the DART physics engine~\citep{Lee2018}\footnote{The environments are defined in DartEnv, hosted at \url{https://github.com/DartEnv}.}.
The aim is to see if \piccolo improves the performance of a base algorithm, even though in these experiments the convexity assumption in the theory does not hold.
We choose several popular first-order mirror descent base algorithms (
\adam~\cite{kingma2014adam}, natural gradient descent \natgrad~\cite{kakade2002natural}, and trust-region optimizer \trpo~\cite{schulman2015trust}). 
We compute $g_n$ by GAE~\cite{schulman2015high}. 
For predictive models, we consider off-policy gradients (with the samples of the last iteration \last or a replay buffer \replay) and gradients computed through simulations with the true or biased dynamics models (\sim or \biased).
We will label a model with \fp if $\hat{g}_n$ is determined by the fixed-point formulation~\eqref{eq:fixed-point problem}\footnote{In implementation, we solve the corresponding optimization problem with a few number of iterations. For example, $\biased$-$\fp$ is aporoximatedly solved with $5$ iterations.}; otherwise, $\hat{g}_n = \Phi_n(\hat{\pi}_n)$.
Please refer to Appendix~\ref{app:exp details} for the details.

In~\cref{fig:simple exps}, we first use CartPole to study Theorem~\ref{th:piccolo averageregret}, which suggests that \piccolo is unbiased and improves the performance when the prediction is accurate. 
Here we additionally consider 
an extremely bad model, \adv, that predicts the gradients adversarially.\footnote{We set 
	$\hat{g}_{n+1} = -  \left(\max_{m=1,\dots,n}\norm{g_m}/\norm{g_n} \right)g_n$.}
Figure~\ref{fig:simple exps} (a) illustrates the performance of \piccolo and \dyna, when \adam is chosen as the base algorithm. 
We observe that \piccolo improves the performance when the model is accurate (i.e. \sim). Moreover, \piccolo is robust to modeling errors. It still converges when the model is adversarially attacking the algorithm,  whereas \dyna fails completely. 
In~\cref{fig:simple exps} (b), we conduct a finer comparison of the effects of different model accuracies (\biased-\fp), when $\hat{g}_n$ is computed using~\eqref{eq:fixed-point problem}.
To realize inaccurate dynamics models to be used in the Prediction step, we change the mass of links of the robot by a certain factor, e.g. \biasedE indicates that the mass of each individual link is either increased or decreased by $80\%$ with probability 0.5, respectively. 
We see that the fixed-point formulation~\eqref{eq:fixed-point problem}, which makes multiple queries of $\Phi_n$ for computing $\hat{g}_n$, performs much better than the heuristic of setting $\hat{g}_n = \Phi(\hat{\pi}_n)$, even when the latter is using the true model (\sim).
Overall, we see \piccolo with \biased-\fp
is able to accelerate learning, though with a degree varying with model accuracies; but even for models with a large bias, it still converges unbiasedly, as we previously observed in~\cref{fig:simple exps} (a),

In~\cref{fig:new_exps}, we study the performance of \piccolo in a range of environments. In general, we find that \piccolo indeed improves the performance\footnote{Note that different base algorithms are not directly comparable, as further fine-tuning of step sizes is required.} though the exact degree depends on how $\hat{g}_n$ is computed. 
In~\cref{fig:new_exps} (a) and (b), we show the results of using \adam as the base algorithm. We observe that, while setting $\hat{g}_n = \Phi_n(\hat{\pi}_n)$ is already an effective heuristic, the performance of \piccolo can be further and largely improved if we adopt the fixed-point strategy in~\eqref{eq:fixed-point problem}, as the latter allows the learner to take more globally informed update directions. Finally, to demonstrate the flexibility of the proposed framework, we also ``\piccolo'' two other base algorithms, \natgrad and \trpo, in~\cref{fig:new_exps} (c) and (d), respectively. 
The complete set of experimental results can be found in Appendix~\ref{app:exp details}.

\vspace{-1mm}
\section{Conclusion}

\piccolo is a general reduction-based framework for solving predictable online learning problems.
It can be viewed as an automatic strategy for generating new algorithms that can leverage prediction to accelerate convergence. Furthermore, \piccolo uses the Correction Step to recover from the mistake made in the Prediction Step, so the presence of modeling errors does not bias convergence, as we show in both the theory and experiments. The design of \piccolo leaves open the question of how to design good predictive models. While \piccolo is robust against modeling error, 
the accuracy of a predictive model can affect its effectiveness. 
\piccolo only improves the performance when the model can make non-trivial predictions.
In the experiments, we found that off-policy and simulated gradients are often useful, but they are not perfect. It would be interesting to see whether a predictive model that is trained to directly minimize the prediction error can further help policy learning.
Finally, we note that, despite the focus of this paper on policy optimization, \piccolo can naturally be applied to other optimization and learning problems. 

\vspace{-1mm}
\section*{Acknowledgements}
This research is supported in part by NSF NRI 1637758 and NSF CAREER 1750483.

\bibliography{ref} 
\bibliographystyle{icml2019}
\clearpage

\input{appendix}

\end{document}

%% file: shortcuts.tex
\usepackage{amsmath}
\usepackage{amsfonts}
\usepackage{amssymb}
\usepackage{amsthm}
\usepackage{bm}
\usepackage{bbm}
\usepackage{mathtools}
\usepackage{enumitem}
\usepackage{thmtools,thm-restate}
\usepackage{algorithm}
\usepackage{algorithmic}
\usepackage{natbib}
\usepackage[capitalise,noabbrev]{cleveref}
\usepackage{color}
\usepackage[dvipsnames]{xcolor}
\usepackage{graphicx}
\usepackage{comment}

\theoremstyle{plain}
\newtheorem{lemma}{Lemma}
\newtheorem{theorem}{Theorem}
\newtheorem{proposition}{Proposition}

\theoremstyle{definition}
\newtheorem{assumption}{Assumption}

\theoremstyle{remark}

\def\DD{\mathcal{D}}

\def\KK{\mathcal{K}}

\def\PP{\mathcal{P}}\def\RR{\mathcal{R}}

\def\Abb{\mathbb{A}}
\def\Ebb{\mathbb{E}}

\def\Sbb{\mathbb{S}}

\def\diag{\mathrm{diag}}

\def\t{\top}
\def\*{\star}

\newcommand{\norm}[1]{ \| #1 \|  }
\newcommand{\abs}[1]{ \left| #1 \right|  }

\newcommand{\lr}[2]{ \left\langle #1, #2 \right\rangle}
\DeclareMathOperator*{\argmin}{arg\,min}

\def\regret{\textrm{Regret}}

\newcommand{\E}{\Ebb}

\usepackage{etex,etoolbox}
\usepackage{environ}

\makeatletter
\providecommand{\@fourthoffour}[4]{#4}
\newcommand\fixstatement[2][\proofname\space of]{%
	\ifcsname thmt@original@#2\endcsname
	\AtEndEnvironment{#2}{%
		\xdef\pat@label{\expandafter\expandafter\expandafter
			\@fourthoffour\csname thmt@original@#2\endcsname\space\@currentlabel}%
		\xdef\pat@proofof{\@nameuse{pat@proofof@#2}}%
	}%
	\else
	\AtEndEnvironment{#2}{%
		\xdef\pat@label{\expandafter\expandafter\expandafter
			\@fourthoffour\csname #1\endcsname\space\@currentlabel}%
		\xdef\pat@proofof{\@nameuse{pat@proofof@#2}}%
	}%
	\fi
	\@namedef{pat@proofof@#2}{#1}%
}

\globtoksblk\prooftoks{1000}
\newcounter{proofcount}

\NewEnviron{proofatend}{%
	\edef\next{%
		\noexpand\begin{proof}[\pat@proofof\space\pat@label]%
			\unexpanded\expandafter{\BODY}}%
		\global\toks\numexpr\prooftoks+\value{proofcount}\relax=\expandafter{\next\end{proof}}
	\stepcounter{proofcount}}

\def\printproofs{%
	\count@=\z@
	\loop
	\the\toks\numexpr\prooftoks+\count@\relax
	\ifnum\count@<\value{proofcount}%
	\advance\count@\@ne
	\repeat}
\makeatother

\fixstatement{lemma}
\fixstatement{theorem}
\fixstatement{proposition}
\fixstatement{corollary}

%% file: appendix.tex
\appendix
\allowdisplaybreaks
\onecolumn

\section{Relationship between \piccolo and Existing Algorithms} \label{app:relationship}

We discuss how the framework of \piccolo unifies existing online learning algorithms and provides their natural adaptive generalization.
To make the presentation clear, we summarize the effective update rule of \piccolo when the base algorithm is mirror descent
\begin{align}  \label{eq:piccolo md}
\begin{split}
\pi_{n} &= \argmin_{\pi \in \Pi} \lr{w_n \hat{g}_n}{\pi} +  B_{R_{n-1}}(\pi||\hat{\pi}_n) \\
\hat{\pi}_{n+1} &= \argmin_{\pi \in \Pi} \lr{w_n e_n}{\pi} +  B_{R_n}(\pi||\pi_n)  
\end{split}
\end{align}
and that when the base algorithm is FTRL,
\begin{align} \label{eq:piccolo ftrl}
\begin{split}
\pi_{n} &= \argmin_{\pi \in \Pi} \lr{w_n \hat{g}_n}{\pi} + \sum_{m=1}^{n-1} \lr{w_m g_m}{\pi} + B_{r_{m}}(\pi||\pi_m)  \\
\hat{\pi}_{n+1} &= \argmin_{\pi \in \Pi} \lr{w_n e_n}{\pi} + B_{r_{n}}(\pi||\pi_n)   + \lr{w_n \hat{g}_n}{\pi} + \sum_{m=1}^{n-1} \lr{w_m g_m}{\pi} + B_{r_{m}}(\pi||\pi_m)  
\end{split}
\end{align}
Because $e_n = g_n - \hat{g}_n$, \piccolo with FTRL  exactly matches the update rule (\mobil) proposed by \citet{cheng2018fast}
\begin{align}  \label{eq:mobil}
\begin{split}
\pi_{n} &= \argmin_{\pi \in \Pi} \lr{w_n \hat{g}_n}{\pi} + \sum_{m=1}^{n-1} \lr{w_m g_m}{\pi} + B_{r_{m}}(\pi||\pi_m)   \\
\hat{\pi}_{n+1} &= \argmin_{\pi \in \Pi} \sum_{m=1}^{n} \lr{w_m g_m}{\pi} + B_{r_{m}}(\pi||\pi_m)  
\end{split}
\end{align}

As comparisons, we consider existing two-step update rules, which in our notation can be written as follows:
\begin{itemize}
	\item Extragradient descent~\citep{korpelevich1976extragradient}, mirror-prox~\citep{nemirovski2004prox,juditsky2011solving} or optimistic mirror descent \citep{chiang2012online,rakhlin2013online}
	\begin{align}  \label{eq:mirror prox}
	\begin{split}
	\pi_{n} &= \argmin_{\pi \in \Pi} \lr{\hat{g}_n}{\pi} +  B_{R}(\pi||\hat{\pi}_n) \\
	\hat{\pi}_{n+1} &= \argmin_{\pi \in \Pi} \lr{g_n}{\pi} +  B_{R}(\pi||\hat{\pi}_n)  
	\end{split}
	\end{align}

	\item FTRL-with-Prediction/optimistic FTRL \citep{rakhlin2013online}
	\begin{align} \label{eq:ftrl with prediction}
	\pi_{n} &= \argmin_{\pi \in \Pi} R(\pi) + \lr{\hat{g}_n}{\pi} + \sum_{m=1}^{n-1} \lr{w_m g_m}{\pi}  
	\end{align}
\end{itemize}

Let us first review the previous update rules.
Originally extragradient descent~\citep{korpelevich1976extragradient} and mirror prox~\citep{nemirovski2004prox,juditsky2011solving} were proposed to solve VIs (the latter is an extension to consider general Bregman divergences).
As pointed out by \citet{cheng2019accelerating}, when applied to an online learning problem, these algorithms effectively assign $\hat{g}_n$ to be the online gradient as if the learner plays a decision at $\hat{\pi}_n$. On the other hand, in the online learning literature, optimistic mirror descent~\citep{chiang2012online} was proposed to use $\hat{g}_n = g_{n-1}$. Later \cite{rakhlin2013online} generalized it to use some arbitrary sequence $\hat{g}_n$, and provided a FTRL version update rule in~\eqref{eq:ftrl with prediction}. However, it is unclear in \citep{rakhlin2013online} where the prediction $\hat{g}_n$ comes from in general, though they provide an example in the form of learning from experts.

Recently \citet{cheng2018fast} generalized the FTRL version of these ideas to design \mobil, which introduces extra features 1) use of weights 2) non-stationary Bregman divergences (i.e. step size) and 3) the concept of predictive models ($\Phi_n \approx \nabla l_n$). The former two features are important to speed up the convergence rate of IL. With predictive models, they propose a conceptual idea (inspired by Be-the-Leader) which solves for $\pi_n$ by the VI of finding $\pi_n$ such that 
\begin{align} \label{eq:mobil vi}
\lr{w_n \Phi_n(\pi_n) + \sum_{m=1}^{n} w_m g_m }{\pi' - \pi_n} \geq 0 \qquad \forall \pi' \in \Pi
\end{align}
and a more practical version~\eqref{eq:mobil} which sets $\hat{g}_n = \Phi_n(\pi_n)$. Under proper assumptions, they prove that the practical version achieves the same rate of non-asymptotic convergence as the conceptual one, up to constant factors.

\piccolo unifies and generalizes the above update rules.
We first notice that when the weight is constant, the set $\Pi$ is unconstrained, and the Bregman divergence is constant, \piccolo with mirror descent in \eqref{eq:piccolo md} is the same as \eqref{eq:mirror prox}, i.e., 
\begin{align*}
\hat{\pi}_{n+1} &= \argmin_{\pi \in \Pi} \lr{e_n}{\pi} +  B_{R}(\pi||\pi_n)  \\
&= \argmin_{\pi \in \Pi} \lr{e_n}{\pi} +  R(\pi) - \lr{\nabla R(\pi_n)}{\pi}\\
&= \argmin_{\pi \in \Pi} \lr{g_n - \hat{g}_n}{\pi} +  R(\pi) - \lr{\nabla R(\hat{\pi}_n)- \hat{g}_n }{\pi}\\
&= \argmin_{\pi \in \Pi} \lr{g_n}{\pi} +  R(\pi) - \lr{\nabla R(\hat{\pi}_n)}{\pi}\\
&= \argmin_{\pi \in \Pi} \lr{g_n}{\pi} +  B_R(\pi||\hat{\pi}_n)
\end{align*}
Therefore, \piccolo with mirror descent includes previous two-step algorithms with proper choices of $\hat{g}_n$. On the other hand, we showed above that \piccolo with FTRL~\eqref{eq:piccolo ftrl} recovers exactly~\eqref{eq:mobil}. 

\piccolo further generalizes these updates in two important aspects. First, it provides a systematic way to make these mirror descent and FTRL algorithms \emph{adaptive}, by the reduction that allows reusing existing adaptive algorithm designed for adversarial settings. By contrast, all the previous update schemes discussed above (even \mobil) are based on constant or pre-scheduled Bregman divergences, which requires the knowledge of several constants of problem properties that are usually unknown in practice. The use of adaptive schemes more amenable to hyperparameter tuning in practice.

Second, \piccolo generalize the use of predictive models from the VI formulation in~\eqref{eq:mobil vi} to the \emph{fixed-point} formulation in~\eqref{eq:fixed-point problem}. One can show that when the base algorithm is FTRL and we remove the Bregman divergence\footnote{Originally the conceptual \mobil algorithm is based on the assumption that $l_n$ is strongly convex and therefore does not require extra Bregman divergence. Here \piccolo with FTRL provides a natural generalization to online convex problems.}, \eqref{eq:fixed-point problem} is the same as~\eqref{eq:mobil vi}.
In other words, \eqref{eq:mobil vi} essentially can be viewed as a mechanism to find $\hat{g}_n$ for~\eqref{eq:mobil}. But importantly, the fixed-point formulation is method agnostic and therefore applies to also the mirror descent case. In particular, in Section~\ref{sec:model loss}, we point out that when $\Phi_n$ is a gradient map, the fixed-point problem reduces to finding a stationary point\footnote{Any stationary point will suffice.} of a non-convex optimization problem. This observation makes implementation of the fixed-point idea much easier and more stable in practice (as we only require the function associated with $\Phi_n$ to be lower bounded to yield a stable problem).

\section{Proof of Lemma~\ref{lm:RL performance bound}} \label{app:proof of RL performance bound}

Without loss of generality we suppose $w_1 =1$ and $J(\pi) \geq 0 $ for all $\pi$. 
And we assume the weighting sequence $\{w_n\}$ satisfies, for all $n \geq m \geq 1$ and $k \geq 0 $, $\frac{w_{n+k}}{w_{n}} \leq  \frac{w_{m+k}}{w_{m}}$. This means $\{w_n\}$ is an non-decreasing sequence and it does not grow faster than exponential (for which $\frac{w_{n+k}}{w_{n}} = \frac{w_{m+k}}{w_{m}}$). For example, if $w_n = n^p$ with $p \geq 0$, it easy to see that
\begin{align*}
\frac{(n+k)^p}{n^p} \leq \frac{(m+k)^p}{m^p}
\impliedby \frac{n+k}{n} \leq \frac{m+k}{m} \impliedby
\frac{k}{n} \leq \frac{k}{m}
\end{align*}
For simplicity, let us first consider the case where $l_n$ is deterministic. Given this assumption, we bound the performance in terms of the weighted regret below. For $l_n$ defined in~\eqref{eq:RL online loss}, we can write 
\begin{align*}
&\sum_{n=1}^{N} w_n J(\pi_n) \\
&=  \sum_{n=1}^{N} w_n J(\pi_{n-1})  +   w_n \E_{d_{\pi_n}} \E_{\pi_n} [ A_{\pi_{n-1}} ] \\
&=  \sum_{n=1}^{N}  w_n J(\pi_{n-1}) +  w_n l_n(\pi_n) \\
&= w_1 J(\pi_0) + \sum_{n=1}^{N-1} w_{n+1}  J(\pi_{n}) + \sum_{n=1}^N  w_n l_n(\pi_n)\\
&= w_1 J(\pi_0) + \sum_{n=1}^{N-1}   w_{n+1}   J(\pi_{n-1}) + \sum_{n=1}^{N-1} w_{n+1} l_n(\pi_n)  +   \sum_{n=1}^N w_n l_n(\pi_n)  \\ 
&= (w_1 + w_2 )  J(\pi_0) + \sum_{n=1}^{N-2}   w_{n+2}   J(\pi_{n}) + \sum_{n=1}^{N-1} w_{n+1} l_n(\pi_n)  +   \sum_{n=1}^N w_n l_n(\pi_n)  \\ 
&= w_{1:N} J(\pi_0) + \left( w_{N} l_1(\pi_1)  + \sum_{n=1}^{2} w_{n+N-2} l_n(\pi_n) + \dots +  \sum_{n=1}^{N-1} w_{n+1} l_n(\pi_n)  +   \sum_{n=1}^N w_n l_n(\pi_n) \right) \\
&= w_{1:N} J(\pi_0) + \left( w_{N} l_1(\pi_1)  +  \sum_{n=1}^{2} \frac{w_{n+N-2}}{w_n} w_{n} l_n(\pi_n) + \dots +  \sum_{n=1}^{N-1} \frac{w_{n+1}}{w_n}  w_n l_n(\pi_n)  +   \sum_{n=1}^N w_n l_n(\pi_n) \right) \\
&\leq w_{1:N} J(\pi_0) + \left( w_{N} l_1(\pi_1)  +  \frac{w_{N-1}}{w_1}  \sum_{n=1}^{2} w_{n} l_n(\pi_n) + \dots +  \frac{w_{2}}{w_1}  \sum_{n=1}^{N-1}   w_n l_n(\pi_n)  +   \sum_{n=1}^N w_n l_n(\pi_n) \right) \\
&= w_{1:N} J(\pi_0) + \left( w_{N} l_1(\pi_1)  +  w_{N-1}  \sum_{n=1}^{2} w_{n} l_n(\pi_n) + \dots +  w_{2} \sum_{n=1}^{N-1}   w_n l_n(\pi_n)  +   \sum_{n=1}^N w_n l_n(\pi_n) \right)
\end{align*}
where the inequality is due to the assumption on the weighting sequence. 

We can further rearrange the second term in the final expression as 
\begin{align*}
&  w_{N} l_1(\pi_1)  +  w_{N-1}  \sum_{n=1}^{2} w_{n} l_n(\pi_n) + \dots +  w_{2} \sum_{n=1}^{N-1}   w_n l_n(\pi_n)  +   \sum_{n=1}^N w_n l_n(\pi_n)\\
=& w_{N} \left( l_1(\pi_1) - \min_{\pi \in \Pi} l_1(\pi) + \min_{\pi \in \Pi} l_1(\pi) \right) \\
&  +  w_{N-1} \left( \sum_{n=1}^{2} w_{n} l_n(\pi_n) -  \min_{\pi \in \Pi} \sum_{n=1}^{2} w_{n} l_n(\pi) + \min_{\pi \in \Pi} \sum_{n=1}^{2} w_{n} l_n(\pi)  \right) \\
&+ \dots +   \sum_{n=1}^N w_n l_n(\pi_n) -  \min_{\pi \in \Pi} \sum_{n=1}^N w_n l_n(\pi) + \min_{\pi \in \Pi} \sum_{n=1}^N w_n l_n(\pi) \\
=&  \sum_{n=1}^{N} w_{N-n+1} \left( \regret_n (f) + w_{1:n} \epsilon_n (f) \right) 
\end{align*}
where the last equality is due to the definition of \emph{static} regret and $\epsilon_n$.

Likewise, we can also write the above expression in terms of \emph{dynamic regret}
\begin{align*}
&  w_{N} l_1(\pi_1)  +  w_{N-1}  \sum_{n=1}^{2} w_{n} l_n(\pi_n) + \dots +  w_{2} \sum_{n=1}^{N-1}   w_n l_n(\pi_n)  +   \sum_{n=1}^N w_n l_n(\pi_n)\\
=& w_{N} \left( l_1(\pi_1) - \min_{\pi \in \Pi} l_1(\pi) + \min_{\pi \in \Pi} l_1(\pi) \right) \\
&  +  w_{N-1} \left( \sum_{n=1}^{2} w_{n} l_n(\pi_n) -   \sum_{n=1}^{2} w_{n} \min_{\pi \in \Pi} l_n(\pi) +  \sum_{n=1}^{2} \min_{\pi \in \Pi} w_{n} l_n(\pi)  \right) \\
&+ \dots +   \sum_{n=1}^N w_n l_n(\pi_n) -   \sum_{n=1}^N \min_{\pi \in \Pi} w_n l_n(\pi) + \sum_{n=1}^N \min_{\pi \in \Pi}  w_n l_n(\pi) \\
=&  \sum_{n=1}^{N} w_{N-n+1} \left( \regret_n^d (l) + w_{1:n} \epsilon_n^d (l) \right) 
\end{align*}
in which we define the weighted dynamic regret as \begin{align*}
\regret_n^d(l) = \sum_{m=1}^n w_m l_m(\pi_m) -   \sum_{m=1}^n w_m \min_{\pi \in \Pi}  l_m(\pi)
\end{align*}
and an expressive measure based on dynamic regret
\begin{align*}
\epsilon_n^d = \frac{1}{w_{1:n}} \sum_{m=1}^n w_m \min_{\pi \in \Pi}  l_m(\pi) \leq 0 
\end{align*}

For stochastic problems, because $\pi_n$ does not depends on $\tilde{l}_n$, the above bound applies to the performance in expectation. 
Specifically, let $h_{n-1}$ denote all the random variables observed before making decision $\pi_n$ and seeing $\tilde{l}_n$. As $\pi_n$ is made independent of $\tilde{l}_n$, we have, for example, 
\begin{align*}
\E[l_n(\pi_n)  | h_{n-1}] 
&= \E[l_n(\pi_n)  | h_{n-1}] - \E[l_n(\pi_n^*)  | h_{n-1}] + \E[l_n(\pi_n^*)  | h_{n-1}]\\
&= \E[\tilde{l}_n(\pi_n)  | h_{n-1}] - \E[\tilde{l}_n(\pi_n^*)  | h_{n-1}] + \E[l_n(\pi_n^*)  | h_{n-1}]\\
&\leq \E[\tilde{l}_n(\pi_n)   -\min_{\pi \in \Pi}\tilde{l}_n(\pi)  | h_{n-1}] + \E[l_n(\pi_n^*)  | h_{n-1}]
\end{align*}
where $
\pi_n^* = \argmin_{\pi \in \Pi} l_n(\pi)
$. By applying a similar derivation as above recursively, we can extend the previous deterministic bounds to bounds in expectation (for both the static or the dynamic regret case), proving the desired statement.

\section{The Basic Operations of Base Algorithms} \label{app:basic operations}

We provide details of the abstract basic operations shared by different base algorithms. 
In general, the update rule of any base mirror-descent or FTRL algorithm can be represented in terms of the three basic operations
\begin{align} \label{eq:basic ops}
h \gets \update(h, H, g, w), \qquad
H \gets \adapt(h, H, g, w), \qquad
\pi \gets \project(h, H)
\end{align}
where $\update$ and $\project$ can be identified standardly, for mirror descent as,
\begin{align} \label{eq:mirror descent ops}
\textstyle 
\update(h, H, g, w) =   \argmin_{\pi' \in \Pi} \lr{wg}{\pi'} +  B_H(\pi|| h) , \qquad 
\project(h,H) = h
\end{align}
and for FTRL as,
\begin{align} \label{eq:FTRL ops}
\textstyle
\update(h,H,g,w) = h+wg, \qquad 
\project(h, H) =  \argmin_{\pi' \in \Pi} \lr{h}{\pi'} +  H(\pi') 
\end{align}
We note that in the main text of this paper the operation $\project$ is omitted for simplicity, as it is equal to the identify map for mirror descent. In general, it represents the decoding from the abstract representation of the decision $h$ to $\pi$. The main difference between and $h$ and $\pi$ is that $h$ represents the sufficient information that defines the state of the base algorithm.

While $\update$ and $\project$ are defined standardly, the exact definition of $\adapt$ depends on the specific base algorithm. 
Particularly, $\adapt$ may depend also on whether the problem is weighted, as different  base algorithms may handle weighted problems differently. 
Based on the way weighted problems are handled, we roughly categorize the algorithms (in both mirror descent and FTRL families) into two classes: the \textit{stationary} regularization class and the \emph{non-stationary} regularization class.
Here we provide more details into the algorithm-dependent $\adapt$ operation,
through some commonly used base algorithms as examples.

Please see also Appendix~\ref{app:relationship} for connection between \piccolo and existing two-step algorithms, like optimistic mirror descent~\citep{rakhlin2013optimization}.

\subsection{Stationary Regularization Class}

The $\adapt$ operation of these base algorithms 
features two major functions: 1) a moving-average adaptation  and 2) a step-size adaption. The moving-average adaptation is designed to estimate some statistics $G$ such that $\norm{g}_* = O(G)$ (which is an important factor in regret bounds), whereas the step-size adaptation updates a scalar multiplier $\eta$ according to the weight $w$ to ensure convergence.

This family of algorithms includes basic mirror descent~\citep{beck2003mirror} and FTRL~\citep{mcmahan2010adaptive,mcmahan2017survey} with a scheduled step size, and adaptive algorithms based on moving average e.g. \textsc{RMSprop}~\citep{tieleman2012lecture} \textsc{Adadelta}~\citep{zeiler2012adadelta}, \adam~\citep{kingma2014adam}, \textsc{AMSGrad}~\citep{reddi2018convergence}, and the adaptive \natgrad we used in the experiments. Below we showcase how $\adapt$ is defined using some examples.

\subsubsection{Basic mirror descent~\citep{beck2003mirror}}
We define $G$ to be some constant such that $G \geq \sup \norm{g_n}_*$ and define
\begin{align} \label{eq:step size}
 \eta_n = \frac{ \eta }{1 + c w_{1:n} / \sqrt{n} },
\end{align}
as a function of the iteration counter $n$, 
where $\eta>0$ is a step size multiplier and $c>0$ determines the decaying rate of the step size. The choice of hyperparameters $\eta$, $c$ pertains to how far the optimal solution is from the initial condition, which is related to the size of $\Pi$.
In implementation, $\adapt$ updates the iteration counter $n$ and updates the multiplier $\eta_n$ using $w_n$ in~\eqref{eq:step size}.

Together $(n, G, \eta_n)$ defines $H_n= R_n $ in the mirror descent update rule~\eqref{eq:mirror descent} through setting $R_n = \frac{G}{\eta_n} R$, where $R$ is a strongly convex function. That is, we can write~\eqref{eq:mirror descent} equivalently as 
\begin{align*} 
\textstyle
\pi_{n+1} 
		  &= \argmin_{\pi \in \Pi} \lr{w_n g_n}{\pi} +  \frac{G}{\eta_n} B_R(\pi || \pi_n)  \\
&= \argmin_{\pi \in \Pi} \lr{w_n g_n}{\pi} +  B_{H_n}(\pi || \pi_n) \\
		  &= \update(h_n, H_n, g_n, w_n) 
\end{align*}
When the weight is constant (i.e. $w_n =1$), we can easily see this  update rule is equivalent to the classical mirror descent with a step size $\frac{\eta/G }{1+c\sqrt{n}}$, which is the optimal step size~\citep{mcmahan2017survey}. 
For general $w_n = \Theta(n^p)$ with some $p>-1$, it can viewed as having an effective step size $\frac{w_n \eta_n}{G} = O(\frac{1}{G \sqrt{n}})$, which is optimal in the weighted setting. The inclusion of the constant $G$ makes the algorithm invariant to the scaling of  loss functions. But as the same $G$ is used across all the iterations, the basic mirror descent is conservative. 

\subsubsection{Basic FTRL~\citep{mcmahan2017survey}}

We provide details of general FTRL
\begin{align}\label{eq:FTRL}
\pi_{n+1} = \argmin_{\pi \in \Pi} \sum_{m=1}^n \lr{g_m}{\pi} + B_{r_m}(\pi||\pi_m)
\end{align}
where $B_{r_m}(\cdot || \pi_m)$ is a Bregman divergence centered at $\pi_m$.

We define, in the $n$th iteration, $h_n$, $H_n$, and $\project$ of FTRL in~\eqref{eq:FTRL ops} as
\begin{align*}
h_n = \sum_{m=1}^{n} w_m g_m, \qquad 
H_n(\pi) = \sum_{m=1}^{n}  B_{r_m} (\pi||\pi_n), \qquad 
\project(h, H) =  \argmin_{\pi' \in \Pi} \lr{h}{\pi'} +  H(\pi') 
\end{align*}
Therefore, we can see that $\pi_{n+1} = \project(h_n, H_n)$ indeed gives the update~\eqref{eq:FTRL}: 
\begin{align*}
\pi_{n+1} &= \project(h_n, H_n) \\
&= \project(\sum_{m=1}^{n} w_m g_m  , \sum_{m=1}^{n}  B_{r_m} (\pi||\pi_n)) \\
&=  \argmin_{\pi \in \Pi} \sum_{m=1}^{n}  \lr{ w_m g_m }{\pi} + B_{r_m}(\pi|| \pi_m)
\end{align*}

For the basic FTRL, the $\adapt$ operator is similar to  the basic mirror descent, which uses a constant $G$ and updates the memory $(n, \eta_n)$  using~\eqref{eq:step size}. The main differences are  how $(G, \eta_n)$ is mapped to $H_n$ and that the basic FTRL updates $H_n$ also using $h_n$ (i.e. $\pi_n$). 
Specifically, it performs $H_n \gets \adapt(h_n, H_{n-1}, g_n, w_n)$ through the following: 
\begin{align*}
H_n(\cdot) = H_{n-1}(\cdot) + B_{r_n}(\cdot||\pi_n)
\end{align*}
where  following~\citep{mcmahan2017survey} we set 
\begin{align*}
	B_{r_n}(\pi||\pi_n) = G(\frac{1}{\eta_n} - \frac{1}{\eta_{n-1}}) B_R(\pi||\pi_n)
\end{align*}
and $\eta_n$ is updated using some scheduled rule.

One can also show that the choice of $\eta_n$ scheduling in~\eqref{eq:step size} leads to an optimal regret. When the problem is uniformly weighted (i.e. $w_n =1$), this gives exactly the update rule in~\citep{mcmahan2017survey}. For general $w_n = \Theta(n^{p})$ with $p>-1$, a proof of optimality can be found, for example, in the appendix of~\citep{cheng2019accelerating}. 

\subsubsection{\adam~\citep{kingma2014adam} and \textsc{AMSGrad}~\citep{reddi2018convergence}}
As a representing mirror descent algorithm that uses moving-average estimates, \adam keeps in the memory of the statistics of the first-order information that is provided in $\update$ and $\adapt$. Here we first review the standard description of \adam and then show how it is summarized in
\begin{align} \tag{\ref{eq:general adaptive scheme}}
\begin{split}
H_n = \adapt(h_n, H_{n-1},g_n, w_n), 
\qquad
h_{n+1} =  \update(h_n, H_n, g_n, w_n) 
\end{split}
\end{align}
using properly constructed  $\update$, $\adapt$, and $\project$ operations.

The update rule of \adam proposed by~\citet{kingma2014adam} is originally written as, for $n\geq 1$,\footnote{We shift the iteration index so it conforms with our notation in online learning, in which $\pi_1$ is the initial policy before any update.} 
\begin{align} \label{eq:adam equations}
\begin{split}
m_{n} &= \beta_1 m_{n-1} + (1-\beta_1) g_n\\
v_{n} &= \beta_2 v_{n-1} + (1-\beta_2) g_n \odot g_n \\
\hat{m}_{n} &= m_{n} / (1 - \beta_1^{n})\\
\hat{v}_{n} &= v_{n} / (1 - \beta_2^{n}) \\
\pi_{n+1} &= \pi_{n} - \eta_n \hat{m}_{n} \oslash (\sqrt{\hat{v}_{n}} + \epsilon )
\end{split}
\end{align}
where $\eta_n > 0$ is the step size, $\beta_1, \beta_2 \in [0,1)$ (default $\beta_1 = 0.9$ and $\beta_2 = 0.999$) are the mixing rate, and $0<\epsilon\ll 1$ is some constant for stability (default $\epsilon = 10^{-8}$), and $m_0=v_0=0$. The symbols $\odot$ and $\oslash$ denote element-wise multiplication and division, respectively. The third and the forth steps are designed to remove the 0-bias due to running moving averages starting from 0.

The above update rule can be written in terms of the three basic operations. First, we define the memories $h_n = (m_n, \pi_n)$ for policy and $(v_n, \eta_n, n)$ for regularization that is defined as
\begin{align} \label{eq:adam's H}
H_n(\pi) =  \frac{1}{2 \eta_n} \pi^\t (\diag(\sqrt{\hat{v}_n} ) + \epsilon I) \pi
\end{align}
where $\hat{v}_n$ is defined in the original \adam equation in~\eqref{eq:adam equations}.

The $\adapt$ operation updates the memory to $(v_n, \eta_n, n)$ 
in the step
\begin{align*}
	H_{n} \gets \adapt(h_n, H_{n-1}, g_n, w_n)
\end{align*}
It updates the iteration counter $n$ and $\eta_n$ in the same way in the basic mirror descent using~\eqref{eq:step size}, and update $v_n$ (which along with $n$ defines $\hat{v}_n$ used in~\eqref{eq:adam's H}) using the original \adam equation in~\eqref{eq:adam equations}.

For $\update$, we slightly modify the definition of $\update$ in~\eqref{eq:mirror descent ops} (replacing $g_n$ with $\hat{m}_n$) to incorporate the moving average and write 
\begin{align} \label{eq:mirror descent update with moving average}
\update(h_n, H_{n}, g_n, w_n) &=  \argmin_{\pi' \in \Pi} \lr{ w_n \hat{m}_n }{\pi'} +  B_{H_n}(\pi'||\pi)
\end{align}
where $m_n$ and $\hat{m}_n$ are defined the same as in the original \adam equations in~\eqref{eq:adam equations}. One can verify that, with these definitions, the update rule in~\eqref{eq:general adaptive scheme} is equivalent to the update rule~\eqref{eq:adam equations}, when the weight is uniform (i.e. $w_n = 1$).

Here the $\sqrt{\hat{v}_n}$ plays the role of $G$ as in the basic mirror descent, which can be viewed as an estimate of the upper bound of $\norm{g_n}_*$. \adam achieves a better performance because a coordinate-wise online estimate is used. With this equivalence in mind, we can easily deduct that using the same scheduling of $\eta_n$ as in the basic mirror descent would achieve an optimal regret (cf.~\citep{kingma2014adam,reddi2018convergence}). We note that \adam may fail to converge in some particular problems due to the moving average~\citep{reddi2018convergence}. \textsc{AMSGrad}~\citep{reddi2018convergence} modifies the moving average and uses strictly increasing estimates. However in practice \textsc{AMSGrad} behaves more conservatively.

For weighted problems, we note one important nuance in our definition above: it separates the weight $w_n$ from the moving average and considers $w_n$ as part of the $\eta_n$ update, because the growth of $w_n$ in general can be much faster than the rate the moving average converges. In other words, the moving average can only be used to estimate a stationary property, not a time-varying one like $w_n$. Hence, we call this class of algorithms, the \emph{stationary} regularization class.

\subsubsection{Adaptive \natgrad}
Given first-order information $g_n$ and weight $w_n$, we consider an update rule based on Fisher information matrix:
\begin{align} \label{eq:natgrad (app)}
	\pi_{n+1} = \argmin_{\pi \in \Pi} \lr{w_n g_n}{\pi} + \frac{\sqrt{\hat{G}_n}}{2 \eta_n} (\pi - \pi_n)^\t F_n (\pi - \pi_n)
\end{align}
where $F_n$ is the Fisher information matrix of policy $\pi_n$~\citep{amari1998natural} and $\hat{G}_n$ is an adaptive multiplier for the step size which we will describe. When $\hat{G}_n=1$, the update in~\eqref{eq:natgrad (app)} gives the standard natural gradient descent update with step size $\eta_n$~\citep{kakade2002natural} . 

The role of $\hat{G}_n$ is to adaptively and \emph{slowly} changes the step size to minimize $\sum_{n=1}^{N} \frac{\eta_n}{\sqrt{G_n}}\norm{g_n}_{F_n,*}^2$, which plays an important part in the regret bound (see Section~\ref{sec:piccolo theories}, Appendix~\ref{app:piccolo regret analysis}, and e.g.~\citep{mcmahan2017survey} for details). Following the idea in \adam, we update $\hat{G}_n$ by setting (with $G_0 = 0$)
\begin{align} \label{eq:adaptive natgrad update (app)}
\begin{split}
G_n &= \beta_2 G_n + (1-\beta_2 )  \frac{1}{2} g_n^\t F_n^{-1} g_n \\
\hat{G}_n &= G_n / (1 - \beta_2^n)
\end{split}
\end{align}
similar to the concept of updating $v_n$ and $\hat{v}_n$ in \adam in~\eqref{eq:adam equations}, and update $\eta_n$ in the same way as in the basic mirror descent using~\eqref{eq:step size}. 
Consequently, this would also lead to a regret like \adam  but in terms of a different local norm. 

The $\update$ operation of adaptive \natgrad is defined standardly in~\eqref{eq:mirror descent} (as used in the experiments). 
The $\adapt$ operation updates $n$ and $\eta_n$ like in \adam and updates $G_n$ through~\eqref{eq:adaptive natgrad update (app)}.

\subsection{Non-Stationary Regularization Class}

The algorithms in the non-stationary regularization class maintains a regularization that is increasing over the number of iterations. Notable examples of this class include \adagrad~\citep{duchi2011adaptive} and \textsc{Online Newton Step}~\citep{hazan2007logarithmic}, and its regularization function is updated by applying BTL in a secondary online learning problem whose loss is an upper bound of the original regret (see~\citep{gupta2017unified} for details). Therefore, 
compared with the previous stationary regularization class, 
the adaption property of $\eta_n$ and $G_n$ exchanges: 
$\eta_n$ here becomes constant and $G_n$ becomes time-varying. 
This will be shown more clearly in the \adagrad example below.
We note while these algorithms are designed to be optimal in the convex,  they are often too conservative (e.g. decaying the step size too fast) for non-convex problems.

\subsubsection{\adagrad}

The update rule of the diagonal version of \adagrad in~\citep{duchi2011adaptive} is given as 
\begin{align} \label{eq:adagrad update}
\begin{split}
G_n &= G_{n-1} + \diag(g_n \odot g_n)\\
\pi_{n+1} &= \argmin_{\pi \in \Pi} \lr{g_n}{\pi} + \frac{1}{2 \eta} (\pi - \pi_n)^\top (\epsilon I  + G_n)^{1/2} (\pi - \pi_n)
\end{split}
\end{align}
where $G_0 = 0$ and $\eta > 0$ is a constant. \adagrad is designed to be optimal for online linear optimization problems. Above we provide the update equations of its mirror descent formulation in~\eqref{eq:adagrad update}; a similar FTRL is also available (again the difference only happens when $\Pi$ is constrained).

In terms of our notation, its $\update$ and $\project$ are defined standardly as in~\eqref{eq:mirror descent ops}, i.e.
\begin{align}
\textstyle
\update(h_n, H_n, g_n, w_n) =   \argmin_{\pi' \in \Pi} \lr{w_n g_n}{\pi'} +  B_{H_n}(\pi'|| \pi_n) 
\end{align}
and its $\adapt$ essentially only updates $G_n$:
\begin{align*}
\adapt(h_n, H_{n-1}, g_n, w_n):  G_n = G_{n-1} + \diag(w_n g_n \odot w_n g_n)
\end{align*} 
where the regularization is defined the updated $G_n$ and the constant $\eta$ as
\begin{align*}
H_{n}(\pi)  =  \frac{1}{2 \eta} \pi^\top (\epsilon I  + G_n)^{1/2} \pi.
\end{align*}
One can simply verify the above definitions of $\update$ and $\adapt$ agrees with~\eqref{eq:adagrad update}.

\section{A Practical Variation of \piccolo} \label{app:practical piccolo}
In Section~\ref{sec:piccolo rules}, we show that, given a base algorithm in mirror descent/FTRL, \piccolo generates a new first-order update rule by recomposing the three basic operations into
\begin{align}
h_{n} &= \update(\hat{h}_{n}, H_{n-1}, \hat{g}_{n}, w_n)  &  \text{[Prediction]} \label{eq:prediction} \\[2.5mm]
\begin{split}
H_{n} &= \adapt(h_n, H_{n-1}, e_n, w_n) \\[-1mm]
\hat{h}_{n+1} &= \update(h_n, H_{n}, e_n, w_n) 
\end{split} & \label{eq:correction} \text{[Correction]} 
\end{align}
where $e_n = g_n - \hat{g}_n$ and $\hat{g}_n$ is an estimate of $g_n$ given by a predictive model $\Phi_n$.

Here we propose a slight variation which introduces another operation $\shift$ inside the Prediction Step. This leads to the new set of update rules: 
\begin{align}
\begin{split}
\hat{H}_{n} &= \shift(\hat{h}_{n}, H_{n-1})  \\
h_{n} &= \update(\hat{h}_{n}, \hat{H}_{n}, \hat{g}_{n}, w_n)  \\
\end{split} &  \text{[Prediction]}
 \label{eq:new prediction} \\[2.5mm]
\begin{split}
H_{n} &= \adapt(h_n, \hat{H}_{n}, e_n, w_n) \\[-1mm]
\hat{h}_{n+1} &= \update(h_n, H_{n}, e_n, w_n) 
\end{split} & \label{eq:new correction} \text{[Correction]} 
\end{align}

The new $\shift$ operator additionally changes the regularization based on $\hat{h}_n$ the current representation of the policy  in the Prediction Step, \emph{independent} of the predicted gradient $\hat{g}_n$ and weight $w_n$. 
The main purpose of including this additional step is to deal with numerical difficulties, such as singularity of $H_n$.
For example, in natural gradient descent, the Fisher information of some policy can be close to being singular along the direction of the gradients that are evaluated at different policies. As a result, in the original Prediction Step of \piccolo, $H_{n-1}$ which is evaluated at $\pi_{n-1}$ might be singular in the direction of $\hat{g}_n$ which is evaluated $\hat{\pi}_n$. 

The new operator $\shift$ brings in an extra degree of freedom to account for such issue. 
Although from a theoretical point of view (cf. Appendix~\ref{app:piccolo regret analysis}) the use of $\shift$ would only increase regrets and should be avoided if possible, in practice, its merits in handling numerical difficulties can out weight the drawback.
Because $\shift$ does not depend on the size of $\hat{g}_n$ and $w_n$, the extra regrets would only be proportional to $O(\sum_{n=1}^{N} \norm{\pi_n - \hat{\pi}_n}_n)$, which can be smaller than other terms in the regret bound (see  Appendix~\ref{app:piccolo regret analysis}).

\section{Example: {\piccolo}ing Natural Gradient Descent} \label{app:example}

We give an alternative example to illustrate how one can use the above procedure to ``\piccolo'' a base algorithm into a new algorithm. Here we consider the adaptive natural gradient descent  rule in Appendix~\ref{app:basic operations} as the base algorithm, which  (given first-order information $g_n$ and weight $w_n$) updates the policy through
\begin{align} \label{eq:natgrad}
\textstyle
\pi_{n+1} = \argmin_{\pi \in \Pi} \lr{w_n g_n}{\pi} + \frac{\sqrt{\hat{G}_n}}{2 \eta_n} (\pi - \pi_n)^\t F_n (\pi - \pi_n)
\end{align}
where $F_n $ is the Fisher information matrix of policy $\pi_n$~\citep{amari1998natural}, $\eta_n$ a scheduled learning rate, and $\hat{G}_n$ is an adaptive multiplier for the step size which we will shortly describe. 
When $\hat{G}_n=1$, the update in~\eqref{eq:natgrad} gives the standard natural gradient descent update with step size $\eta_n$~\citep{kakade2002natural} . 

The role of $\hat{G}_n$ is to adaptively and \emph{slowly} changes the step size to minimize $\sum_{n=1}^{N} \frac{\eta_n}{\sqrt{G_n}}\norm{g_n}_{F_n,*}^2$, which plays an important part in the regret bound (see Section~\ref{sec:piccolo theories}, Appendix~\ref{app:piccolo regret analysis}, and e.g.~\citep{mcmahan2017survey} for details). To this end, we update $\hat{G}_n$ by setting (with $G_0 = 0$)
\begin{align} \label{eq:adaptive natgrad update}
\textstyle
G_n = \beta_2 G_{n-1} + (1-\beta_2 )  \frac{1}{2} g_n^\t F_n^{-1} g_n, \qquad
\hat{G}_n &= G_n / (1 - \beta_2^n)
\end{align}
similar to the moving average update rule in \adam, and update $\eta_n$ in the same way as in the basic mirror descent algorithm (e.g. $\eta_n = O(1/\sqrt{n})$).
As a result, this leads to a similar regret like \adam with $\beta_1 =0$, but in terms of a local norm specified by the Fisher information matrix.

Now, let's see how  to {\piccolo} the adaptive natural gradient descent rule above. 
First, it is easy to see that the adaptive natural gradient descent rule  is an instance of mirror descent (with $h_n = \pi_n$ and $H_n(g) = \frac{\sqrt{\hat{G}_n}}{2\eta_n} g^\t F_n g$), so the $\update$ and $\project$ operations are defined in the standard way, as in Section~\ref{sec:piccolo rules}. The $\adapt$ operation updates the iteration counter $n$, the learning rate $\eta_n$, and updates $\hat{G}_n$ through~\eqref{eq:adaptive natgrad update}.

To be more specific, let us explicitly write out the Prediction Step and the Correction Step of the {\piccolo}ed adaptive natural gradient descent rule in closed form as below: e.g. if $\eta_n = \frac{1}{\sqrt{n}}$, then we can write them as
\begin{align*} 
\text{[Prediction]} &&
\begin{split} 
\pi_{n} &= \textstyle \argmin_{\pi \in \Pi} \lr{w_n \hat{g}_n}{\pi} + \frac{\sqrt{\hat{G}_{n-1}}}{2 \eta_{n-1}} (\pi - \hat{\pi}_{n})^\t F_{n-1} (\pi - \hat{\pi}_{n})
\end{split}\\
\\
\text{[Correction]} && 
\begin{split}
\eta_n &= 1/\sqrt{n} \\
G_n&= \textstyle \beta_2 G_{n-1} + (1-\beta_2 )  \frac{1}{2} g_n^\t F_n^{-1} g_n\\
\hat{G}_n &= G_n / (1 - \beta_2^n)\\
\hat{\pi}_{n+1} &= \textstyle\argmin_{\pi \in \Pi} \lr{w_n e_n}{\pi} + \frac{\sqrt{\hat{G}_n}}{2 \eta_n} (\pi - \pi_n)^\t F_n (\pi - \pi_n)
\end{split}
\end{align*}

\section{Regret Analysis of \piccolo} \label{app:piccolo regret analysis}

The main idea of \piccolo is to achieve optimal performance in predictable online learning problems by \emph{reusing} existing adaptive, optimal first-order algorithms that are designed for adversarial online learning problems.  
This is realized by the reduction techniques presented in this section.

Here we prove the performance of \piccolo in general predictable online learning problems, independent of the context of policy optimization. 
In Appendix~\ref{app:reduction from predictable problem to adversarial problem}, we first show an elegant reduction from predictable problems to adversarial problems. 
Then we prove Theorem~\ref{th:piccolo} in Appendix~\ref{app:regret bounds}, showing how the optimal regret bound for predictable linear problems can be achieved by {\piccolo}ing mirror descent and FTRL algorithms. Note that we will abuse the notation $l_n$ to denote the per-round losses in this general setting.

\subsection{Reduction from Predictable Online Learning to Adversarial Online Learning} \label{app:reduction from predictable problem to adversarial problem}

Consider a predictable online learning problem with per-round losses $\{l_n\}$. Suppose in round $n$, before playing $\pi_n$ and revealing $\l_n$, we have access to some prediction of $l_n$, called $\hat{l}_n$.
In particular, we consider the case where $\hat{l}_n(\pi) = \lr{\hat{g}_n}{\pi}$ for some vector $\hat{g}_n$. 
Running an (adaptive) online learning algorithm designed for the general adversarial setting is not optimal here, as its regret would be in $O( \sum_{n=1}^N \norm{\nabla l_n}_{n,*}^2)$, where $\norm{\cdot}_n$ is some local norm chosen by the algorithm and $\norm{\cdot}_{n,*}$ is its dual norm. Ideally, we would only want to pay for the information that is unpredictable. Specifically, we wish to achieve an optimal regret in  $O( \sum_{n=1}^N \norm{\nabla l_n - \nabla \hat{l}_n}_{n,*}^2)$ instead~\citep{rakhlin2013online}.

To achieve the optimal regret bound yet without referring to specialized, nested two-step algorithms (e.g. mirror-prox~\cite{juditsky2011solving}, optimistic mirror descent~\citep{rakhlin2013optimization}, FTRL-prediction~\cite{rakhlin2013online}), we consider decomposing a \emph{predictable} problem with $N$ rounds into an \emph{adversarial} problem with $2N$ rounds:
\begin{align} \label{eq:problem decomposition}
\sum_{n=1}^{N} l_n(\pi_n) = \sum_{n=1}^{N} \hat{l}_n(\pi_n) + \delta_n(\pi_n) 
\end{align}
where $\delta_n = l_n - \hat{l}_n$.
Therefore, we can treat the predictable problem as a new adversarial online learning problem with a loss sequence $\hat{l}_1, \delta_1, \hat{l}_2, \delta_2, \dots, \hat{l}_N, \delta_N$ and consider solving this new problem with some standard online learning algorithm designed for the adversarial setting. 

Before analysis, we first introduce a new decision variable $\hat{\pi}_n$ and 
denote the decision sequence in this new problem as $\hat{\pi}_1, \pi_1, \hat{\pi}_2, \pi_2, \dots, \hat{\pi}_N, \pi_N$, so the definition of the variables are consistent with that in the problem before.
Because this new problem is unpredictable,  the optimal regret of this new decision sequence is
\begin{align} \label{eq:regret of the 2N problem}
\sum_{n=1}^{N} \hat{l}_n(\hat{\pi}_n) + \delta_n(\pi_n) - \min_{\pi \in \Pi} \sum_{n=1}^{N} \hat{l}_n(\pi) + \delta_n(\pi) = O(\sum_{n=1}^N \norm{\nabla  \hat{l}_n}_{n,*}^2  +  \norm{\nabla \delta_n}_{n+1/2,*}^2)
\end{align}
where the subscript $n+1/2$ denotes the extra round due to the reduction.

At first glance, our reduction does not meet the expectation of achieving regret in $O( \sum_{n=1}^N \norm{\nabla l_n - \nabla \hat{l}_n}_{n,*}^2) = O( \sum_{n=1}^N  \norm{\nabla \delta_n}_{n,*}^2)$. However, we note that the regret  for the new problem is too loose for the regret  of the original problem, which is
\begin{align*}
\sum_{n=1}^{N} \hat{l}_n(\pi_n) + \delta_n(\pi_n) - \min_{\pi \in \Pi} \sum_{n=1}^{N} \hat{l}_n(\pi) + \delta_n(\pi) 
\end{align*}
where the main difference is that originally we care about $ \hat{l}_n(\pi_n)$ rather than $ \hat{l}_n(\hat{\pi}_n)$. Specifically, we can write
\begin{align*}
\sum_{n=1}^{N} l_n(\pi_n) 
&= \sum_{n=1}^{N} \hat{l}_n(\pi_n) + \delta_n(\pi_n)  \\
&= \left( \sum_{n=1}^{N} \hat{l}_n(\hat{\pi}_n) + \delta_n(\pi_n) \right) + \left(  \sum_{n=1}^{N} \hat{l}_n(\pi_n) - \hat{l}_n(\hat{\pi}_n) \right)
\end{align*}
Therefore, if the update rule for generating the decision sequence $\hat{\pi}_1, \pi_1, \hat{\pi}_2, \pi_2, \dots, \hat{\pi}_N, \pi_N$ contributes sufficient negativity in the term $\hat{l}_n(\pi_n) - \hat{l}_n(\hat{\pi}_n)$ compared with the regret of the new adversarial problem, then the regret of the original problem can be smaller than~\eqref{eq:regret of the 2N problem}.
This is potentially possible, as $\pi_n$ is made after $ \hat{l}_n$ is revealed. Especially, in the fixed-point formulation of \piccolo, $\pi_n$ and $ \hat{l}_n$ can be decided simultaneously.

In the next section, we show that when the base algorithm, which is adopted to solve the new adversarial problem given by the reduction, is in the family of mirror descent and FTRL. Then the regret bound of \piccolo with respect to the original predictable problem is  optimal.

\subsection{Optimal Regret Bounds for Predictable Problems}
 \label{app:regret bounds}

We show that if the base algorithm of \piccolo belongs to the family of optimal mirror descent and FTRL designed for adversarial problems, then \piccolo can achieve the optimal regret of predictable problems. 
In this subsection, we assume the loss sequence is linear, i.e. $l_n(\pi) = \lr{g_n}{\pi}$ for some $g_n$, and the results are summarized as Theorem~\ref{th:piccolo} in the main paper (in a slightly different notation).

\subsubsection{Mirror Descent}
First, we consider mirror descent as the base algorithm. In this case, we can write the \piccolo update rule as
\begin{align*}
\pi_{n} &= \argmin_{\pi \in \Pi}  \lr{\nabla \hat{l}_{n}(\hat{\pi}_n)}{x} + B_{H_{n-1}}(\pi || \hat{\pi}_{n})  
&  \text{[Prediction]}\\
\hat{\pi}_{n+1} &= \argmin_{\pi \in \Pi}  \lr{\nabla \delta_n(\pi_n)}{\pi} + B_{H_{n}}(\pi || \pi_n)
&  \text{[Correction]}
\end{align*}
where $H_n$ can be updated based on $e_n \coloneqq \nabla \delta_n (\pi_n)= \nabla l_n(\pi_n) - \nabla \hat{l}_n( \hat{\pi}_n)$ (recall by definition $\nabla l_n(\pi_n) = g_n$ and $\nabla \hat{l}_n( \hat{\pi}_n) = \nabla \hat{l}_n( \pi_n) = \hat{g}_n$).
Notice that in the Prediction Step, \piccolo uses the regularization from the previous Correction Step. 

To analyze the performance, we use a lemma of the mirror descent's properties. The proof is a straightforward application of the optimality condition of the proximal map~\cite{nesterov2013introductory}. We provide a proof here for completeness.
\begin{lemma} \label{lm:mirror descent}
	Let $\KK$ be a convex set. Suppose $R$ is 1-strongly convex with respect to norm $\norm{\cdot}$. Let $g$ be a vector in some Euclidean space and let
	\begin{align*}
	y =  \argmin_{z \in \KK} \lr{g}{z} + \frac{1}{\eta}B_{R}(z||x)
	\end{align*}
	Then for all $z \in \KK$
	\begin{align} \label{eq:mirror descent new decision}
	\eta \lr{g}{y - z} &\leq  B_R(z||x) - B_R(z||y) - B_R(y||x)  
	\end{align}
	which implies
	\begin{align} \label{eq:mirror descent old decision}
	\eta \lr{g}{x - z} \leq B_R(z||x) - B_R(z||y) + \frac{\eta^2}{2} \norm{g}_*^2
	\end{align}	
\end{lemma}
\begin{proof}
Recall the definition $B_R(z||x) =  R(z) - R(x) - \lr{\nabla R(x)}{z-x}$. 
The optimality of the proximal map can be written as 
\begin{align*}
\lr{ \eta g + \nabla R(y) -  \nabla R(x)  }{y - z} \leq 0,  \qquad \forall z \in \KK
\end{align*}
By rearranging the terms, we can rewrite the above inequality in terms Bregman divergences as follows and derive the first inequality~\eqref{eq:mirror descent new decision}:
\begin{align*}
\lr{ \eta  g   }{y - z} &\leq  \lr{ \nabla R(x) -  \nabla R(y)  }{y - z} \\
&=  B_R(z||x) - B_R(z||y) + \lr{ \nabla R(x) -  \nabla R(y)  }{y} -  \lr{\nabla R(x)}{x} + \lr{\nabla R(y)}{y}  + R(x)  - R(y)  \\
&=  B_R(z||x) - B_R(z||y) + \lr{ \nabla R(x)  }{y - x}   + R(x)  - R(y)  \\
&=  B_R(z||x) - B_R(z||y) - B_R(y||x)  
\end{align*}
The second inequality is the consequence of~\eqref{eq:mirror descent new decision}. First, we rewrite~\eqref{eq:mirror descent new decision} as
\begin{align*}
\lr{ \eta  g }{x - z}
&=  B_R(z||x) - B_R(z||y) - B_R(y||x)  + \lr{ \eta  g }{x - y}
\end{align*}
Then we use the fact that $B_R$ is 1-strongly convex with respect to $\norm{\cdot}$, which implies 
\begin{align*}
- B_R(y||x)  + \lr{ \eta  g }{x - y} \leq  - \frac{1}{2}\norm{x-y}^2 +    \lr{ \eta  g }{x - y}  \leq \frac{\eta^2}{2} \norm{g}_*^2
\end{align*}
Combining the two inequalities yields~\eqref{eq:mirror descent old decision}. 
\end{proof}

Lemma~\ref{lm:mirror descent} is usually stated with \eqref{eq:mirror descent old decision}, which concerns the decision made before seeing the per-round loss (as in the standard adversarial online learning setting). Here, we additionally concern $\hat{l}_n(\pi_n)$, which is the decision made after seeing $\hat{l}_n$, so we need a tighter bound~\eqref{eq:mirror descent new decision}. 

Now we show that the regret bound of \piccolo in the predictable linear problems when the base algorithm is mirror descent. 
\begin{proposition} \label{pr:piccolo+mirro descent}
	Assume the base algorithm of \piccolo is mirror descent satisfying the Assumption~\ref{as:base algorithm}. 
	Let $g_n = \nabla l_n(\pi_n)$ and $e_n = g_n - \hat{g}_n$.
	Then it holds that, for any $\pi \in \Pi$,
	\begin{align*} 
	\sum_{n=1}^N w_n \lr{ g_n}{\pi_n - \pi} 
	\leq 
	M_{N}  +   \sum_{n=1}^N  \frac{w_n^2}{2} \norm{e_n }_{*,n}^2  -  \frac{1}{2} \norm{\pi_n - \hat{\pi}_n}_{n-1}^2
	\end{align*}
\end{proposition}
\begin{proof}
Suppose $R_n$, which is defined by $H_n$, is 1-strongly convex with respect to $\norm{\cdot}_{n}$. Then by Lemma~\ref{lm:mirror descent}, we can write, for all $\pi \in \Pi$, 
\begin{align} \label{eq:mirror descent one step}
w_n \lr{ g_n}{\pi_n - \pi} 
&= w_n \lr{\hat{g}_n }{\pi_n - \pi} + w_n \lr{ e_n}{\pi_n - \pi} \nonumber   \\
&\leq  B_{R_{n-1}}(\pi||\hat{\pi}_n) - B_{R_{n-1}}(\pi||\pi_n) -  B_{R_{n-1}}(\pi_n||\hat{\pi}_n)  \nonumber  \\
&\quad +   B_{R_{n}}(\pi||\pi_n) - B_{R_n}(\pi||\hat{\pi}_{n+1})  + \frac{w_n^2}{2} \norm{e_n }_{*,n}^2  
\end{align}	
where we use~\eqref{eq:mirror descent new decision} for $\hat{g}_n$ and~\eqref{eq:mirror descent old decision} for the loss $e_n$. 

To show the regret bound of the original (predictable) problem, we first notice that 
\begin{align*}
&\sum_{n=1}^{N}  B_{R_{n-1}}(\pi||\hat{\pi}_n) - B_{R_{n-1}}(\pi||\pi_n) + B_{R_{n}}(\pi||\pi_n) - B_{R_n}(\pi||\hat{\pi}_{n+1})   \\
&= B_{R_0}(\pi||\hat{\pi}_{1}) - B_{R_N}(\pi||\hat{\pi}_{N+1}) +  \sum_{n=1}^{N}  B_{R_{n-1}}(\pi||\hat{\pi}_n) - B_{R_{n-1}}(\pi||\pi_n)   +   B_{R_{n}}(\pi||\pi_n) - B_{R_{n-1}}(\pi||\hat{\pi}_{n})\\
&= B_{R_0}(\pi||\hat{\pi}_{1}) - B_{R_N}(\pi||\hat{\pi}_{N+1}) +  \sum_{n=1}^{N}   B_{R_{n}}(\pi||\pi_n) -  B_{R_{n-1}}(\pi||\pi_n)  \leq M_{N}
\end{align*}
where the last inequality follows from  the assumption on the base algorithm. Therefore, by telescoping the inequality in~\eqref{eq:mirror descent one step} and using the strong convexity of $R_n$, we get 
\begin{align*} 
\sum_{n=1}^N w_n \lr{ g_n}{\pi_n - \pi} 
&\leq 
M_{N}  +   \sum_{n=1}^N  \frac{w_n^2}{2} \norm{e_n }_{*,n}^2  - B_{R_{n-1}}(\pi_n||\hat{\pi}_n)  \\
&\leq M_{N}  +   \sum_{n=1}^N  \frac{w_n^2}{2} \norm{e_n }_{*,n}^2  - \frac{1}{2} \norm{\pi_n - \hat{\pi}_n}_{n-1}^2 \qedhere
\end{align*}	
\end{proof}

\subsubsection{Follow-the-Regularized-Leader}

We consider another type of base algorithm, FTRL, which is mainly different from mirror descent in the way that constrained decision sets are handled~\citep{mcmahan2017survey}.
In this case, the exact update rule of \piccolo can be written as
\begin{align*}
\pi_{n} &=  \argmin_{\pi \in \Pi} \lr{w_n \hat{g}_n}{\pi} + \sum_{m=1}^{n-1} \lr{ w_m g_m }{\pi} + B_{r_{m}}(\pi|| \pi_m) 
&  \text{[Prediction]} \\
\hat{\pi}_{n+1} 
&=  \argmin_{\pi \in \Pi} \sum_{m=1}^{n}  \lr{ w_m g_m  }{\pi} + B_{r_{m}}(\pi|| \pi_m)  &  \text{[Correction]}
\end{align*}
From the above equations, we verify that \mobil~\citep{cheng2019accelerating} is indeed a special case of \piccolo, when the base algorithm is FTRL.

We show \piccolo with FTRL has the following guarantee. 
\begin{proposition} \label{pr:piccolo+FTRL}
	Assume the base algorithm of \piccolo is FTRL satisfying the Assumption~\ref{as:base algorithm}. Then it holds that, for any $\pi \in \Pi$,
	\begin{align*} 
	\sum_{n=1}^N w_n \lr{ g_n}{\pi_n - \pi} 
	\leq 
	M_{N}  +   \sum_{n=1}^N  \frac{w_n^2}{2} \norm{e_n }_{*,n}^2  -  \frac{1}{2} \norm{\pi_n - \hat{\pi}_n}_{n-1}^2
	\end{align*}
\end{proposition}

We show the above results of \piccolo using a different technique from~\citep{cheng2019accelerating}. Instead of developing a specialized proof like they do, we simply use the properties of FTRL on the 2$N$-step new adversarial problem!

To do so, we recall some facts of the base algorithm FTRL. First, FTRL in~\eqref{eq:FTRL} is equivalent to Follow-the-Leader (FTL) on a surrogate problem with the per-round loss is $\lr{g_n }{\pi} + B_{r_{n}}(\pi|| \pi_n)$. Therefore, the regret of FTRL can be bounded by the regret of FTL in the surrogate problem plus the size of the additional regularization $B_{r_{n}}(\pi|| \pi_n)$. Second, we recall a standard techniques in proving FTL, called Strong FTL Lemma (see e.g.~\citep{mcmahan2017survey}), which is proposed for \emph{adversarial} online learning. 
\begin{restatable}[Strong FTL Lemma~\citep{mcmahan2017survey}]{lemma}{strongFTL} \label{lm:strong FTL}
	For any sequence $\{\pi_n \in \Pi \}$ and $\{l_n\}$, 
	\begin{align*}
		\regret_N(l) \coloneqq  \sum_{n=1}^{N} l_n(\pi_n) - \min_{\pi \in \Pi} \sum_{n=1}^{N} l_n(\pi) \leq  \sum_{n=1}^{N}  l_{1:n} (\pi_n)- l_{1:n}( \pi_{n}^\star)
	\end{align*} where $
	\pi_n^\star \in \arg \min_{\pi \in \Pi} l_{1:n}(\pi)
	$.
\end{restatable}

Using the decomposition idea above, we show the performance of \piccolo following sketch below: first, we show a bound on the regret in the surrogate predictable problem with per-round loss $\lr{ g_n }{\pi} + B_{r_{n}}(\pi|| \pi_n)$; second, we derive the bound for the original predictable problem with per-round loss $\lr{ g_n }{\pi}$ by considering the effects of $B_{r_{n}}(\pi|| \pi_n)$. 
We will prove the first step by applying FTL on the transformed $2N$-step adversarial problem of the original $N$-step predictable surrogate problem and  then showing that \piccolo achieves the optimal regret in the original $N$-step predictable surrogate problem.  Interestingly, we recover the bound in the stronger FTL Lemma (Lemma~\ref{lm:stronger FTL}) by~\citet{cheng2019accelerating}, which they suggest is necessary for proving the improved regret bound of their FTRL-prediction algorithm (\mobil).

\begin{restatable}[Stronger FTL Lemma~\citep{cheng2019accelerating}]{lemma}{strongerFTL} \label{lm:stronger FTL}	
 	For any sequence $\{\pi_n\}$ and $\{l_n\}$,
	\begin{align*}
		\regret_N(l)  =  \sum_{n=1}^{N} l_{1:n}(\pi_n)- l_{1:n}( \pi_{n}^\star) - \Delta_{n}
	\end{align*} where
	$\Delta_{n+1} := l_{1:n}(\pi_{n+1}) - l_{1:n}(\pi_{n}^\star) \geq 0$ and 
	$
	\pi_n^\star \in \arg \min_{\pi \in \Pi} l_{1:n}(\pi)
	$.
\end{restatable}

Our new reduction-based regret bound is presented below.
\begin{proposition} \label{pr:Strong FTL is enough}
	Let $\{l_n\}$ be a predictable loss sequence with predictable information $\{\hat{l}_n\}$. Suppose the decision sequence $\hat{\pi}_1, \pi_1, \hat{\pi}_2, \dots, \hat{\pi}_N, \pi_N$ is generated by running FTL on the transformed adversarial loss sequence $\hat{l}_1, \delta_1, \hat{l}_2, \dots, \hat{l}_N, \delta_N$, then the bound in the Stonger FTL Lemma holds. That is, 
	$	\regret_N(l)  \leq  \sum_{n=1}^{N} l_{1:n}(\pi_n)- l_{1:n}( \pi_{n}^\star) - \Delta_{n} $, where
	$\Delta_{n+1} := l_{1:n}(\pi_{n+1}) - l_{1:n}(\pi_{n}^\star) \geq 0$ and  $
	\pi_n^\star \in \arg \min_{\pi \in \Pi} l_{1:n}(\pi)
	$.
\end{proposition}
\begin{proof}
First, we transform the loss sequence and write
\begin{align*}
\sum_{n=1}^{N} l_n(\pi_n) 
= \sum_{n=1}^{N} \hat{l}_n(\pi_n) + \delta_n(\pi_n)  
= \left( \sum_{n=1}^{N} \hat{l}_n(\hat{\pi}_n) + \delta_n(\pi_n) \right) + \left( \sum_{n=1}^{N} \hat{l}_n(\pi_n) - \hat{l}_n(\hat{\pi}_n) \right)
\end{align*}
Then we apply standard Strong FTL Lemma on the new adversarial problem in the left term.
\begin{align*}
&\sum_{n=1}^{N} \hat{l}_n(\hat{\pi}_n) + \delta_n(\pi_n) \\
&\leq  \sum_{n=1}^{N} (\hat{l}+\delta)_{1:n}(\pi_n) - \min_{\pi \in \Pi} (\hat{l}+\delta)_{1:n}(\pi) +  \sum_{n=1}^{N} ((\hat{l}+\delta)_{1:n-1}+\hat{l}_{n})(\hat{\pi}_n) - \min_{\pi \in \Pi} ((\hat{l}+\delta)_{1:n-1}+\hat{l}_{n})(\pi)\\
&=  \sum_{n=1}^{N} l_{1:n}(\pi_n) - \min_{\pi \in \Pi} l_{1:n}(\pi) +  \sum_{n=1}^{N} (l_{1:n-1}+\hat{l}_{n})(\hat{\pi}_n) -  (l_{1:n-1}+\hat{l}_{n})(\pi_n)
\end{align*}
where the first inequality is due to Strong FTL Lemma and the second equality is because FTL update assumption.

Now we observe that if we add the second term above and $ \sum_{n=1}^{N} \hat{l}_n(\pi_n) - \hat{l}_n(\hat{\pi}_n)$ together, we have 
\begin{align*}
&\sum_{n=1}^{N} (l_{1:n-1}+\hat{l}_{n})(\hat{\pi}_n) -  (l_{1:n-1}+\hat{l}_{n})(\pi_n) + (\hat{l}_n(\pi_n) - \hat{l}_n(\hat{\pi}_n))\\
& = \sum_{n=1}^{N} (l_{1:n-1})(\hat{\pi}_n) - l_{1:n-1}(\pi_n) = \Delta_n
\end{align*}
Thus, combing previous two inequalities, we have the bound in the Stronger FTL Lemma: 
\begin{align*}
\sum_{n=1}^{N} l_n(\pi_n)  &\leq \sum_{n=1}^{N} l_{1:n}(\pi_n) - \min_{\pi \in \Pi} l_{1:n}(\pi) - \Delta_n \qedhere
\end{align*}
\end{proof}

Using Proposition~\ref{pr:Strong FTL is enough}, we can now bound the regret of \piccolo in Proposition~\ref{pr:piccolo+FTRL} easily. 
\begin{proof}[Proof of Proposition~\ref{pr:piccolo+FTRL}]
Suppose $\sum_{m=1}^{n} B_{r_m}(\cdot|| \pi_m )$ is 1-strongly convex with respect to some norm $\norm{\cdot}_{n}$. 
Let $f_n = \lr{  w_n g_n }{\pi_n} + B_{r_{n}}(\pi|| \pi_m)$. 
Then by a simple convexity analysis (see e.g. see~\citep{mcmahan2017survey}) and Proposition~\ref{pr:Strong FTL is enough}, we can derive
\begin{align*}
	\regret_N(f)  &\leq  \sum_{n=1}^{N}(f_{1:n}(\pi_n)- \min_{\pi \in \Pi}f_{1:n}( \pi )) - ( f_{1:n-1}(\pi_{n}) - f_{1:n-1}(\hat{\pi}_{n})) \\
	&\leq  \sum_{n=1}^{N} \frac{w_n^2}{2} \norm{ e_n}_{n,*}^2 - \frac{1}{2}\norm{\pi_n - \hat{\pi}_n}_{n-1}^2 
\end{align*}

Finally, because $r_n$ is proximal (i.e. $B_{r_n}(\pi_n || \pi_n) = 0$),  we can bound the original regret: for any $\pi \in \Pi$, it satisfies that 
\begin{align*} 
\sum_{n=1}^N w_n \lr{ g_n}{\pi_n - \pi} 
&\leq 
\sum_{n=1}^{N} f_n(\pi_n)  - f_n (\pi) +  B_{r_n}(\pi||\pi_n)  \\
&\leq 
M_{N}  +   \sum_{n=1}^N  \frac{w_n^2}{2} \norm{e_n}_{*,n}^2  -  \frac{1}{2} \norm{\pi_n - \hat{\pi}_n}_{n-1}^2
\end{align*}
where we use Assumption~\ref{as:base algorithm} and the bound of $\regret_N(f)$
in the second inequality.
\end{proof}

\section{Policy Optimization Analysis of \piccolo} \label{app:analysis of policy optimization}

In this section, we discuss how to interpret the bound given in Theorem~\ref{th:piccolo}
\begin{align*} 
\sum_{n=1}^N w_n \lr{ g_n}{\pi_n - \pi} 
\leq 
M_{N}  +   \sum_{n=1}^N  \frac{w_n^2}{2} \norm{e_n }_{*,n}^2  -  \frac{1}{2} \norm{\pi_n - \hat{\pi}_n}_{n-1}^2
\end{align*}
in the context of policy optimization 
and show exactly how the optimal bound
\begin{align}  \label{eq:piccolo optimal bound}
\E\left[\sum_{n=1}^N  \lr{ w_n  g_n }{\pi_n - \pi} \right]
\leq O(1) + C_{\Pi,\Phi} \frac{w_{1:N}}{\sqrt{N}}
\end{align}
is derived. We will  discuss how model learning can further help minimize the regret bound later in Appendix~\ref{app:model learning}.

\subsection{Assumptions}
We introduce some assumptions to characterize the sampled gradient $g_n$. Recall $g_n = \nabla \tilde{l}_n(\pi_n)$.  
\begin{assumption} \label{as:env assumption}
	$\norm{\E[g_n]}_*^2 \leq G_g^2$ and $\norm{g_n - \E[g_n]}_*^2 \leq \sigma_g^2$ for some finite constants $G_g$ and $\sigma_g$. 
\end{assumption}
Similarly, we consider properties of the predictive model $\Phi_n$ that is used to estimate the gradient of the next per-round loss.  
Let $\PP$ denote the class of these models (i.e. $\Phi_n \in \PP$), which can potentially be \emph{stochastic}. We make assumptions on the size of $\hat{g}_n$ and its variance.
\begin{assumption} \label{as:model assumption}	
	$\norm{\E[\hat{g}_n]}_*^2 \leq G_{\hat{g}}^2$
	and 
	$
	\E[\norm{\hat{g}_n - \E[\hat{g}_n] }_*^2] \leq \sigma_{\hat{g}}^2
	$
	for some finite constants $G_{\hat{g}}$ and  $\sigma_{\hat{g}}$.
\end{assumption}
Additionally, we assume these models are Lipschitz continuous.
\begin{assumption} \label{as:model Lipschitz assumption}
	There is a constant $L \in [0, \infty)$	 such that, 
	for any instantaneous cost $\psi$ and any $\Phi \in \PP$, it satisfies
	$
	\norm{\E[\Phi (\pi)] - \E[\Phi (\pi')] }_* \leq L \norm{\pi - \pi'} 
	$.
\end{assumption}

Lastly, as \piccolo is agnostic to the base algorithm, we assume the local norm $\norm{\cdot}_n$ chosen by the base algorithm at round $n$ satisfies $\norm{\cdot}_n^2 \geq \alpha_n \norm{\cdot}^2$ for some $\alpha_n > 0$. This condition implies that  $\norm{\cdot}_{n,*}^2 \leq \frac{1}{\alpha_n} \norm{\cdot}_*^2$.
In addition, we assume $\alpha_n$ is non-decreasing so that $M_N = O(\alpha_{N})$ in Assumption~\ref{as:base algorithm}, where the leading constant in the bound $O(\alpha_{N})$ is proportional to $|\Pi|$, as commonly chosen in online convex optimization.

\subsection{A Useful Lemma}

We study the bound in Theorem~\ref{th:piccolo} under the assumptions made in the previous section. 
We first derive a basic inequality, following the idea in \citep[Lemma 4.3]{cheng2019accelerating}. 
\begin{lemma} \label{lm:expected error gradient}
	Under Assumptions~\ref{as:env assumption},~\ref{as:model assumption}, and~\ref{as:model Lipschitz assumption}, it holds
	\begin{align*}
\E[\norm{e_n}_{*,n}^2]	= \E[\norm{g_n - \hat{g}_n}_{*,n}^2] &\leq  \frac{4}{\alpha_n} \left(\sigma_g^2 +  \sigma_{\hat{g}}^2  + L_n^2 \norm{\pi_n - \hat{\pi}_n}_n^2 + E_n(\Phi_n)  \right)
	\end{align*}
	where $E_n(\Phi_n) = \norm{\E[g_n] - \E[\Phi_n(\pi_n, \psi_n)] }_{*}^2 $ is the prediction error of model $\Phi_n$. 
\end{lemma}
\begin{proof}
	Recall $\hat{g}_n = \Phi_n(\hat{\pi}_n, \psi_n)$. Using the triangular inequality, we can simply derive
	\begin{align*}
	& \E[\norm{g_n - \hat{g}_n}_{*,n}^2] \\
	&\leq 4 \left(  \E[ \norm{g_n - \E[g_n]}_{*,n}^2 ]  
	+ \norm{\E[g_n] - \E[\Phi_n(\pi_n, \psi_n)] }_{*,n}^2 
	+ \norm{ \E[\Phi_n( \pi_n, \psi_n)] - \E[\hat{g}_n] }_{*,n}^2 
	+ \E[\norm{\E[\hat{g}_n] - \hat{g}_n}_{*,n}^2] \right)\\
	&= 4 \left(  \E[ \norm{g_n - \E[g_n]}_{*,n}^2 ]  
	+ \norm{\E[g_n] - \E[\Phi_n(\pi_n, \psi_n)] }_{*,n}^2 
	+ \norm{ \E[\Phi_n( \pi_n, \psi_n)] - \E[\Phi_n( \hat{\pi}_n, \psi_n)] }_{*,n}^2 
	+ \E[\norm{\E[\hat{g}_n] - \hat{g}_n}_{*,n}^2] \right)\\
	&\leq 4 \left( \frac{1}{\alpha_n} \sigma_g^2 +  \frac{1}{\alpha_n} E_n(\Phi_n)   +   \norm{ \E[\Phi_n( \pi_n, \psi_n)] - \E[\Phi_n(\hat{\pi}_n, \psi_n) ]}_{*,n}^2 
	+ \frac{1}{\alpha_n} \sigma_{\hat{g}}^2 \right) \\
	&\leq \frac{4}{\alpha_n} \left(\sigma_g^2 +  \sigma_{\hat{g}}^2  + L^2 \norm{\pi_n - \hat{\pi}_n}_n^2 + E_n(\Phi_n)  \right)
	\end{align*}
	where the last inequality is due to Assumption~\ref{as:model Lipschitz assumption}.
\end{proof}

\subsection{Optimal Regret Bounds}

We now analyze the regret bound in Theorem~\ref{th:piccolo}
\begin{align}  \label{eq:piccolo regret bound}
\sum_{n=1}^N w_n \lr{ g_n}{\pi_n - \pi} 
\leq 
M_{N}  +   \sum_{n=1}^N  \frac{w_n^2}{2} \norm{e_n }_{*,n}^2  -  \frac{1}{2} \norm{\pi_n - \hat{\pi}_n}_{n-1}^2
\end{align}
We first gain some intuition about the size of 
\begin{align} \label{eq:the term considered by adpat}
M_{N} +  \E\left[  \sum_{n=1}^N  \frac{w_n^2}{2} \norm{ e_n }_{*,n}^2 \right].
\end{align}
Because when $\adapt(h_n, H_{n-1}, e_n, w_n)$ is called in the Correction Step in~\eqref{eq:correction} with the error gradient $e_n$ as input, an optimal base algorithm (e.g. all the base algorithms listed in~Appendix~\ref{app:basic operations}) would choose a local norm sequence $\norm{\cdot}_n$ such that~\eqref{eq:the term considered by adpat} is optimal. For example, suppose $\norm{ e_n }_{*}^2 = O(1)$ and $w_n = n^p$ for some $p>-1$. If the base algorithm is basic mirror descent (cf. Appendix~\ref{app:basic operations}), then $\alpha_n = O(\frac{w_{1:n}}{\sqrt{n}})$. By our assumption that $M_N = O(\alpha_{N})$, it implies~\eqref{eq:the term considered by adpat} can be upper bounded by
\begin{align*}
M_N  + \E\left[  \sum_{n=1}^N  \frac{w_n^2}{2} \norm{ e_n }_{*,n}^2 \right] 
&\leq O\left(\frac{w_{1:N}}{\sqrt{N}}\right) + \left[ \sum_{n=1}^N \frac{w_n^2 \sqrt{n} }{ 2  w_{1:n}} \norm{ e_n }_{*}^2 \right] \\
&\leq O\left( \frac{w_{1:N}}{\sqrt{N}} +  \sum_{n=1}^N \frac{w_n^2 \sqrt{n} }{w_{1:n}}  \right)
= O\left( N^{p+1/2}  \right)
\end{align*}
which will lead to an optimal weighted average regret in $O(\frac{1}{\sqrt{N}})$.

\piccolo actually has a better regret than the simplified case discussed above, because of the negative term $- \frac{1}{2} \norm{\pi_n - \hat{\pi}_n}_{n-1}^2$ in~\eqref{eq:piccolo regret bound}. To see its effects, we combine
Lemma~\ref{lm:expected error gradient} with~\eqref{eq:piccolo regret bound} to reveal some insights: 
\begin{align} 
\E\left[ \sum_{n=1}^N w_n \lr{  g_n }{\pi_n - \pi}\right] 
&\leq 
O(\alpha_{N})  +  \E\left[ \sum_{n=1}^N  \frac{w_n^2}{2} \norm{e_n }_{*,n}^2  -  \frac{1}{2} \norm{\pi_n - \hat{\pi}_n}_{n-1}^2 \right] \label{eq:piccolo tigher bound}  \\
&\leq O(\alpha_{N})  +  \E\left[ \sum_{n=1}^N  \frac{2w_n^2}{\alpha_n} \left(\sigma_g^2 +  \sigma_{\hat{g}}^2  + L^2 \norm{\pi_n - \hat{\pi}_n}_n^2 + E_n(\Phi_n)  \right)  -  \frac{\alpha_{n-1}}{2} \norm{\pi_n - \hat{\pi}_n}^2  \right] \nonumber  \\
&= \left( O( \alpha_{N})  +  \E\left[ \sum_{n=1}^N  \frac{2w_n^2}{\alpha_n} \left(\sigma_g^2 +  \sigma_{\hat{g}}^2  + E_n(\Phi_n)  \right) \right] \right) + \left( \E\left[ \sum_{n=1}^N ( \frac{2w_n^2}{\alpha_n} L^2 - \frac{\alpha_{n-1}}{2}) \norm{\pi_n - \hat{\pi}_n}^2 \right]  \right) 
\label{eq:piccolo looser bound}
\end{align}

The first term in~\eqref{eq:piccolo looser bound} plays the same role as~\eqref{eq:the term considered by adpat}; when the base algorithm has an optimal $\adapt$ operation and $w_n = n^p$ for some $p>-1$, it would be in $O\left( N^{p+1/2}  \right)$. Here we see that the constant factor in this bound is proportional to $\sigma_g^2 +  \sigma_{\hat{g}}^2  + E_n(\Phi_n)$. Therefore, if the variances $\sigma_g^2$, $\sigma_{\hat{g}}^2$ of the gradients are small, the regret would mainly depend on the prediction error $E_n(\Phi_n)$ of $\Phi_n$. In the next section (Appendix~\ref{app:model learning}), we will show that when $\Phi_n$ is learned online (as the authors in~\citep{cheng2019accelerating} suggest),  on average the regret is close to the regret of using the best model in the hindsight. 
The second term in~\eqref{eq:piccolo looser bound} contributes to $O(1)$ in the regret, when the base algorithm adapts properly to $w_n$. For example, when $\alpha_n = \Theta(\frac{w_{1:n}}{\sqrt{n}})$ and $w_n = n^p$ for some $p>-1$, then 
\begin{align*}
\sum_{n=1}^{N}\frac{2w_n^2}{\alpha_n} L^2 - \frac{\alpha_{n-1}}{2} = \sum_{n=1}^{N} O(n^{p-1/2} - n^{p+1/2} ) = O(1)
\end{align*}
In addition, because $ \norm{\pi_n - \hat{\pi}_n}$ would converge to zero, the effects of the second term in~\eqref{eq:piccolo looser bound} becomes even minor.

In summary, for a reasonable base algorithm and $w_n = n^p$ with  $p>-1$, running \piccolo has the regret bound 
\begin{align} \label{eq:piccolo regret summary}
\E\left[ \sum_{n=1}^N w_n \lr{  g_n }{\pi_n - \pi} \right] 
&= O(\alpha_N) + O \left( \frac{w_{1:N}}{\sqrt{N}}  (\sigma_g^2 +  \sigma_{\hat{g}}^2 )   \right) + O(1) +  \E\left[ \sum_{n=1}^N  \frac{2w_n^2}{\alpha_n} E_n(\Phi_n)  \right] 
\end{align}
Suppose $\alpha_n = \Theta(|\Pi|\frac{w_{1:n}}{\sqrt{n}})$ and $w_n = n^p$ for some $p>-1$,
This implies the inequality 
\begin{align}  \tag{\ref{eq:piccolo optimal bound}}
\E\left[\sum_{n=1}^N  \lr{ w_n  g_n }{\pi_n - \pi} \right]
&\leq O(1) + C_{\Pi,\Phi} \frac{w_{1:N}}{\sqrt{N}}
\end{align}
where $C_{\Pi,\Phi} = O(|\Pi| + \sigma_g^2 +  \sigma_{\hat{g}}^2 +  \sup_n E_n(\Phi_n) )$. The use of non-uniform weights can lead to a faster on average decay of the standing $O(1)$ term in the final weighted average regret bound, i.e.
\begin{align*}
\frac{1}{w_{1:N}}\E\left[\sum_{n=1}^N  \lr{ w_n  g_n }{\pi_n - \pi} \right]
\leq O\left(\frac{1}{w_{1:N}}\right) +  \frac{C_{\Pi,\Phi}}{\sqrt{N}}
\end{align*}
In general, the authors in~\citep{cheng2018fast,cheng2019accelerating} recommend using $p\ll N$ (e.g. in the range of $[0,5]$) to remove the undesirable constant factor, yet without introducing large multiplicative constant factor.

\subsection{Model Learning} \label{app:model learning}

The regret bound in~\eqref{eq:piccolo regret summary} reveals an important factor that is due to the prediction error
 $\E\left[ \sum_{n=1}^N  \frac{2w_n^2}{\alpha_n} E_n(\Phi_n)  \right] $,
where we recall $E_n(\Phi_n) = \norm{\E[g_n] - \E[\Phi_n(\pi_n)] }_{*}^2 $. 
\citet{cheng2019accelerating} show that, to minimize this error sum through model learning, a secondary online learning problem with per-round loss $E_n(\cdot)$ can be considered.
Note that this is a standard weighted adversarial online learning problem (weighted by $\frac{2w_n^2}{\alpha_n}$), because $E_n(\cdot)$ is revealed after one commits to using model $\Phi_n$. 

While in implementation the exact function $E_n(\cdot)$ is unavailable (as it requires infinite data), we can adopt an unbiased upper bound. For example, \citet{cheng2019accelerating} show that $E_n(\cdot)$ can be upper bounded by the single- or multi-step prediction error of a transition dynamics model. 
More generally, we can learn a neural network to minimize the gradient prediction error directly. As long as this secondary online learning problem is solved by a no-regret algorithm, the error due to online model learning would contribute a term in $O(w_{1:N} \epsilon_{\PP,N}/\sqrt{N} ) +  o(w_{1:N}/\sqrt{N})$ in~\eqref{eq:piccolo regret summary}, where $\epsilon_{\PP,N}$ is the minimal error achieved by the best model in the model class $\PP$ (see~\citep{cheng2019accelerating} for details).

\section{Experimental Details} \label{app:exp details}

\subsection{Algorithms}

\paragraph{Base Algorithms}
In the experiments, we consider three commonly used first-order online learning algorithms: 
\adam, \natgrad, and \trpo,
all of which adapt the regularization online to alleviate the burden of learning rate tuning. We provide the decomposition of \adam into the basic three operations in Appendix~\ref{app:basic operations}, and that of \natgrad in Appendix~\ref{app:example}. In particular, the adaptivity of \natgrad is achieved by adjusting the step size based on a moving average of the dual norm of the gradient. 
\trpo adjusts the step size to minimize a given cost function (here it is a linear function defined by the first-order oracle)
within a pre-specified KL divergence centered at the current decision.
While greedily changing the step size in every iteration makes \trpo an inappropriate candidate for adversarial online learning. Nonetheless, it can still be written in the form of mirror descent and allows a decomposition using the three basic operators; its $\adapt$ operator can be defined as the process of finding the maximal scalar step along the natural gradient direction such that the updated decision stays within the trust region.
For all the algorithms, a decaying step size multiplier in the form  $\eta /(1 + \alpha \sqrt{n} )$ is also used; for \trpo, it is used to specify the size of trust regions.
The values chosen for the hyperparameters $\eta$ and $\alpha$ can be found in Table~\ref{table:tasks}.
To the best of our knowledge, the conversion of these approaches into unbiased model-based algorithms is novel.

\paragraph{Reinforcement Learning  Per-round Loss}
In iteration $n$, 
in order to compute the online gradient~\eqref{eq:RL online loss},
GAE~\citep{schulman2015high} is used to estimate the advantage function $A_{\pi_{n-1}}$.
More concretely, this advantage estimate utilizes an estimate of value function $V_{\pi_{n-1}}$ (which we denote $\hat{V}_{\pi_{n-1}}$) and on-policy samples. 
We chosen $\lambda=0.98$ in GAE to reduce influence of the error in $V_{\pi_{n-1}}$,  which can be catastrophic. 
Importance sampling can be used to estimate $A_{\pi_{n-1}}$ in order to leverage data that are  collected on-policy by running $\pi_n$. 
However, since we select a large $\lambda$, importance sampling can lead to vanishing importance weights, making the gradient extremely noisy. Therefore, in the experiments, importance sampling is not applied.

\paragraph{Gradient Computation and Control Variate}
The gradients are computed using likelihood-ratio trick and the associated advantage function estimates described above. A scalar control variate is further used to reduce the variance of the sampled gradient, which is set to the mean of the advantage estimates evaluated on newly collected data.

\paragraph{Policies and Value Networks}
Simple feed-forward neural networks are used to construct all of the function approximators (policy and value function) in the tasks. They have 1 hidden layer with 32 $\tanh$ units for all policy networks, and have 2 hidden layers with 64 $\tanh$ units for value function networks.
Gaussian stochastic policies are considered, i.e., for any state $s \in \Sbb$, $\pi_s$ is Gaussian, and the mean of $\pi_s$ is modeled by the policy network, whereas the diagonal covariance matrix is state independent (which is also learned).
Initial value of $\log \sigma$ of the Gaussian policies $-1.0$, the standard deviation for initializing the output layer is $0.01$, and the standard deviation for initialization hidden layer is $1.0$. 
After the policy update, a new value function estimate $\hat V_{\pi_{n}}$ is computed by minimizing the mean of squared difference between $\hat V_{\pi_{n}}$ and $\hat V_{\pi_{n-1}} + \hat A_{\pi_{n}}$, where $\hat A_{\pi_{n}}$ is the GAE estimate using $\hat V_{\pi_{n-1}}$ and $\lambda=0.98$, 
through \adam with batch size 128, number of batches 2048, and learning rate 0.001.
Value function is pretrained using examples collected by executing the randomly initialized policy. 

\paragraph{Computing Model Gradients}
We compute $\hat{g}_n$ in two ways. The first approach is to use the simple heuristic that sets $\hat{g}_n = \Phi_n(\hat{\pi}_n)$, where $\Phi_n$ is some predictive models depending on the exact experimental setup. 
The second approach is to use the fixed-point formulation~\eqref{eq:fixed-point problem}. This is realized by solving the equivalent optimization problem mentioned in the paper. 
In implementation, we only solves this problem approximately using some finite number of gradient steps; though this is insufficient to yield a stationary point as desired in the theory, we experimentally find that it is sufficient to yield improvement over the heuristic $\hat{g}_n = \Phi_n(\hat{\pi}_n)$.

\paragraph{Approximate Solution to Fixed-Point Problems of \piccolo}

\piccolo relies on the predicted gradient $\hat{g}_n$ in the Prediction Step. 
Recall ideally we wish to solve the fixed-point problem that finds $h_n^*$ such that 
\begin{align} 
h_n^* &= \update(\hat{h}_{n}, H_{n-1}, \Phi_n(\pi_n(h_n^*)), w_n)  
\end{align}
and then apply $\hat{g}_n = \Phi_n(\pi_n(h_{n}^*))$ in the Prediction Step to get $h_n$, i.e., 
\begin{align*}
h_{n} &= \update(\hat{h}_{n}, H_{n-1}, \hat{g}_n, w_n)  
\end{align*}
Because $h_n^*$ is the solution to the fixed-point problem, we have $h_n = h_n^*$. Such choice of $\hat{g}_n$ will fully leverage the information provided by $\Phi_n$, as it does not induce additional linearization due to evaluating $\Phi_n$ at  points different from $h_n$. 

Exactly solving the fixed-point problem is difficult. In the experiments, we adopt a heuristic which computes an approximation to $h_n^*$ as follows. We suppose $\Phi_n = \nabla f_n$ for some function $f_n$, which is the case e.g. when $\Phi_n$ is the simulated gradient based on some (biased) dynamics model. This restriction makes the fixed-point problem as finding a stationary point of the optimization problem $\min_{\pi \in \Pi} f_n(\pi)  + B_{R_{n-1}}(\pi||\hat\pi_n)$. 
In implementation, we initialize the iterate in this subproblem as $\update(\hat{h}_{n}, H_{n-1}, \Phi_n(\hat{\pi}_n), w_n) $, which is the output of the Prediction Step if we were to use $\hat{g}_n = \Phi_n(\hat{\pi}_n)$. We made this choice in initializing the subproblem, as we know that using $\hat{g}_n = \Phi_n(\hat{\pi}_n)$ in \piccolo already works well (see the experiments) and it can be viewed as the solution to the fixed-point problem with respect to the linearized version of $\Phi_n$ at $\hat{\pi}_n$. Given the this initialization point, we proceed to compute the approximate solution to the fixed-point by applying the given base algorithm for $5$ iterations and then return the last iterate as the approximate solution. For example, if the base algorithm is natural gradient descent, we fixed the Bregman divergence (i.e. its the Fisher information matrix as $\hat{\pi}_n$) and only updated the scalar stepsize adaptively along with the policy in solving this \emph{regularized model-based RL problem} (i.e. $\min_{\pi \in \Pi} f_n(\pi)  + B_{R_{n-1}}(\pi||\hat\pi_n)$). While such simple implementation is not ideal, we found it works in practice, though we acknowledge that a better implementation of the subproblem solver would improve the results.

\subsection{Tasks}
The robotic control tasks that are considered in the experiments are 
CartPole, Hopper, Snake, 
and Walker3D from OpenAI Gym~\citep{brockman2016openai} with the DART physics engine~\citep{Lee2018}\footnote{The environments are defined in DartEnv, hosted at https://github.com/DartEnv.}.
CartPole is a classic control problem, and its goal is to keep a pole balanced in a upright posture, by only applying force to the cart. 
Hopper, Snake, and Walker3D are locomotion tasks, of which the goal is to control an agent  to move forward as quickly as possible without falling down (for Hopper and Walker3D) or
deviating too much from moving forward (for Snake). 
Hopper is monopedal and Walker3D is bipedal, and both of them are subjected to significant contact discontinuities that are hard or even impossible to predict.

\subsection{Full Experimental Results}
In Figure~\ref{fig:simple exps supp}, we empirically study the properties of \piccolo that are predicted by theory on CartPole environment.
In Figure~\ref{fig:exps}, we ``\piccolo" three base algorithms: \adam, \natgrad, \trpo, and apply them on four simulated environments: Cartpole, Hopper, Snake, and Walker3D.

\begin{figure}[t] 	
	\centering
	\begin{subfigure}{.24\textwidth}
		\includegraphics[width=\textwidth]{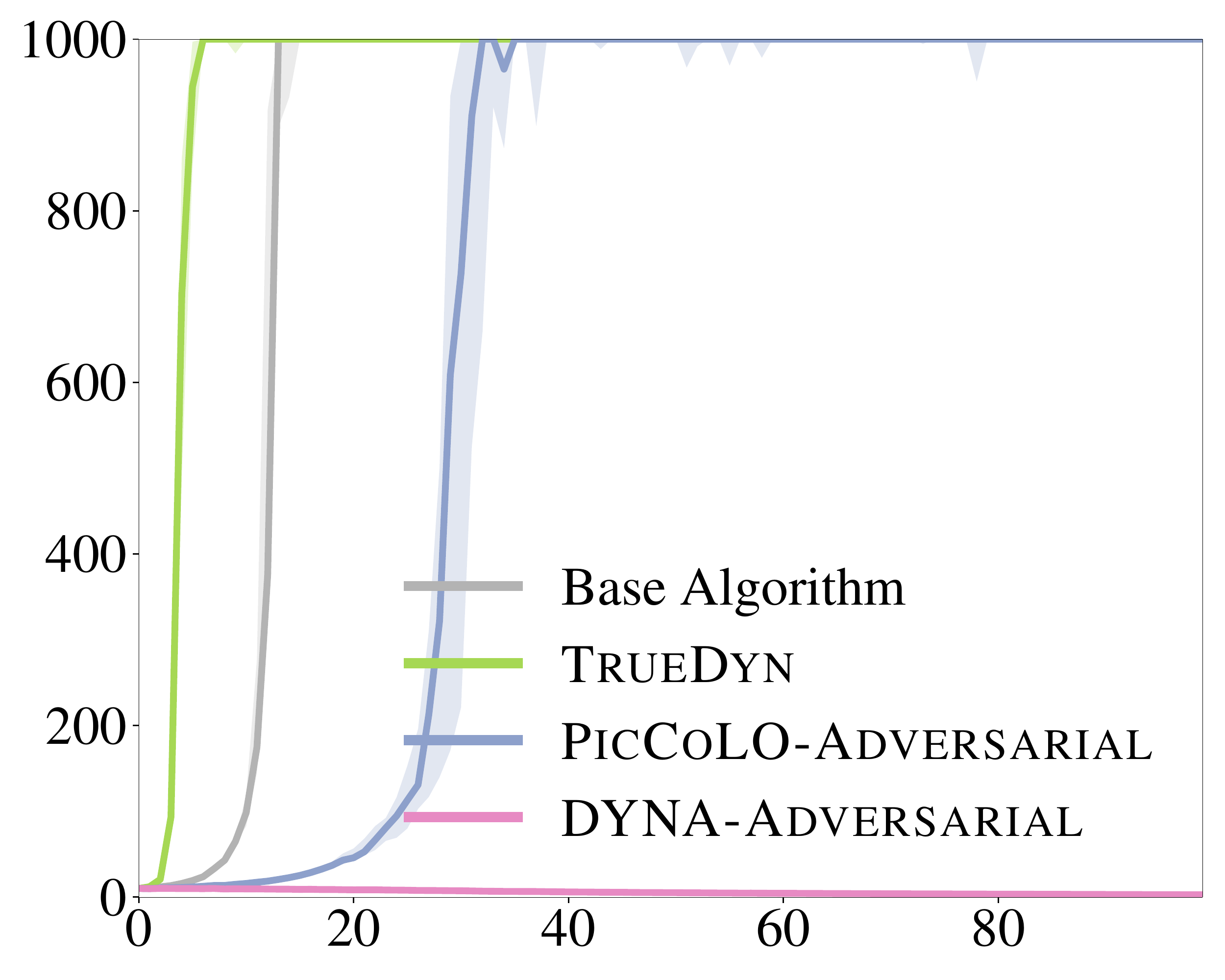}
		\caption{Adv. model, \natgrad}
	\end{subfigure}
	\begin{subfigure}{.24\textwidth}
		\includegraphics[width=\textwidth]{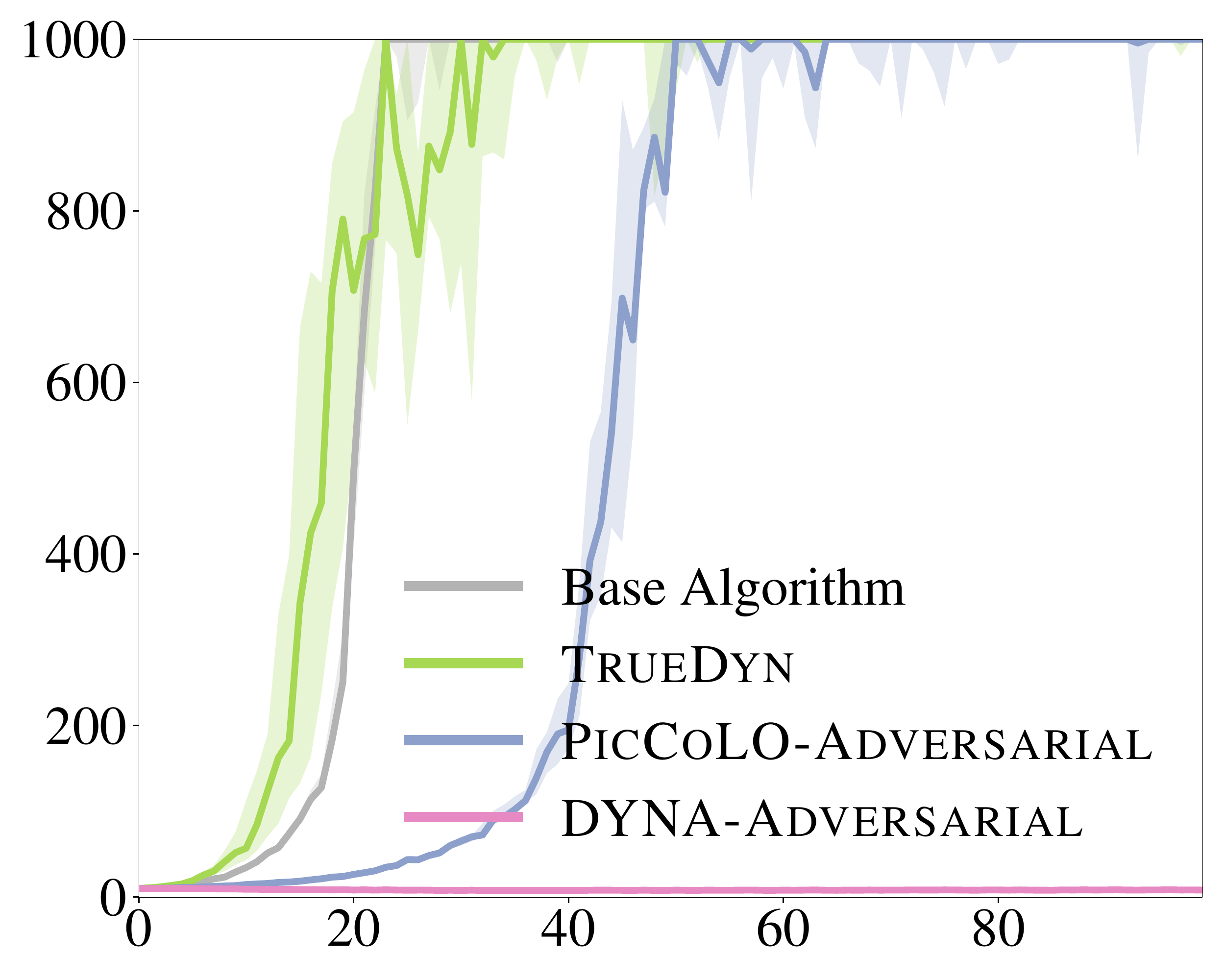}
		\caption{Adv. model, \trpo}
	\end{subfigure}
	\begin{subfigure}{.24\textwidth}
	\includegraphics[width=\textwidth]{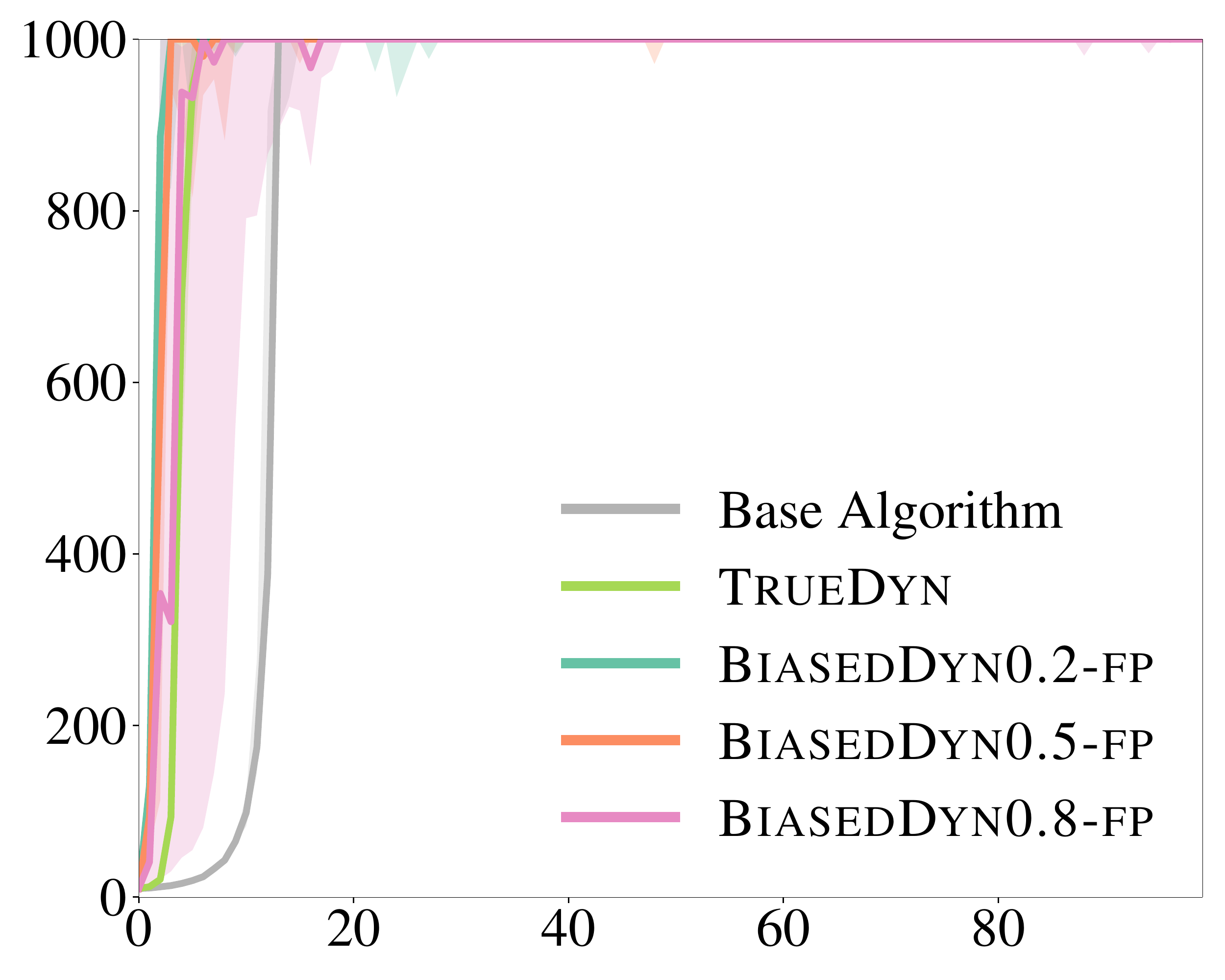}
	\caption{Diff. fidelity, \natgrad}
	\end{subfigure}
	\begin{subfigure}{.24\textwidth}
	\includegraphics[width=\textwidth]{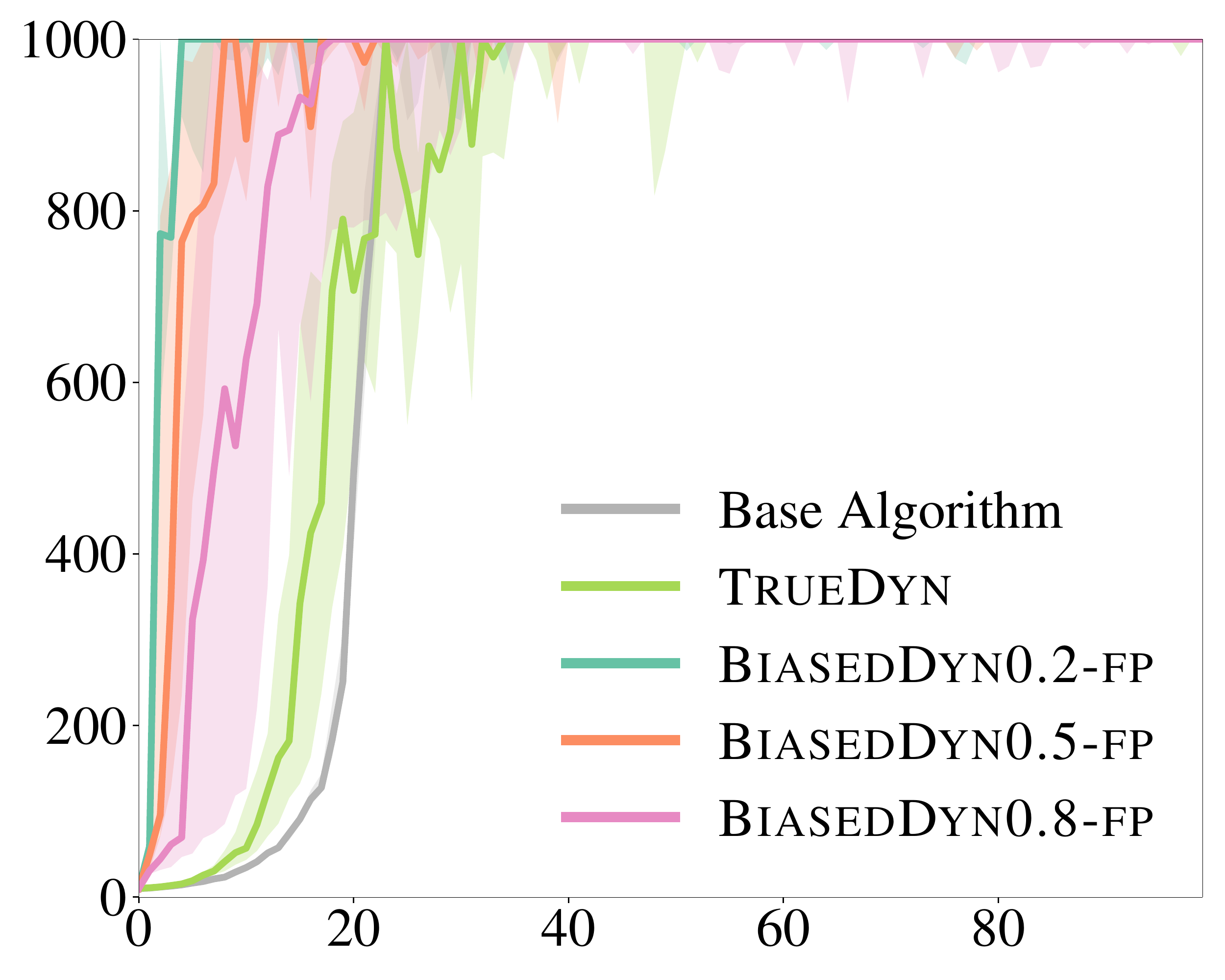}
	\caption{Diff. fidelity, \trpo}
	\end{subfigure}
	\caption{ 
	Performance of \piccolo with different predictive models on CartPole. 
	$x$ axis is iteration number and $y$ axis is sum of rewards. The curves are the median among 8 runs with different seeds, and the shaded regions account for $25\%$ percentile.
	The update rule, by default, is $\piccolo$. For example \sim in (a) refers to \piccolo with \sim predictive model. 
	(a), (b): Comparison of \piccolo and \dyna with adversarial model using \natgrad and \trpo as base algorithms.  
	(c), (d): \piccolo with the fixed-point setting~\eqref{eq:fixed-point problem} with dynamics model in different fidelities.
	\biasedE indicates that the mass of each individual robot link is either increased or decreased by $80\%$ with probablity 0.5 respectively. 
	}
	\label{fig:simple exps supp}
\end{figure}

\begin{figure*} [h]	
	\centering
\begin{subfigure}{.24\textwidth}
	\includegraphics[width=\textwidth]{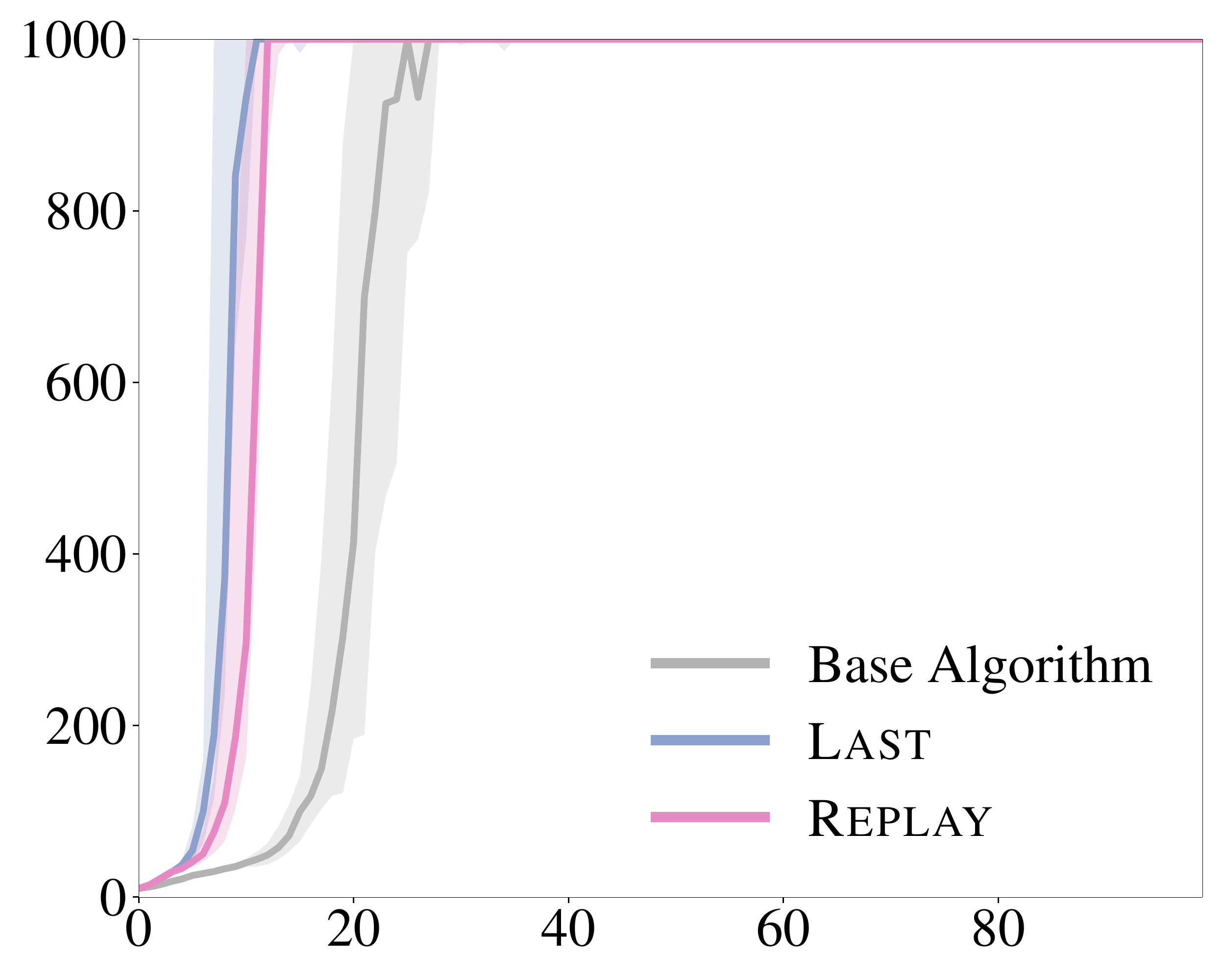}
	\caption{CartPole \adam}
\end{subfigure}
\begin{subfigure}{.24\textwidth}
	\includegraphics[width=\textwidth]{hopper_adam}
	\caption{Hopper \adam}
\end{subfigure}
\begin{subfigure}{.24\textwidth}
	\includegraphics[width=\textwidth]{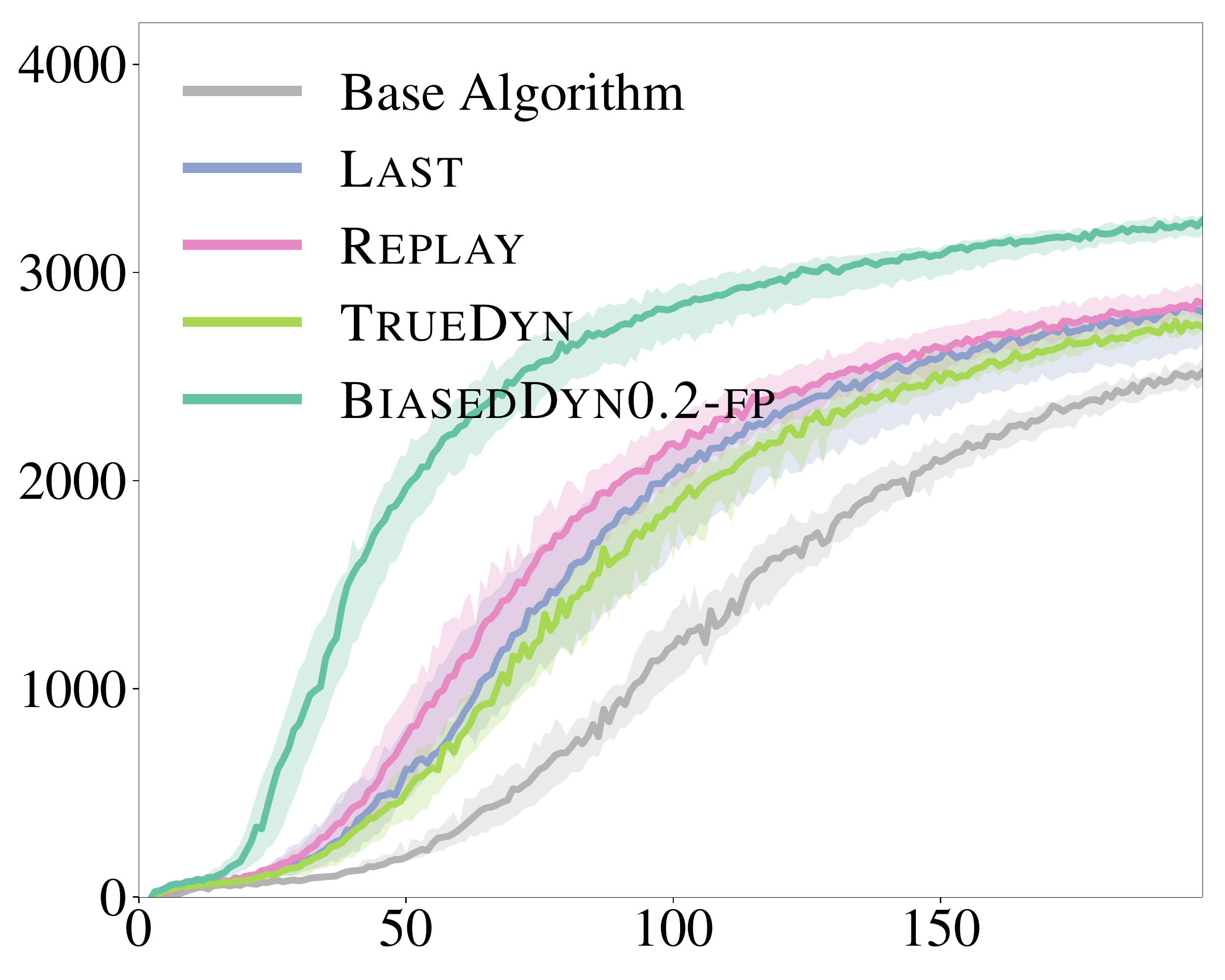}
	\caption{Snake \adam}
\end{subfigure}
\begin{subfigure}{.24\textwidth}
	\includegraphics[width=\textwidth]{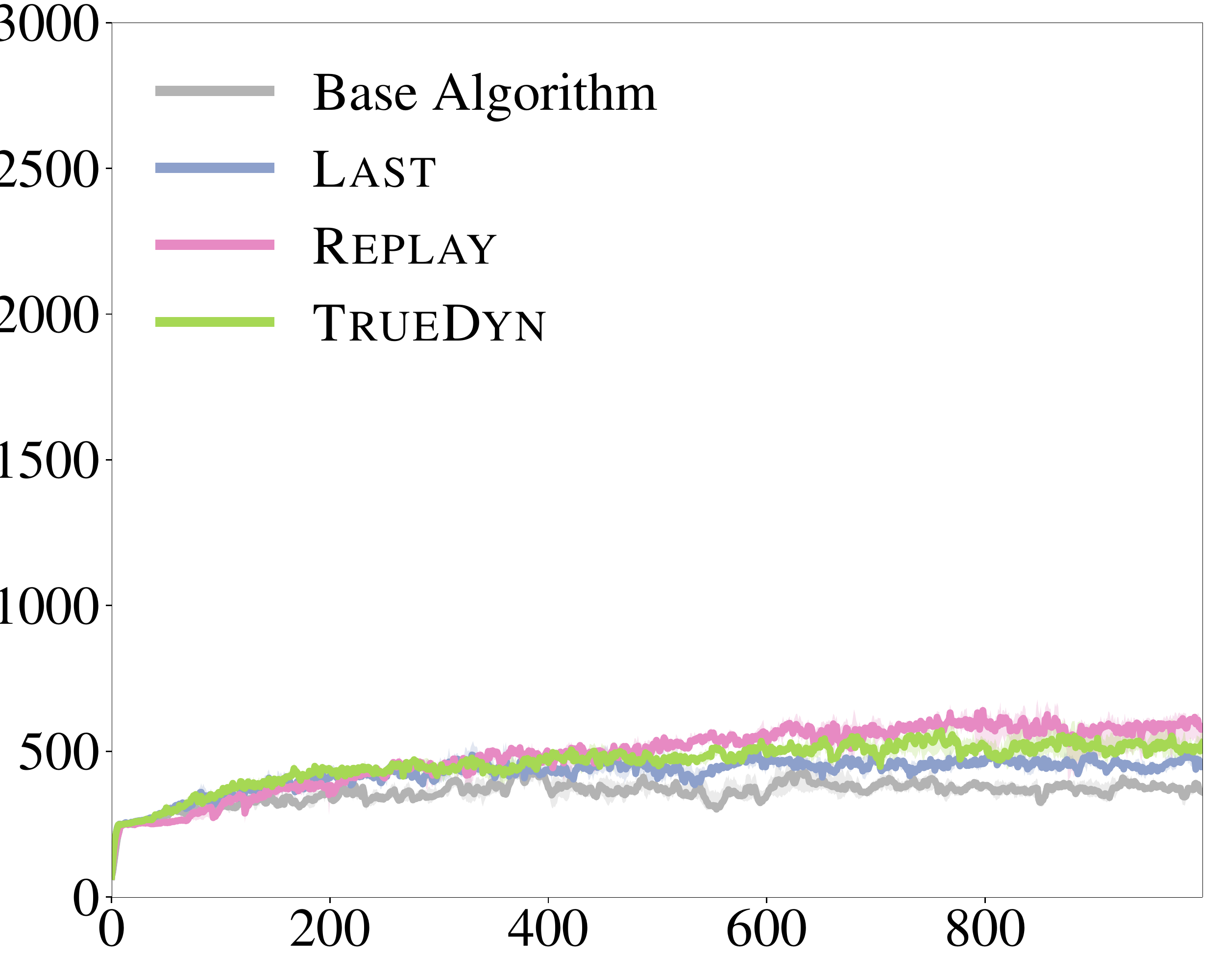}
	\caption{Walker3D \adam}
\end{subfigure}\\
	\centering
\begin{subfigure}{.24\textwidth}
	\includegraphics[width=\textwidth]{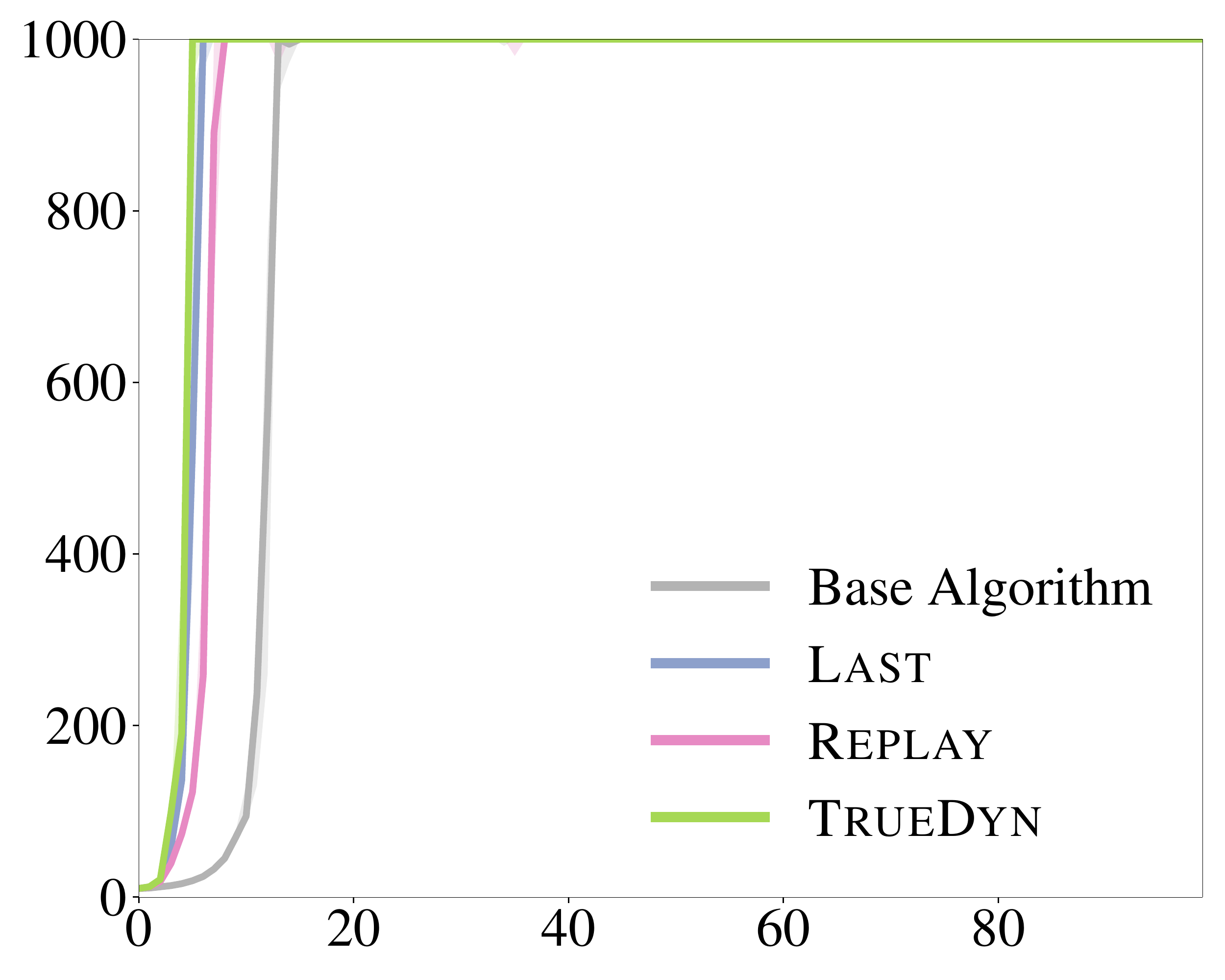}
	\caption{CartPole \natgrad}
\end{subfigure}
\begin{subfigure}{.24\textwidth}
	\includegraphics[width=\textwidth]{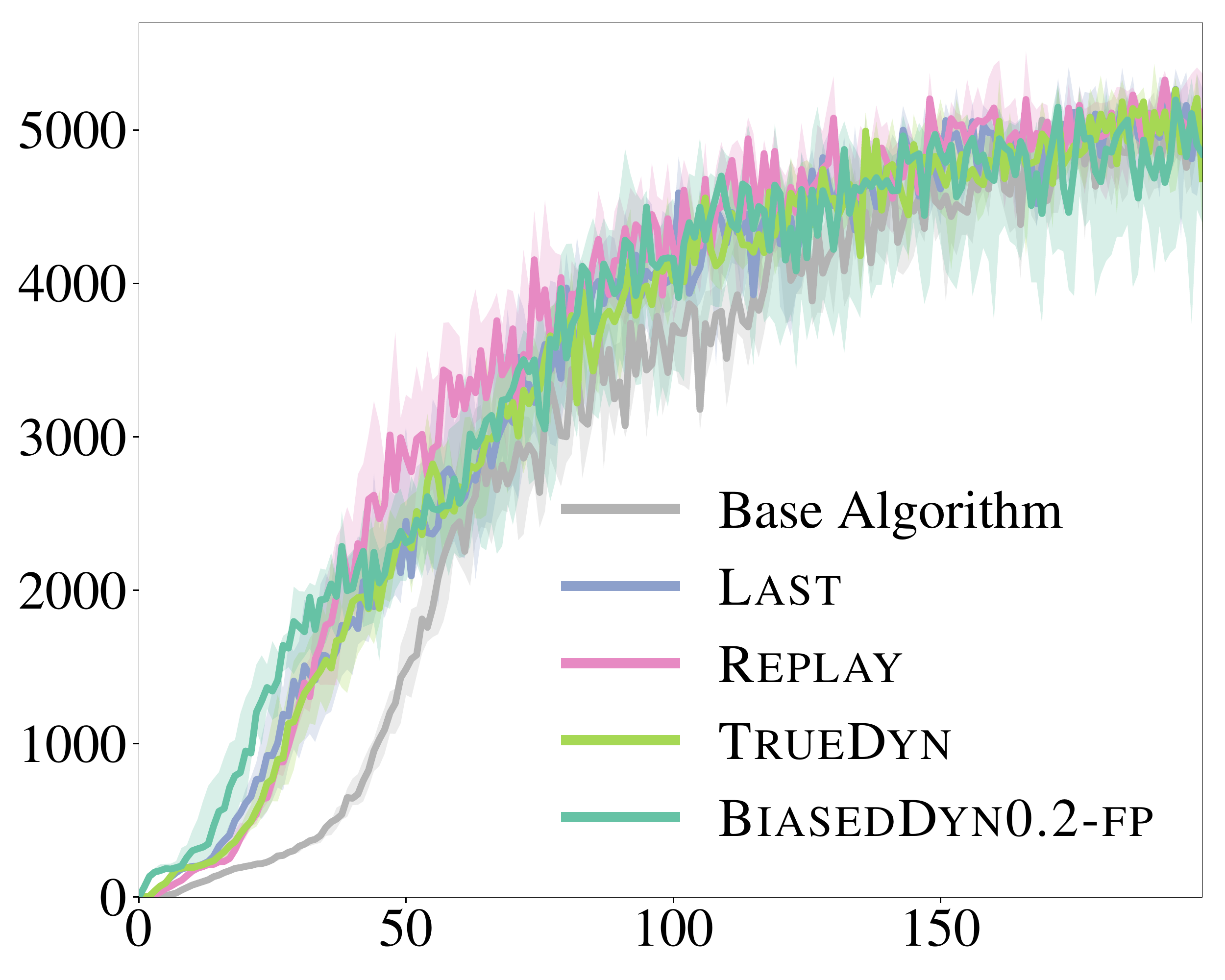}
	\caption{Hopper \natgrad}
\end{subfigure}
\begin{subfigure}{.24\textwidth}
	\includegraphics[width=\textwidth]{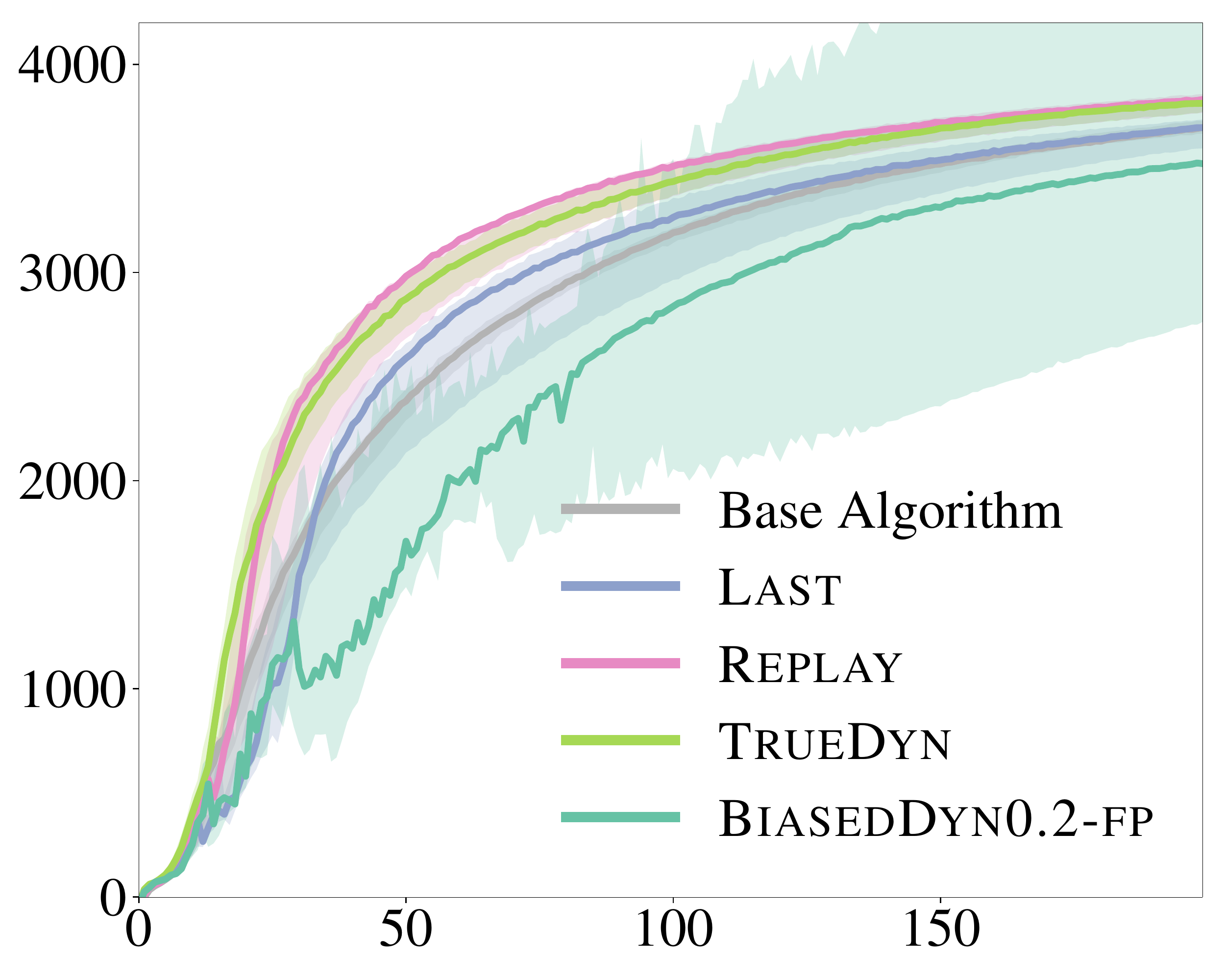}
	\caption{Snake \natgrad}
\end{subfigure}
\begin{subfigure}{.24\textwidth}
	\includegraphics[width=\textwidth]{walker3d_natgrad}
	\caption{Walker3D \natgrad}
\end{subfigure}\\
\centering
\begin{subfigure}{.24\textwidth}
\includegraphics[width=\textwidth]{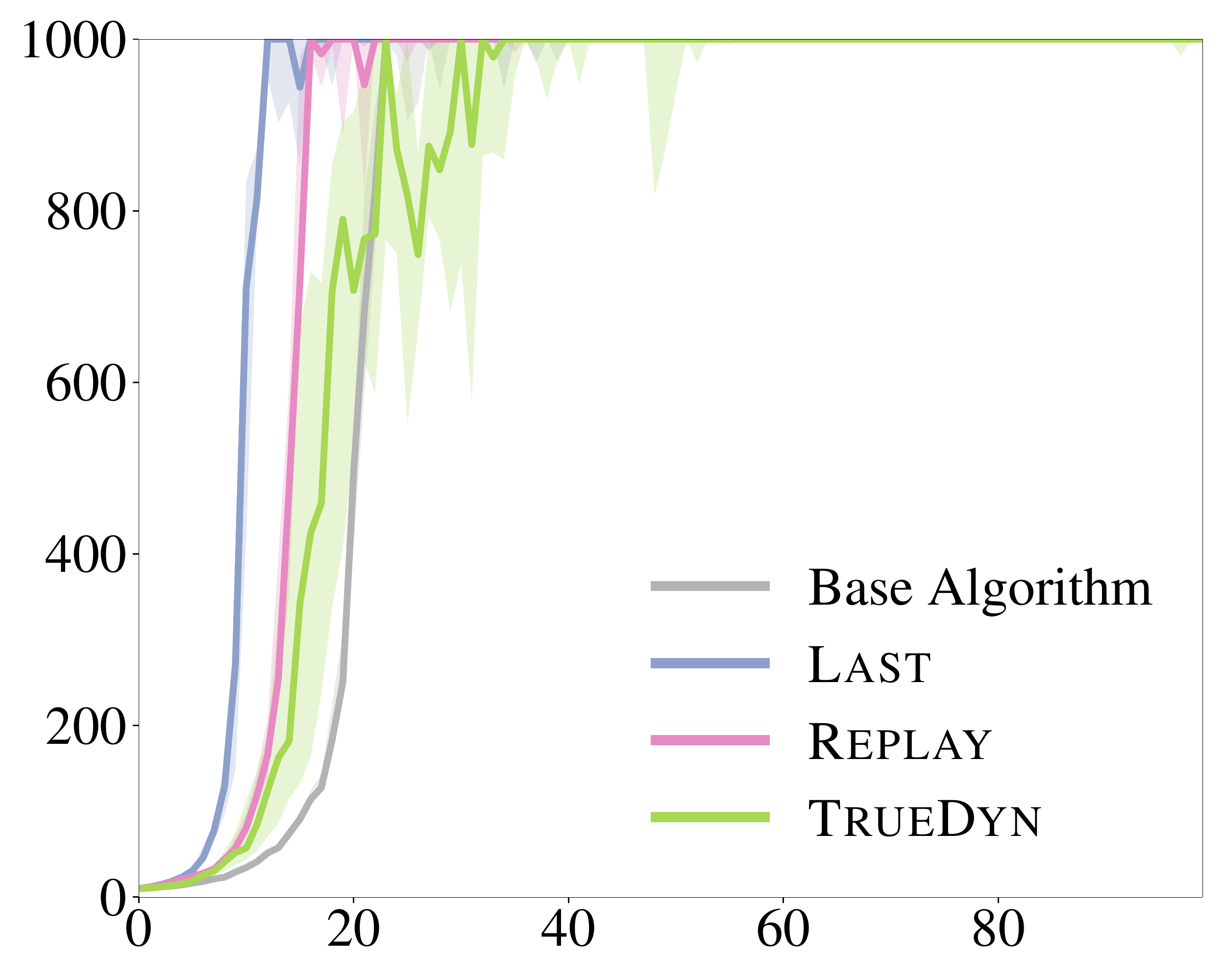}
\caption{CartPole \trpo}
\end{subfigure}
\begin{subfigure}{.24\textwidth}
	\includegraphics[width=\textwidth]{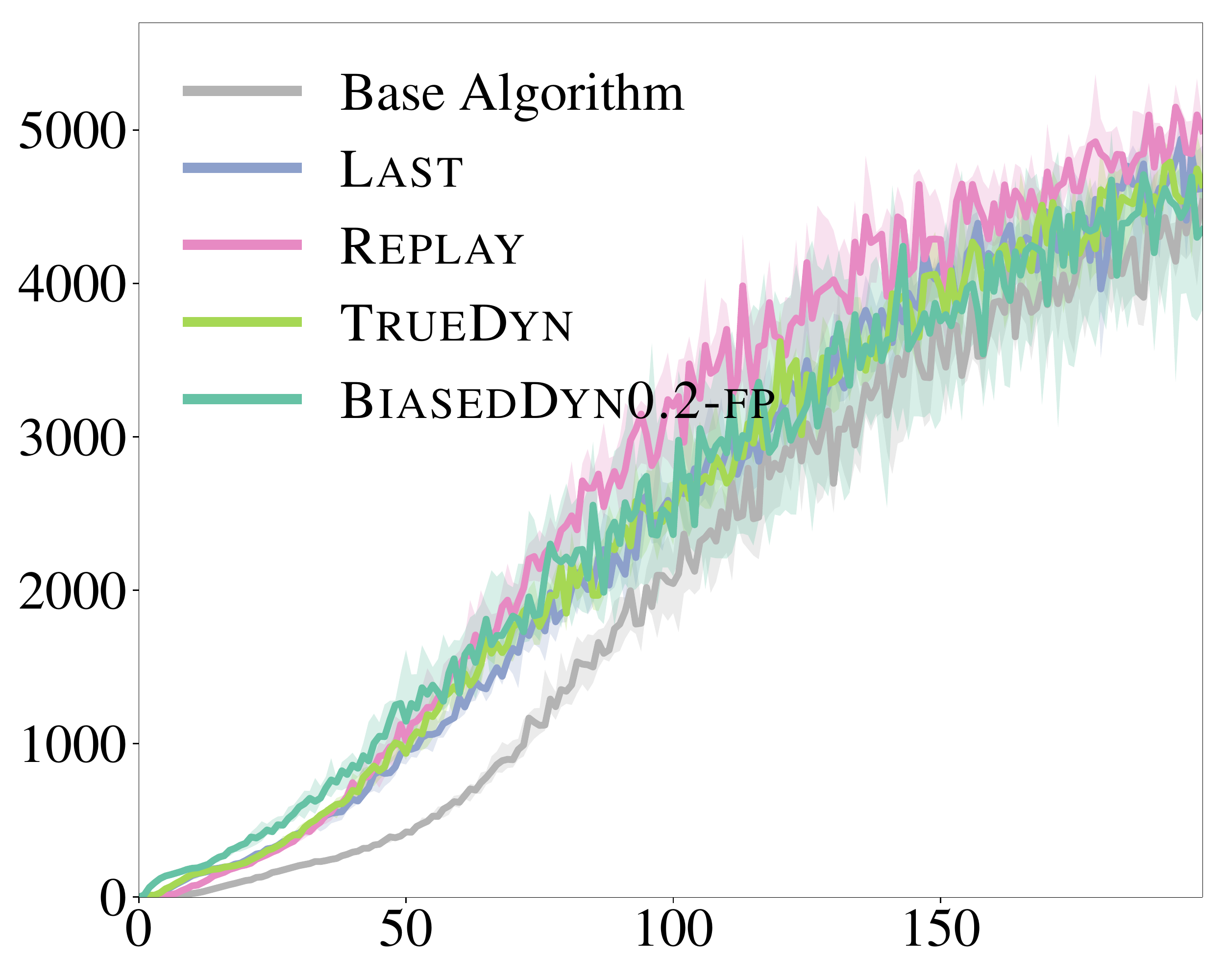}
	\caption{Hopper \trpo}
\end{subfigure}
\begin{subfigure}{.24\textwidth}
	\includegraphics[width=\textwidth]{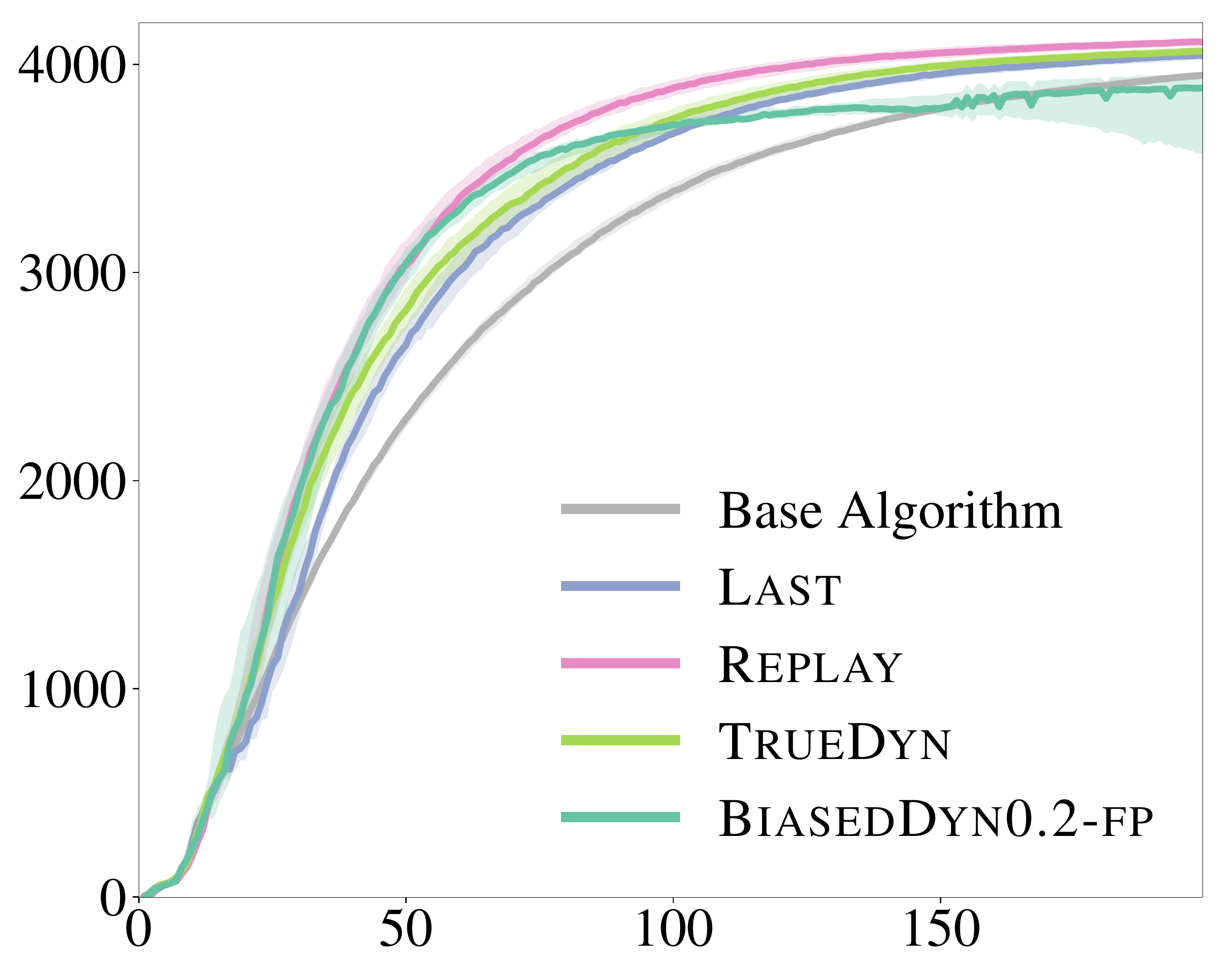}
	\caption{Snake \trpo}
\end{subfigure}
\begin{subfigure}{.24\textwidth}
	\includegraphics[width=\textwidth]{walker3d_trpo}
	\caption{Walker3D \trpo}
\end{subfigure}\\
\caption{
	The performance of \piccolo with different predictive models on various tasks, compared to base algorithms.
	The rows use \adam, \natgrad, and \trpo as the base algorithms, respectively. 
	$x$ axis is iteration number and $y$ axis is sum of rewards. The curves are the median among 8 runs with different seeds, and the shaded regions account for $25\%$ percentile. 
}
\label{fig:exps}
\end{figure*}

\subsection{Experiment Hyperparameters}

The  hyperparameters used in the experiments and the  basic attributes of the environments are detailed in Table~\ref{table:tasks}. 
\begin{table*}[h]
	\begin{center}
		\begin{tabular}{lccccc}
			&   CartPole &  Hopper & Snake &  Walker3D \\
			\hline \\
			Observation space dimension   						
			&	4  	&	11	&	17	&	41	\\
			Action space dimension        							
			& 	1	&	3	&	6	&	15	\\
			State space dimension 									 
			& 4    &  12	& 18 &  42 \\
			Number of samples from env. per iteration 	
			&	4k 	&	16k	&	16k	&	32k	\\
			Number of samples from model dyn. per iteration &	4k 	&	16k	&	16k	&	32k	\\		
			Length of horizon 
			& 1,000 & 	1,000 & 1,000 &  1,000 \\
			Number of iterations     								
			&	100	&	200	&	200 &	1,000 \\
			Number of iterations of samples for $\replay$ buffer &
			5 & 4 & 3 & 2 (3 for $\adam$) \\
			$\alpha$ 
			\footnotemark
			& 0.1 & 0.1 & 0.1 & 0.01 \\
			$\eta$ in \adam & 0.005 &  0.005 &  0.002 & 0.01 \\
			$\eta$ in \natgrad & 0.05 &  0.05 &  0.2 & 0.2 \\
			$\eta$ in \trpo & 0.002 &  0.002 &  0.01 & 0.04 \\
		\end{tabular}
		\caption{Tasks specifics and hyperparameters. }
		\label{table:tasks}
	\end{center}
\end{table*}
\footnotetext{$\alpha$ and $\eta$ appear in the decaying step size multiplier for all the algorithms in the form $\eta /(1 + \alpha \sqrt{n} )$. $\alpha$ influences how fast the step size decays. We chose $\alpha$ in the experiments based on the number of iterations. }